\documentclass[Afour,sageh,times]{sagej}
\usepackage{CJK}
\usepackage{float}
\usepackage{cases}
\usepackage{amsmath}
\usepackage{amssymb}
\DeclareGraphicsExtensions{.eps,.ps,.jpg,.bmp,.pdf}
\usepackage{moreverb,url}

\usepackage{amsmath,amssymb,amsfonts}
\usepackage{algorithmicx}
\usepackage{algorithm}
\usepackage{algpseudocode}
\usepackage{graphicx}
\usepackage{subfigure}
\usepackage{textcomp}
\usepackage{wrapfig}
\usepackage{enumitem}
\usepackage{siunitx}
\usepackage{xcolor}
\usepackage{booktabs}
\usepackage[normalem]{ulem}
\usepackage{cancel}

\usepackage{supertabular}
\usepackage{longtable}
\usepackage{hhline}
\usepackage{multirow}
\usepackage{makecell}
\usepackage{dblfloatfix}    % To enable figures at the bottom of page
\usepackage[normalem]{ulem}
\usepackage{arydshln}
\usepackage{dblfloatfix}    % To enable figures at the bottom of page
\makeatletter

\newcommand{\Rmnum}[1]{\expandafter\@slowromancap\romannumeral #1@}

\makeatother

\usepackage[utf8]{inputenc}
\usepackage[english]{babel}
\usepackage{amsthm}

\newtheorem{remark}{\hspace{1em}Remark}
\newtheorem{theorem}{Theorem}

\newtheorem{lemma}[theorem]{Lemma}

\newtheorem{assumption}{\hspace{1em}Assumption}

\let\vec\boldvec
\let\mat\boldvec

\usepackage[colorlinks,bookmarksopen,bookmarksnumbered,citecolor=red,urlcolor=red]{hyperref}

\newcommand\BibTeX{{\rmfamily B\kern-.05em \textsc{i\kern-.025em b}\kern-.08em
T\kern-.1667em\lower.7ex\hbox{E}\kern-.125emX}}

\def\volumeyear{2025}

\definecolor{GfLi}{RGB}{0, 0, 255}  % blue, for revision
\definecolor{mygray}{gray}{0.5}     % for some explanation
\definecolor{myred}{RGB}{255, 0, 0} % for revisions
\definecolor{mygreen}{RGB}{100, 200, 100}% for comments
\definecolor{orange}{RGB}{255, 100, 100}  % 

\begin{document}
%\begin{CJK}{song}

\runninghead{Gaofeng Li {\it et al.}}

\title{Orientation Learning and Adaptation towards Simultaneous Incorporation of Multiple Local Constraints}

\author{Gaofeng Li*\affilnum{1}, Peisen Xu\affilnum{1}, Ruize Wang\affilnum{1}, Qi Ye\affilnum{1}, Dezhen Song\affilnum{2}, Jiming Chen*\affilnum{1} and Yanlong Huang*\affilnum{3}}

\affiliation{\affilnum{1}College of Control Science and Engineering, Key Laboratory of Collaborative Sensing and Autonomous Unmanned Systems of Zhejiang Province, Zhejiang University, Hangzhou, China. \\ 
\affilnum{2}Department of Robotics, Mohamed Bin Zayed University of Artificial Intelligence, Abu Dhabi, The United Arab Emirates.\\
\affilnum{3}School of Computer Science, University of Leeds, Leeds, UK.
}

\corrauth{Gaofeng Li, Jiming Chen, and Yanlong Huang}

\email{ \{gaofeng.li, cjm\}@zju.edu.cn and Y.L.Huang@leeds.ac.uk}

\begin{abstract}
Orientation learning plays a pivotal role in many tasks. However, the rotation group SO(3) is a Riemannian manifold given a distance metric. As a result, the distortion caused by the non-Euclidean geometric nature introduces difficulties in the incorporation of local constraints, especially for the simultaneous incorporation of multiple local constraints. To address this issue, we propose the Angle-Axis Space-based orientation representation method to solve several orientation learning problems, including orientation adaptation and minimization of angular acceleration. Specifically, we propose a weighted average mechanism in SO(3) based on the angle-axis representation method. Our main idea is to generate multiple trajectories by considering different local constraints at different basepoints. Then these multiple trajectories are fused to generate a smooth trajectory by our proposed weighted average mechanism, achieving the goal of incorporating multiple local constraints simultaneously. Compared with existing solutions, ours can address the distortion issue and make the off-the-shelf Euclidean learning algorithm re-applicable in non-Euclidean spaces. Simulation and experimental evaluations validate that our solution can not only adapt orientations towards arbitrary desired via-points and cope with angular acceleration constraints, but also incorporate multiple local constraints simultaneously to achieve extra benefits, e.g., achieving smaller acceleration costs.

\end{abstract}


\maketitle

\section{Introduction}
\label{Section-Introduction}

\begin{table}[t]
  \centering
  \caption{ Acronym Table }
  \label{Tables-of-Acronym}
  \setlength{\tabcolsep}{5mm}{
  \begin{tabular}{c|c}
        \hline
         {\bf Acronym}   &  {\bf Full Name}         \\ \hline
         IL    &  Imitation Learning                \\
         LfD   &  Learning from Demonstration       \\
         KMP   &  Kernelized Movement Primitives    \\
         GMM   &  Gaussian Mixture Model            \\
         GMR   &  Gaussian Mixture Regression       \\
         IOC   &  Incomplete Orientation Constraint \\
         IOVP  &  Incomplete Orientation Via-Point  \\
         DoF   &  Degree-of-Freedom                 \\
        \hline
  \end{tabular}
  }
\end{table}

In many complex tasks, such as fine manipulation of objects and human-robot collaboration, it is non-trivial to plan proper trajectories for robots in advance. In order to endow a robot with agile skills, imitation learning (IL) provides a promising solution to deal with such situations. Many remarkable algorithms in the scope of learning from demonstrations have been developed, such as Dynamical Movement Primitives (DMP) (\cite{ijspeert2013dynamical, 2023IJRR-MSaveriano-DMPSurvey}), Probabilistic Movement Primitives (ProMP) (\cite{paraschos2013probabilistic, gomez2020adaptation, 2018AuRo-AParaschos-IL-ProMP}) and Kernelized Movement Primitives (KMP) (\cite{2019IJRR-KMP, 2023ICRA-YanlongHuang-IM, 2020ICRA-YLHuang-IL-KMP, 2021RAS-Dakka-IL-KMP}), which have been applied successfully in numerous applications, e.g., peg-in-hole (\cite{abu2015adaptation, 2023FronNeuRo-SunLining-ImiLearning-PeginHole}), striking tasks (\cite{gomez2020adaptation, 2022ICCA-Lu-SkillLearningToolUse}), and locomotion (\cite{ding2020robust, 2023TMech-JTDing-Locomotion}).

Similar to skill learning in linear spaces (e.g., Cartesian space or joint space), the orientation learning also plays a pivotal role in many tasks, such as pouring water into a cup or turning a key to open a door. As a result, the orientation learning by imitation has been an attractive topic and several orientation learning methods have been reported. Currently, existing learning approaches are extended to a Riemannian space $\mathcal{M}$ by following such a procedure: First, the Riemannian data or variable of interest is projected onto a single tangent space $\mathcal{T}_{\vec x}\mathcal{M}$ at a given point ${\vec x}$; later, an off-the-shelf Euclidean learning algorithm takes the projected data as inputs to fit a Euclidean Gaussian Mixture Model (GMM) in this tangent space $\mathcal{T}_{\vec x}\mathcal{M}$; and finally the results are projected back into the Riemannian manifold. This framework is referred to as the ``single-tangent-space learning" (\cite{2024ICRA-Jaquier-SingleTangentSpaceFallacy}). For example, \cite{ude2014orientation} proposed a method to transform the quaternions into a Euclidean space, which enables the learning of quaternions by DMP without the need of additional re-normalization. A similar transformation between quaternions and Euclidean space was also exploited in our previous work (\cite{2020TRO-KMP-Orientation}). In \cite{zeestraten2017approach}, orientation learning was studied by resorting to the Riemannian topology, which ensures that proper orientations are generated.

\begin{figure*}[t]
    \centering
    \includegraphics[width=0.95\hsize]{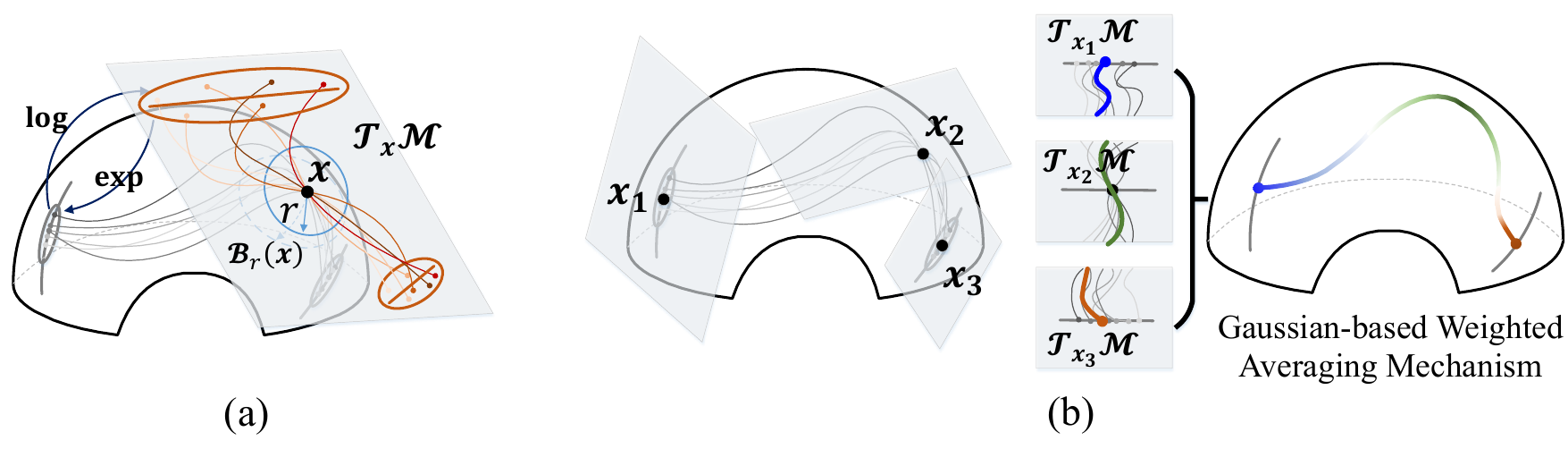}
    \caption{Problem Illustration. (a) The single-tangent-space learning framework. All data are projected into a single tangent space to encode the demonstration distributions. Due to the non-Euclidean structure, the distribution for points that are outside the neighborhood $\mathcal{B}_{r}\left({\vec x}\right)$ is distorted. For example, the starting/ending points of every demonstration lie on a geodesic in the manifold $\mathcal{M}$. However, due to the distortion, this property does not hold in the tangent space. Therefore, to incorporate a local constraint, this framework has to project all data into the tangent space that is extended around the given basepoint. Obviously, this framework is inapplicable when multiple local constraints are required.  (b) To address this problem, we propose to generate multiple trajectories (the colored curves in each tangent space) by projecting data into multiple tangent spaces, to incorporate each local constraint at different basepoints. Then these multiple trajectories are then fused into a final trajectory (the colored curve in the right subfigure, whose color changes from light to dark as the weights increase) by using a fusing mechanism, which is challenging to design. The difficulties lie in two aspects. First, how to guarantee that each reproduced trajectory is dominant in the final trajectory at the given local time domain, such that the local constraint is satisfied. Second, how to guarantee the smoothness of the final reproduced trajectory. We propose a Gaussian-based weighted averaging mechanism to address the above difficulties.}
    \label{fig1-problem-illustration}
\end{figure*}

Nevertheless, due to distance distortion, the single-tangent-space learning faces severe challenges in the simultaneous incorporation of multiple local constraints.
%{\color{red}\sout{the orientation learning still faces many challenges due to the nonlinear structure of the rotation group SO(3).}} 
Unlike the learning of Cartesian positions and joint angles (\cite{2020ICRA-YLHuang-IL-KMP, 2023ICRA-YanlongHuang-IM, armesto2017learning, zhou2017task}), in which the incorporation of additional constraints has been well studied, the orientation learning is still unable to incorporate multiple local constraints simultaneously.
%{\color{red}\sout{must respect a certain nonlinear constraint.}} 
Although it seems appealing to comply with existing Euclidean approaches, such a local linear approximation unavoidably leads to distorted models as they overlook the intrinsic geometry of the manifold. More importantly, such distortions depend on the choice of the tangent space. This is because, as shown in Fig. \ref{fig1-problem-illustration}, the tangent space $\mathcal{T}_{\vec x}\mathcal{M}$ is a local diffeomorphism of the manifold $\mathcal{M}$ at a neighborhood $\mathcal{B}_r\left({\vec x}\right)$ of the given point ${\vec x}$. Here the subscript $r$ represents the injectivity radius that fundamentally quantifies the extent to which geodesics behave like straight lines in $\mathbb{R}^d$. When the data is far from the given point ${\vec x}$, the distortions induced by projecting data onto the Euclidean tangent space lead to inaccuracy and/or even failure in considering additional constraints. Therefore, to incorporate an additional constraint, an appropriate basepoint ${\vec x}$ must be selected to achieve optimal performance. 
%{\color{red}\sout{For example, a rotation matrix $\vec{R}$ must be an orthogonal matrix (i.e., $\vec{R}^\top\vec{R}=\vec{I}$), a quaternion $\vec{q}$ must have an unit magnitude (i.e., $||\vec{q}||=1$). Such constraint enforces the entries of an orientation coupled and leads to the difficulties in considering additional constraints.}} 
Taking the Incomplete Orientation Constraints (IOCs) as an example, to fulfill a given task, it is often required to relax the constraint on one DoF, while setting strict constraints on the other DoFs for a desired orientation via-point ${\mat R}_d$ (\cite{ude2014orientation,abu2015adaptation, 2020TRO-KMP-Orientation, zeestraten2017approach}). 
%{\color{red}\sout{as a result of the coupling of orientation entries, it is intractable to apply a different variance to each entry separately, as done at linear space, making previous methods (e.g., [10], [16]-[18]) inappropriate when dealing with incomplete or partial orientation constraints.}}
However, owing to the distortion problem, the single-tangent-space learning has to project all datapoints into the tangent space $\mathcal{T}_{{\vec R}_d}SO(3)$ extended at the given desired via-point ${\mat R}_d$. Otherwise, the strict constraints on the other DoFs are unable to be guaranteed, as shown in Fig. \ref{fig1-problem-illustration}(a), since the single-tangent-space learning disregards the injectivity radius $r$. Based on this feature, we refer to these constraints as {\it local constraints}. 
Obviously, the single-tangent-space learning is unable to incorporate multiple local constraints simultaneously, due to the distortion problem.

To address the distortion problem, many geometry-aware approaches (\cite{2020RAM-Calinon-RiemannianLearning, 2024ICRA-Jaquier-SingleTangentSpaceFallacy}) have been proposed by reformulating the optimization objective functions on the Riemannian manifold. For example, in \cite{zeestraten2017approach}, the likelihood function is redefined using the parallel transport, to derive the Gaussian conditioning and Gaussian product on Riemannian manifolds. However, the extra constraints, e.g., adaptation to via-points, have not been considered. In order to consider different constraints, these approaches need to design different optimization objective functions. In addition, they generally rely on time-consuming iterative algorithms, e.g., the Karcher flow, to search for solutions. Therefore, the incorporation of multiple local constraints is still an open issue for orientation imitation learning.

The incorporation of multiple local constraints is an appealing topic. To address the distortion problem, a straightforward solution is to generate multiple trajectories by projecting data into multiple tangent spaces, in order to incorporate each local constraint at different basepoints. Then these multiple trajectories can be fused into a final trajectory. However, the difficulties lie in two aspects. First, it is necessary to ensure the dominance of each reproduced trajectory, such that the local constraint is satisfied. As stated, the distortion is larger for points that are far away from the basepoint. Every generated trajectory is obtained by projecting data into tangent spaces extended around given basepoints at a given time. Therefore, we hope each trajectory can be dominant in the final trajectory at the corresponding local time domain. Second, the smoothness of the final reproduced trajectory must be ensured. Although the KMP can guarantee the smoothness of each trajectory, the smoothness of the final trajectory is also affected by the fusing mechanism.%To incorporate multiple local constraints, it is a key to consider multiple tangent spaces in the design of non-Euclidean learning.

Despite the big challenges, the success in considering multiple local constraints with off-the-shelf Euclidean GMMs would extend existing IL frameworks to skill learning in non-linear spaces. This motivates us to develop the weighted average mechanism in this paper. Taking the IOC as an example, we aim to develop a new orientation learning method which is able to tackle the simultaneous incorporation of multiple local constraints. Our main contributions include:

\begin{itemize}
    \item We propose to encode and/or decode the demonstrations and via-points in the Angle-Axis Space $\vec{B}_\pi \backslash \mathcal{S}^2_\pi$, which is a subset of $\mathbb{R}^3$ and bijective to SO(3). Then the KMP algorithm is leveraged to solve several orientation learning problems, including orientation adaptation and minimization of angular acceleration, where the transformation of demonstrations and via-points into $\vec{B}_\pi \backslash \mathcal{S}^2_\pi$ is provided. %We extend the Angle-Axis Space-based representation method to solve several orientation learning problems, including orientation adaptation and minimization of angular velocity, by exploiting KMP. The transformation method to transfer demonstrations and via-points into $\vec{B}_\pi \backslash \mathcal{S}^2_\pi$ are provided.
        We demonstrate that the local constraints, e.g., the IOCs, can be naturally addressed using our framework. %{\color{red}\sout{, which however are difficult for traditional methods.}}
        
    \item More importantly, we propose a weighted average mechanism in SO(3), in which the geodesic curve is utilized to achieve a weighted sum in non-Euclidean spaces, based on the angle-axis representation. In addition, a memory-based algorithm is proposed to address the discontinuity issue in the weighted rotation average, to guarantee the smoothness and robustness of the averaging result. To the best of our knowledge, there is still no attempt to illustrate and address the discontinuity issue in weighted rotation average to date.
    
    \item We propose to generate multiple trajectories by projecting data into multiple tangent spaces to incorporate each local constraint at different basepoints. Then we design a Gaussian-based weighted curve to fuse these multiple trajectories into a smooth trajectory, achieving the goal of incorporating multiple local constraints simultaneously. Although the distortion exists for each reproduced trajectory, our weighted rotation average mechanism can guarantee that each trajectory is only effective at the given local time domain. Consequently, the distortion problem is addressed by our proposed method. Compared with existing solutions, ours can make the off-the-shelf Euclidean learning algorithm re-applicable in non-Euclidean spaces.
    
    \item Simulation and experimental evaluations show that our proposed method can not only inherit KMP's capability of adapting orientations towards arbitrary desired via-points and coping with angular acceleration constraints, but also incorporate multiple local constraints, e.g., IOCs, to obtain additional benefits.
\end{itemize}

%In this paper, we propose to represent orientations in Axis-Angle Space, which allows for imposing incomplete orientation constraints in a straightforward way. Specifically, such representation can be extended to orientation learning and adaptations, including learning the distribution of multiple orientation demonstrations, adapting the learning orientation skills towards desired points with incomplete constraints, as well as incorporating the cost of acceleration minimization.

The rest of this paper is organized as follows. Section~\ref{Section-RelatedWorks} discusses the related work. Section~\ref{Section_MathematicPre} introduces basic concepts and discussions regarding the Angle-Axis Space. Section~\ref{Section_Methods} explains orientation learning and adaptations in Angle-Axis Space. Section~\ref{sec-IOVPs} handles the incorporation of multiple local constraints. The Gaussian-based weighted average mechanism is also designed in this section. Section~\ref{Section_Simulations} provides several evaluations to validate the learning and adaptation capabilities of the proposed method in simulations. Section~\ref{Section_Experiments} validates the benefits of our proposed method by experimental results. The conclusion is finally drawn in Section~\ref{Section_Conclusion}.

\section{Related Works}
\label{Section-RelatedWorks}

Reliable generation and execution of human-like trajectories are basic but essential for robots working in human-inhabited environments. The Imitation Learning (IL) or Learning from Demonstration (LfD) paradigm provides an effective way to reliably transfer human-like skills to robots (\cite{2020AnnualReview-HRavichandar-ILSurvey}), which is done by extracting motion features from few demonstrations and subsequently employing them in new scenarios (\cite{2022JAS-YLHuang-ILSurvey}). The typical IL/LfD approaches include Gaussian Mixture Model (GMM) (\cite{2007TSMC-Calinon-IL-GMM}) / Gaussian Mixture Regression (GMR) (\cite{1996JAIR-GMR}), Dynamical Movement Primitives (DMP) (\cite{ijspeert2013dynamical, 2023IJRR-MSaveriano-DMPSurvey}), Stable Estimator of Dynamical Systems (SEDS) (\cite{2011TRO-SMZadeh-IL-SEDS, 2023TRO-ChenguangYang-NeuEnerFuction}), Probabilistic Movement Primitives (ProMP) (\cite{paraschos2013probabilistic, gomez2020adaptation, 2018AuRo-AParaschos-IL-ProMP}), Task-parameterized GMM (TP-GMM) (\cite{2014ICRA-Calinon-IL-TPGMM, 2016IntellSerRob-Tutorial-TPGMM}), Kernelized Movement Primitives (KMP) (\cite{2019IJRR-KMP, 2023ICRA-YanlongHuang-IM, 2020ICRA-YLHuang-IL-KMP, 2021RAS-Dakka-IL-KMP}), and so on.

Over the past decades, these approaches have achieved great success in transferring linear skills, i.e., demonstrations and/or via-point data can be encoded in Euclidean spaces, like Cartesian positions/velocities, joint positions/velocities, and forces, etc. For example, the GMM/GMR method was used in surgical tasks by \cite{2010IEMB-Reiley-ILinSurgical} to generate smooth trajectories ({\it Cartesian positions}) from expert surgeons' demonstrations. The DMP method was used by \cite{2016Humanoid-YLHuang-PingpangHittingPoint} to model human demonstrations and generate hitting trajectories ({\it Cartesian positions/velocities}) for robotic table tennis tasks. In \cite{2016PlosOne-Peternel-DMP-ForceGenerator}, the DMP was directly used as a {\it joint torque} generator for exoskeleton actuators in the control loop, which is closed by a feedback from the human user's muscle activity. \cite{2019AdRo-Joshi-DMP-ClothingDrssing} proposed a framework for robotic clothing assistance by imitation learning (DMP) from human demonstrations ({\it Cartesian positions}) to a Baxter Robot. The ProMP method was applied to generate plucking motion ({\it Cartesian positions}) for a tea harvesting robot, with adaptation ability to the stiffness of branches (measured according to {\it force} feedback) (\cite{2020RAL-Motokura-ProMP-ForceAndPluckingMotions}). \cite{2021RAL-Zou-KMP-JointAngles} combined the Neural Networks and KMP to generate individualized gait patterns ({\it joint positions}) for a lower limb exoskeleton. In summary, existing approaches have shown great abilities for linear skills learning in adapting to arbitrary desired points (e.g., starting/via/end-points) while taking into account high-dimensional inputs and smoothness constraints.

For many robotic tasks of interest, it is common to have non-Euclidean data like orientation, impedance, and/or manipulability. Naturally, it is expected to directly extend existing IL/LfD paradigms from Euclidean spaces to non-Euclidean spaces. For example, by using quaternion notation, \cite{2009ICRA-Pastor-OrientationLearning-DirectExtension} added the gripper's orientation into the movement trajectory and learned the position and orientation trajectories simultaneously by exploiting DMP. Similarly, \cite{2015IROS-Silverio-OrientationLearning-TpGMMDirectExtension} also directly extended the TP-GMM method to orientation learning. However, the reproduced quaternion cannot satisfy the unit magnitude constraint, since the non-Euclidean geometric property is not considered. Thus an additional re-normalization is required. To address this issue, \cite{ude2014orientation} proposed a method to transform the quaternions into a Euclidean space, which enables the learning of quaternions by DMP without the need of additional re-normalization. A similar transformation between quaternions and Euclidean space was also exploited in \cite{2020TRO-KMP-Orientation, 2019ICRA-KMP-Orientation-Quaternion} by using KMP and in \cite{2019AuRo-Ravichandar-OrientationLearning-DirectExtension} by using SEDS. In addition, \cite{2019ICRA-Saveriano-OrientationLearning-DirectExtension} extended previous approaches to merge the trajectories generated by position DMPs and orientation DMPs to a unified trajectory that includes both position and orientation parts.

However, the above approaches neglect the non-Euclidean geometric property of the rotation group SO(3). As development progressed, many researchers have realized that neglecting the non-Euclidean geometric nature can adversely affect learning performance and accuracy. As a result, the geometry-aware IL/LfD approaches are receiving growing attention and many valuable works have emerged in recent years.

Initially, these geometry-aware IL/LfD approaches were following the single-tangent-space learning paradigm, i.e., projecting the Riemannian data to a single tangent space, learning by an off-the-shelf Euclidean IL/LfD paradigm, and then projecting back to the Riemannian manifold. For example, \cite{2023IJRR-HBMohammadi-LearnInRiemannian} revealed that Riemannian manifolds may be learned via human demonstrations in which geodesics are natural motion skills. In this work, the geodesics are generated using a learned Riemannian metric produced by a variational autoencoder (VAE), which is especially intended to recover full-pose end-effector states and joint space configurations. By exploiting the Riemannian geometric structure of the orientation data, \cite{2022arXiv-Naseem-OrientationLearningbyReinforcementLearning} proposed a reinforcement learning (RL) framework to apply policy parameterization on a tangent space, and then map the result to its corresponding manifold. Similarly, in \cite{2024NeuroComputing-ChenguangYang-LearningInRiemannian}, the DMP formulation is reconstructed based on the $\mathcal{S}^3$ Riemannian metric to learn, reproduce and generalize the robot's quaternions. Then simultaneous encoding of position, orientation and stiffness is performed to transfer human-like multi-space skills to the robot.
%Another reference: Wang {\it et al.} \cite{2022IEEEAccess-WTWang-LearnDeepInRiemannian}

However, as revealed in \cite{2024ICRA-Jaquier-SingleTangentSpaceFallacy}, the single-tangent-space learning paradigm would exhibit a mathematically-flawed simplification and such an approximation severely affects the performance of machine learning methods, even in low-dimensional manifolds of constant curvature. The good practices for operating with Riemannian manifolds include: using a Riemannian metric, leveraging multiple tangent spaces, applying a Riemannian Gaussian Distribution (RGD), and so on. For example, to extend the dynamic model of vector fields from flat Euclidean manifolds to non-Euclidean manifolds (e.g., Lie Groups), \cite{2022RAL-JUrain-LearnStableVectorFieldsInRiemannian} presented a novel vector field model that can guarantee most of the properties of previous approaches, i.e., stability, smoothness, and reactivity beyond the Euclidean space. Similarly, RL algorithms on Riemannian manifolds should be formulated on a case-by-case basis by adapting their core principles. For instance, Riemannian Bayesian optimization requires Riemannian kernels and optimization techniques (\cite{2020CoRL-Jaquier-BayesianInRiemannian, 2021CoRL-Jaquier-BayesianInRiemannian}). A.-Dakka et al. (\cite{2020ICRA-Dakka-GaDMP, 2022arXiv-Dakka-GaDMP}) used Riemannian metrics to reformulate Dynamic Movement Primitives (DMPs) such that the robot skills involving Riemannian manifold data, including, but not limited to, unit quaternions and symmetric and positive definite matrices, can be learned. \cite{zeestraten2017approach} extended the Gaussian conditioning, Gaussian product, and Gaussian Mixture Regression from Euclidean data to the rotation group by using the parallel transport.

By considering the non-Euclidean geometric nature, many IL/LfD approaches have been extended to Riemannian manifolds successfully. However, as we have stated, there is still no reported work on imitation learning that considers extra constraints, e.g., the adaptation to via-points, on a Riemannian manifold. Therefore, the incorporation of multiple local constraints on Riemannian manifolds remains open.
%{\color{red}\sout{Currently, there is still little work focusing on the incorporation of additional constraints on the non-Euclidean space.}} 
Taking the Incomplete Orientation Constraints (IOCs) as an example, the learning algorithm has to project all datapoints into a tangent space extended at the given desired via-point, to apply a different variance to each entry separately. Otherwise, the strict constraints on the other DoFs are unable to be guaranteed due to the distortion problem. Therefore, it is necessary to consider multiple local constraints in multiple tangent spaces.

However, how to merge the results learned in multiple tangent spaces is still unexplored. Although the RMPflow (\cite{cheng2018rmp}) had been proposed to combine multiple local policies in different manifolds, it is inapplicable to our problem. The RMPflow requires parameterizing the non-Euclidean behaviors as Geometric Dynamical Systems (GDS) to construct a data structure, RMP-tree, in which the local policies are structured as nodes, and the transformations between manifolds are represented as edges. In our problem, however, all the trajectories are learned in the same manifold, i.e., the rotation group SO(3), but are expressed in different tangent spaces. Thus the different local policies do not naturally form a hierarchical tree. In addition, the expert skills are embedded implicitly in the demonstration data rather than being explicitly modeled as geometric dynamical systems. Based on these reasons, a new fusing mechanism for combining multiple trajectories learned in different tangent spaces is required.

The IOC, which can be tracked back to the 1990s and was originally named ``functional redundancy", was refined in previous work (\cite{2021TMech-GaofengLi-IOC, 2023TASE-GaofengLi-TrajIOC}). As we have stated in Section \ref{Section-Introduction}, the incorporation of multiple local IOCs can bring many benefits. Therefore, we would like to focus on the orientation learning problem with the incorporation of %{\color{red}\sout{additional}}
multiple local constraints (e.g., IOCs) in this paper. Our main idea is to generate multiple trajectories by considering different local constraints at different basepoints. Then these trajectories are fused into a final reproduced trajectory by our proposed weighted average mechanism. Our method guarantees that each reproduced trajectory is only effective at the given local time domain in the final fused trajectory. In this way, the distortion problem can be addressed. Compared with existing solutions, our method lowers the difficulties of learning the Riemannian GMM and Riemannian Expectation-Maximization algorithm and makes all the off-the-shelf Euclidean learning algorithms re-applicable. 
%{\color{red}\sout{Our work is a further improvement to existing IL/LfD approaches.}}

\section{Preliminary: the Angle-Axis Space}
\label{Section_MathematicPre}

\subsection{Definition of the Angle-Axis Space}
%$\vec{B}_\pi \backslash \mathcal{S}^2_\pi$

This section gives the definition of the Angle-Axis Space and the mappings between the rotation group SO(3) and the Angle-Axis Space. Please refer to \cite{2013IJCV_Hartley_RotationAveraging, 2010MTNS_Hartley_RotationAveraging, 2009ACCV_Hartley_RotationAveragingApplication, Gaofeng-RAL2018} for more details.

As is well known, every orientation $\vec{R} \in SO\left(3\right)$ can be obtained by rotating an axis by at most $\pi$ radians from the identity. Thus an orientation $\vec{R}$ can be represented by the rotational axis and the rotational angle. Since the rotational angle is no more than $\pi$ radians, we can define a closed ball $\vec{B}_\pi$ of radius $\pi$ in $\mathbb{R}^3$ as:
\begin{align}
    \label{BallRegion}
    \begin{split}
        \vec{B}_\pi = \left\{ {\vec{\psi}} \in \mathbb{R}^3 \left| \, \left\| {\vec{\psi}} \right\| \leq \pi \right. \right\},
    \end{split}
\end{align}
where $\left\|\cdot\right\|$ denotes the $L_2$ norm of a vector.

Consequently, any orientation ${\vec R} \in SO\left(3\right)$ can be represented by a vector ${\vec{\psi}} \in \vec{B}_\pi$, in which the direction of ${\vec{\psi}}$ is the axis and the norm $\left\| {\vec{\psi}} \right\|$ is the angle.

Although the mapping from $\vec{B}_\pi$ to SO(3) is one-to-one in the interior, the two-to-one mapping occurs on the boundary of the ball. To avoid the non-bijective mapping, we define a half sphere $\mathcal{S}^2_\pi$ in $\mathbb{R}^3$ as:
\begin{align}
    \label{HalfSphereRegion}
    \begin{split}
        \mathcal{S}^2_\pi = \left\{ {\vec{\psi}} \in \mathbb{R}^3 \left| \, \left\| {\vec{\psi}} \right\| = \pi \, \text{and}\, \left(c_1 \, \text{or} \, c_2 \, \text{or} \, c_3\right) \right. \right\}
    \end{split},
\end{align}
where conditions $c_1$, $c_2$, and $c_3$ are defined as:
\begin{align}
    \label{conditions}
    \begin{split}
        c_1: & \quad \psi_x < 0,\\
        c_2: & \quad \psi_x = 0 \quad \text{and} \quad \psi_y < 0,\\
        c_3: & \quad \psi_x = 0 \quad \text{and} \quad \psi_y = 0 \quad \text{and} \quad \psi_z < 0,
    \end{split}
\end{align}
where $\psi_x$, $\psi_y$, and $\psi_z$ are the entries of ${\vec{\psi}}$.

\begin{figure}[b]
    \centering
    \includegraphics[width=0.4\hsize]{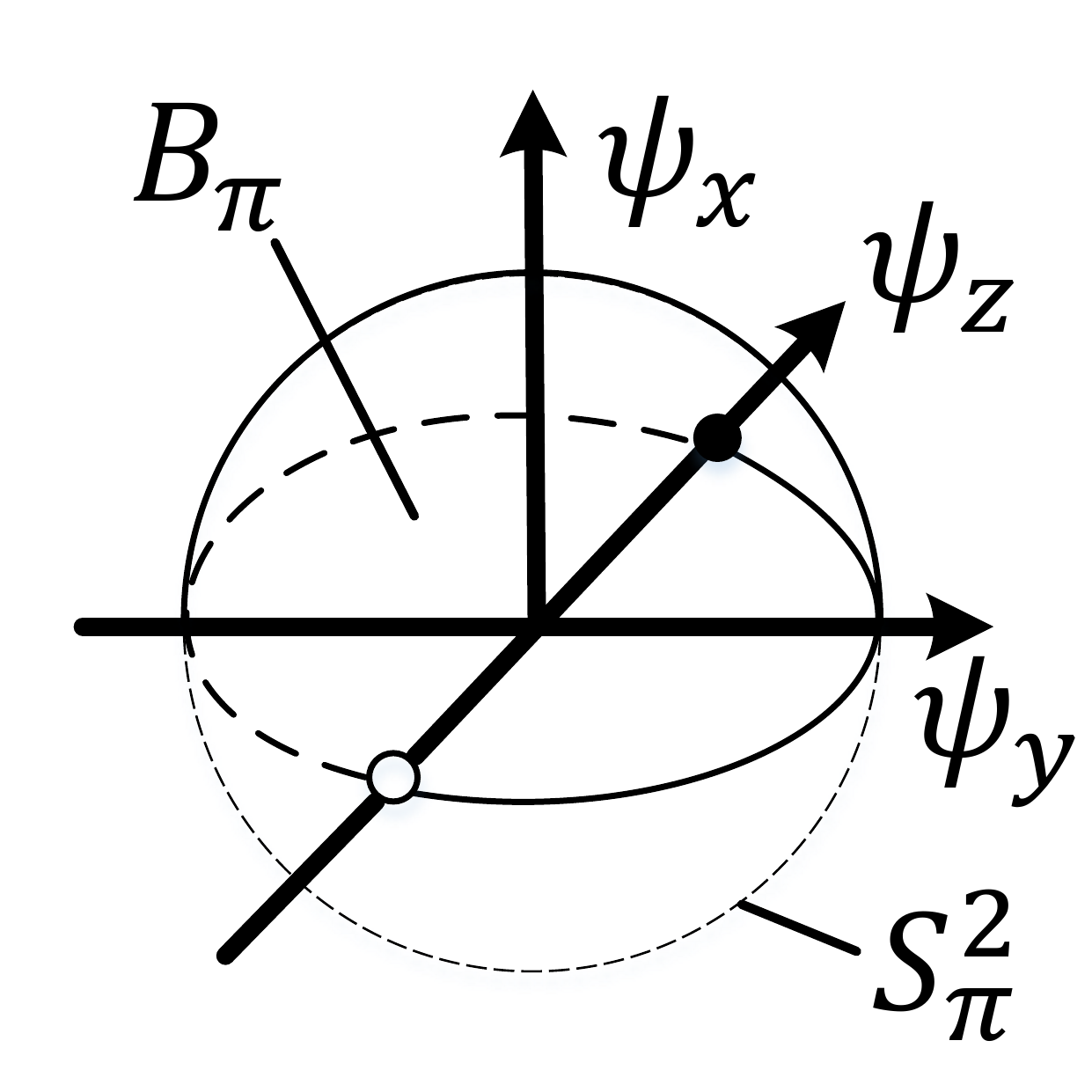}
    \caption{The Angle-Axis Space $\vec{B}_\pi \backslash \mathcal{S}^2_\pi$ is bijective to SO(3).}
    \label{fig:bijective}
\end{figure}

Thus, SO(3) is bijective to $\vec{B}_\pi \backslash \mathcal{S}^2_\pi$. Here the notation $A \backslash B := \left\{x \left| \, x \in A \,\text{and}\, x \notin B \right.\right\}$ represents the set of elements that belongs to set $A$, but does not belong to set $B$, i.e., the relative complement of $B$ in $A$. $\vec{B}_\pi \backslash \mathcal{S}^2_\pi$ is called {\bf Angle-Axis Space}, as illustrated in Fig.~\ref{fig:bijective}.

The mappings between SO(3) and $\vec{B}_\pi \backslash \mathcal{S}^2_\pi$ can be given by the exponential mapping $\exp \left( {\vec{\psi}} \right): \mathbb{R}^3 \longmapsto SO\left(3\right)$ and the logarithm mapping $\log \left(\vec{R}\right): SO\left(3\right) \longmapsto \vec{B}_\pi \backslash \mathcal{S}^2_\pi$ (\cite{Gaofeng-RAL2018}).

The exponential mapping $\exp \left( {\vec{\psi}} \right): \mathbb{R}^3 \longmapsto SO\left(3\right)$ is given by the Rodrigues' rotation formula (see for instance \cite{2013IJCV_Hartley_RotationAveraging}):
\begin{align}
    \label{Eq-exponential-map}
    \begin{split}
        \exp \left({\vec \psi} \right) = {\vec I}_3 + \frac{\sin \left\| {\vec \psi} \right\|}{\left\| {\vec \psi} \right\|} \hat {\vec \psi} + \frac{1 - \cos \left\| {\vec \psi} \right\|}{\left\| {\vec \psi} \right\|^2} \hat {\vec \psi}^2
    \end{split},
\end{align}
where $\hat {\vec \psi}$ is the skew-symmetric matrix of ${\vec \psi}$. 

\begin{figure}[t]
    \centering
    \includegraphics[width=\hsize]{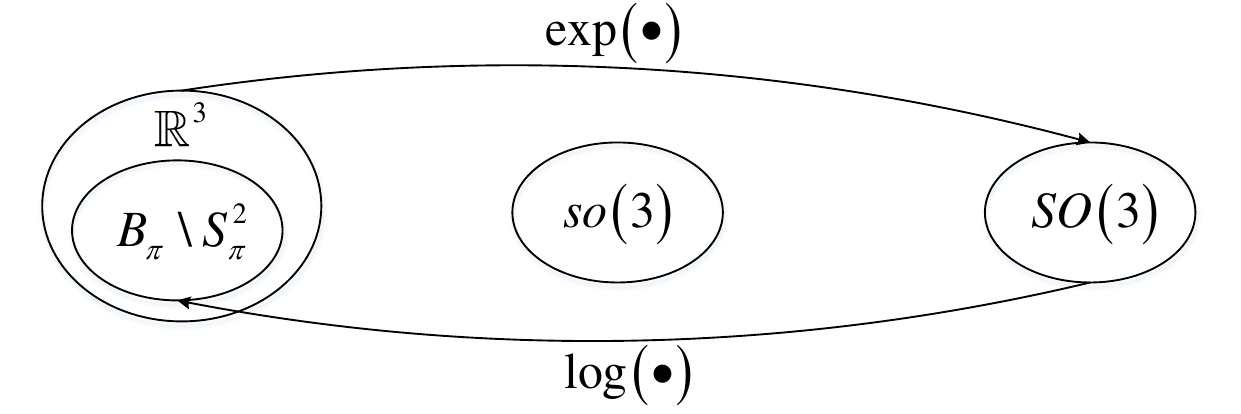}
    \caption{ Our exponential and logarithm mappings are redefined between the Angle-Axis Space $\vec{B}_\pi \backslash \mathcal{S}^2_\pi$ and SO(3), while the intermediate set $so\left(3\right)$ is omitted. }
    \label{fig-mapping}
\end{figure}

The logarithm mapping $\log \left(\vec{R}\right): SO\left(3\right) \longmapsto \vec{B}_\pi \backslash \mathcal{S}^2_\pi$ is redefined as (\cite{Gaofeng-RAL2018, 2023TOH-GaofengLi-IOM}):
\begin{align}
    \label{log-map}
    \begin{split}
        \log \left( {\vec R} \right) = \left\{ {\begin{array}{*{20}{l}}
                                            \bf{0},  & {\vec R} = {\vec I}_3,\\
                                            {\vec \psi}', & {\vec R} \neq {\vec I}_3 \& \left\| {\vec \psi} \right\| = 0, \\
                                            \arccos \left(\frac{tr\left( {\vec R} \right)-1}{2}\right){\frac{ {\vec \psi} }{\left\| {\vec \psi} \right\|}}, & {\text{Otherwise}},
                                       \end{array}}
                               \right.
    \end{split}
\end{align}
where ${\vec \psi} = \left[\psi_x, \psi_y, \psi_z\right]^T$ is computed by solving the equation:
\begin{align}
    \label{omega}
    \begin{split}
        \frac{1}{2} \left( {\vec R} - {\vec R}^\top\right) = \left[ {\begin{array}{*{20}{c}}
                                                    0                  &{ - {\psi _z}}            &{{\psi _y}}\\
                                                    {{\psi _z}}        &0                           &{ - {\psi _x}}\\
                                                    { - {\psi _y}}     &{{\psi _x}}               &0
                                                \end{array}}
                                         \right]
    \end{split},
\end{align}
and ${\vec \psi}' = \left[\psi'_x, \psi'_y, \psi'_z\right]^T$ is computed by solving the equations:
\begin{align}
    \label{omega-pi}
    \begin{split}
        {\vec R} = {\vec I}_3 + \frac{2}{\pi^2}\hat {\vec \psi}'^2 & \quad \text{and} \quad \left\| {\vec \psi}' \right\| = \pi
    \end{split}.
\end{align}

Here we add the constraint $\left\| {\vec \psi}' \right\| = \pi$ such that the mapping result is on the positive half sphere of the ball ${\vec B}_\pi$, instead of the negative half sphere $\mathcal{S}^2_\pi$.

Since the skew-symmetric matrix is little used in this paper, the exponential and logarithm mappings are redefined as mappings between the Angle-Axis Space $\vec{B}_\pi \backslash \mathcal{S}^2_\pi$ and the rotation group SO(3), as illustrated in Fig. \ref{fig-mapping}, instead of the traditional ones between $so\left(3\right)$ and SO(3).

\subsection{Discussions on the Angle-Axis Space}
\label{subsect-discussion-angle-axis}

The rotation group SO(3) is both a Lie Group and a Riemannian manifold. Given the angular distance metric: 
\begin{align}
    \label{Eq_AngleDistance}
    \begin{split}
        d\left({\vec R}_1, {\vec R}_2\right) = \left\| \log\left( {\vec R}^\top_1 {\vec R}_2 \right) \right\|,
    \end{split}
\end{align}
the rotation group SO(3) can also be regarded as a Riemannian geometric space. From the geometric aspect, the Angle-Axis Space $\vec{B}_\pi \backslash \mathcal{S}^2_\pi \subset \mathbb{R}^3$ is actually a local coordinate chart of SO(3), since the 3-d manifold SO(3) is conceptualized as a set of points that locally, but not globally, resemble the Euclidean space $\mathbb{R}^3$. Owing to the non-linear property, any local charts would encounter the distortion problem since they are only local parameterizations of the manifold. However, based on the following reasons, we believe that the Angle-Axis Space is more suitable for imitation learning that incorporates local constraints.

First, the Angle-Axis Space provides a minimal representation for orientations, making it easy to encode and decode the demonstrations. As is well known, SO(3) is a 3-d manifold. However, the rotation matrices (contain 9 entries) and the unit quaternions (contain 4 entries) both require additional constraints, e.g., the determinant of a rotation matrix has to be 1 and the norm of a quaternion has to be unit. This property makes the IL frameworks difficult to directly encode and decode the demonstration data. As for the fixed-angles and Euler-angles, although only three angles are used to represent an orientation, they have to specify the rotational sequence of the principal axes, e.g., Z-Y-Z type. These properties make the fixed-angles and Euler-angles unsuitable to encode and decode the demonstrations. In contrast, the Angle-Axis Space, which is a local diffeomorphism of the manifold SO(3) at a neighborhood of the given point ${\vec I}_3$, uses only 3 independent entries to represent an orientation. Although it would encounter the distortion problem for points that are far from the ${\vec I}_3$, the demonstration data can be directly encoded and decoded in the Angle-Axis Space.

\begin{figure}[h]
    \centering
    \includegraphics[width=\hsize]{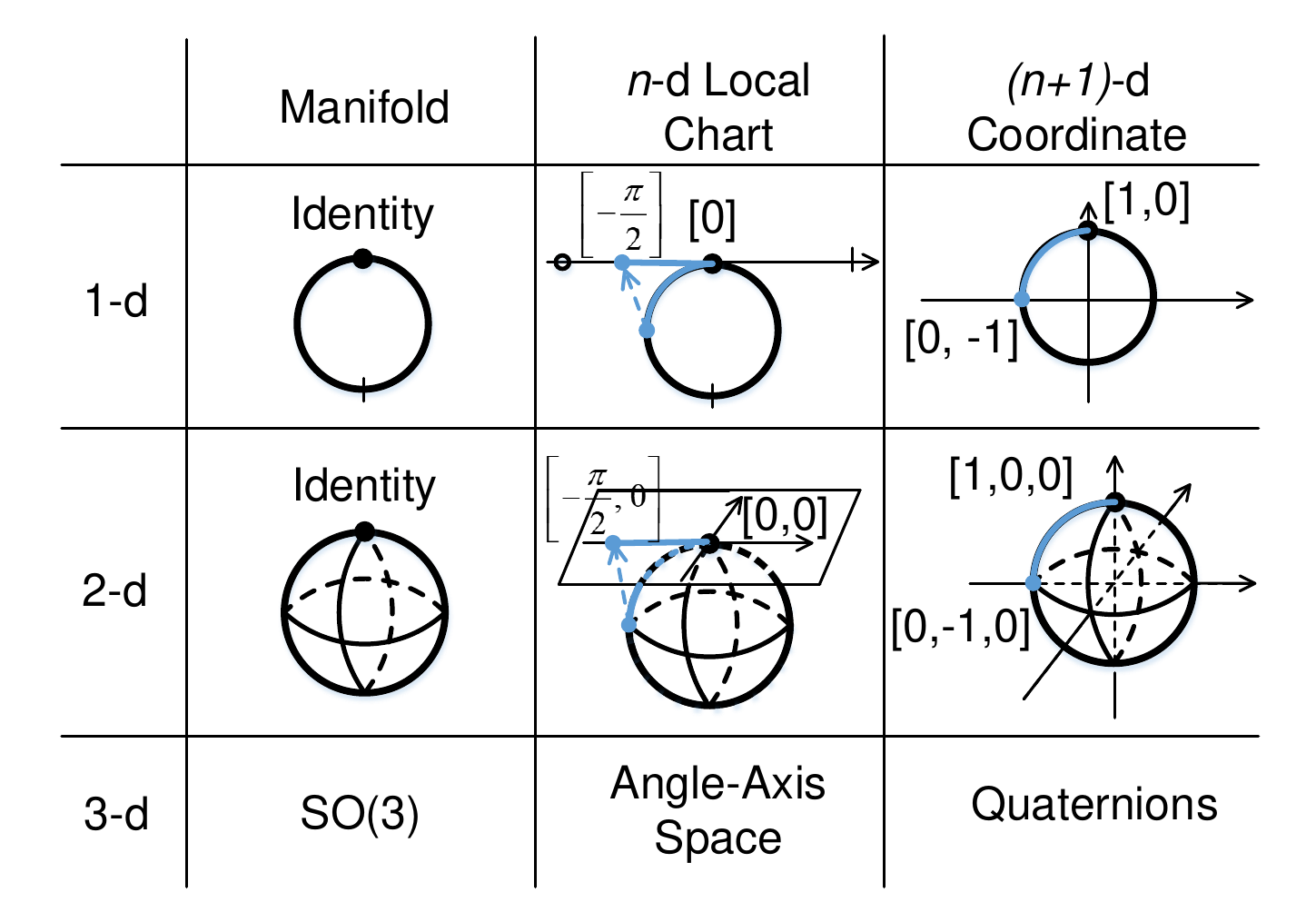}
    \caption{Comparison between the Angle-Axis Space and the quaternion space. Analogous to $\mathcal{S}^1$ and $\mathcal{S}^2$ cases, the Angle-Axis Space  is a tangent space of SO(3) around the identity ${\vec I}_3$. While the quaternions are to represent the $\mathcal{S}^n$ in a $(n+1)$-d coordinate system, where $n$ is the dimension of the manifold. The Euclidean distance of the quaternions cannot directly characterize the geodesic distance of orientations.}
    \label{Fig-comparisons-AngleAxis-Quaternion}
\end{figure}

Specifically, here we give a comparison between the Angle-Axis Space and the quaternion space. For better illustration, we give the spheres in 1-d and 2-d cases as examples\footnote{Since the 3-d hypersphere is difficult to be depicted explicitly with a graph, we would use 2-d spheres for illustration in the subsequent sections of this paper.}. As shown in Fig. \ref{Fig-comparisons-AngleAxis-Quaternion}, the sphere $\mathcal{S}^n$ can be smoothly mapped to a subset of $\mathbb{R}^n$, where $n$ is the dimension of the manifold. There are two ways to represent the points in the manifold $\mathcal{S}^n$. One is to use the $n$-d local chart and the other one is to represent the $\mathcal{S}^n$ in a $(n+1)$-d coordinate system. Similar to the 1-d and 2-d cases, the geometric structure of SO(3) is diffeomorphic to $\mathbb{RP}^3$, which is locally indistinguishable from $\mathcal{S}^3$\footnote{Strictly speaking, $SO(3) \cong \mathbb{RP}^3$, but since all our computations are local (tangent-space based), using $\mathcal{S}^3$ as a local model is harmless and does not affect any numerical result.}, although it is difficult to be depicted explicitly with a graph. The Angle-Axis Space is obviously a local diffeomorphism of SO(3) at the neighborhood around the identity ${\vec I}_3$, or, a local extension of SO(3) around the ${\vec I}_3$. In other words, the Angle-Axis Space is a tangent space of SO(3), i.e., $\mathcal{T}_{{\vec I}_3}SO(3)$. In contrast, the quaternions are to represent the $\mathcal{S}^n$ in a $(n+1)$-d coordinate system. Therefore, the quaternions require an additional constraint that their norms have to be unit, which brings difficulty to the encoding and decoding of demonstrations. 

\begin{figure}[h]
    \centering
    \subfigure[]{\includegraphics[width = 0.49\hsize]{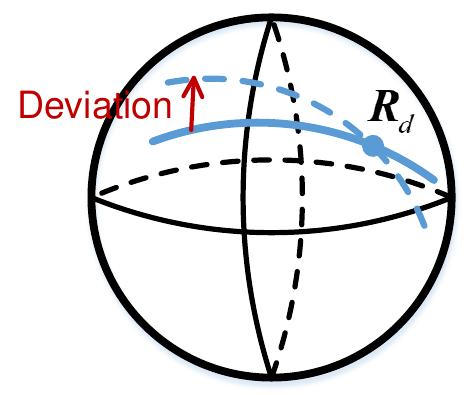}}
    \subfigure[]{\includegraphics[width = 0.49\hsize]{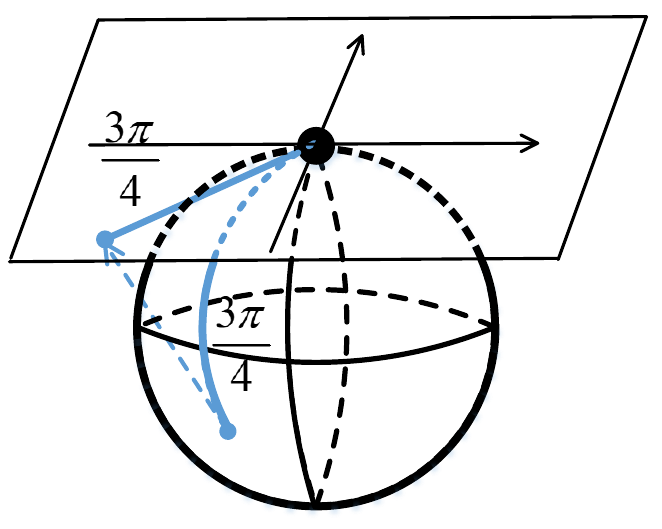}}
    \caption{ (a) The larger variance on the relaxed DoF drives the reproduced orientation to deviate from the desired geodesic (the blue and solid line), due to distance distortion. This results in the failure in guaranteeing the strict constraints on the other DoFs. (b) The distances relative to the basepoint are preserved. }
    \label{Fig_IllustrationForDistanceDistortion}
\end{figure}

Second, the geodesic distance in the manifold and the Euclidean distance in Angle-Axis Space, relative to the basepoint ${\vec I}_3$, are aligned with no distortion, making the Angle-Axis Space suitable to incorporate the local constraint on ${\vec I}_3$. As aforementioned, it is expected to apply a loose constraint on one DoF, while setting strict constraints on the other DoFs for a task with IOC, i.e., the IOC requires applying a larger variance on the relaxed DoF and small variances on the other DoFs. In this way, the desired via-point is transformed from a single point to a geodesic in SO(3), which is formed by the locus of ${\vec R}_d$ rotating around the relaxed DoF, denoted by $\gamma_{{\vec R}_d, {\vec \kappa}}\left(\theta\right)$. As shown in Fig. \ref{Fig_IllustrationForDistanceDistortion}(a), the strict constraints require the desired point to ``stay" on $\gamma_{{\vec R}_d, {\vec \kappa}}\left(\theta\right)$ (the blue and solid line), instead of a single point. However, the released constraint allows the desired via-point to slide along $\gamma_{{\vec R}_d, {\vec \kappa}}\left(\theta\right)$.
But the tangent spaces intrinsically distort distances. Therefore, applying distance-based learning algorithms on tangent spaces unavoidably leads to distorted models as they overlook the intrinsic geometry of the manifold. In addition, the distortions are more pronounced for data lying far away from the basepoint. As a result, if a larger variance is applied to one DoF of the orientation than the other DoFs, the distortion can drive the reproduced orientation to deviate from the desired geodesic (as illustrated by the blue dashed line in Fig. \ref{Fig_IllustrationForDistanceDistortion}(a)), resulting in the failure in guaranteeing the strict constraints on the other DoFs.

Fortunately, a good property is that the distances with respect to the basepoint are preserved in tangent spaces, as shown in Fig. \ref{Fig_IllustrationForDistanceDistortion}(b). And the Angle-Axis Space is directly a tangent space of SO(3) extended at the basepoint ${\vec I}_3$. This means that the injectivity radius, which quantifies the extent to which geodesics behave like straight lines in $\mathbb{R}^3$, can approach $\pi$ for any geodesics that pass through the basepoint ${\vec I}_3$. As is well known, the distance between any two orientations is no more than $\pi$ radians. Based on this property, it is feasible to apply a large variance to a given DoF of the via-point without worrying about the distance distortion problem. Therefore, the Angle-Axis Space is more suitable to incorporate the local constraint on ${\vec I}_3$.

\begin{figure}[h]
    \centering
    \subfigure[]{\includegraphics[width = \hsize]{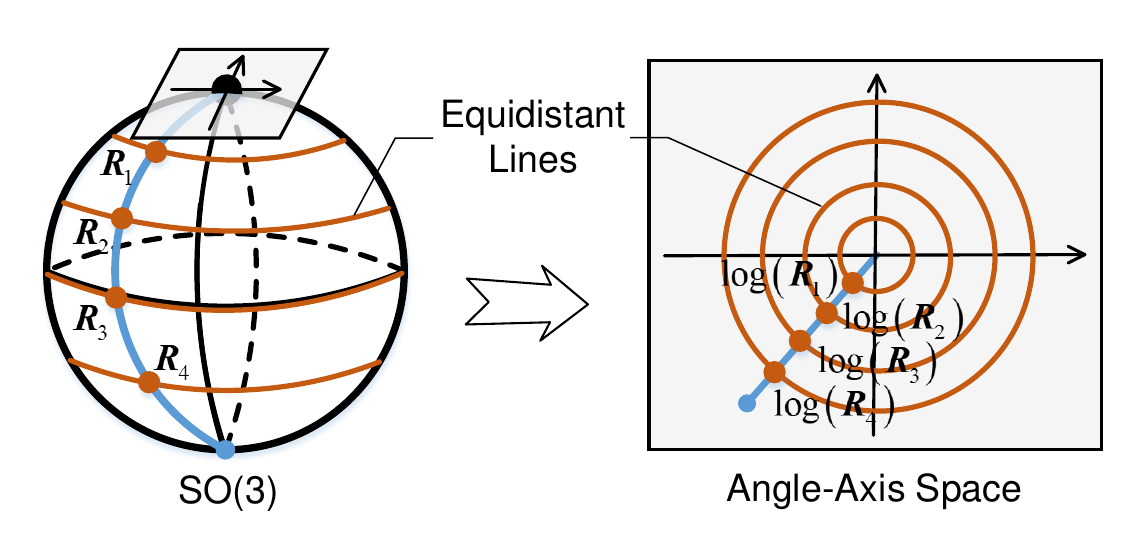}}
    \subfigure[]{\includegraphics[width = \hsize]{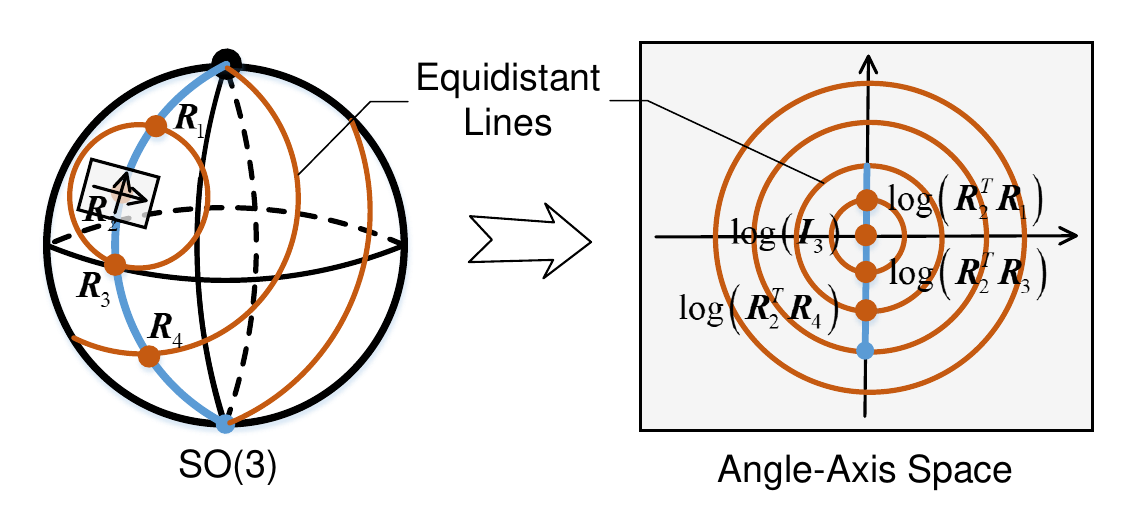}}
    \caption{(a) The original Angle-Axis Space is a local extension of SO(3) around the identity. (b) By projecting orientations to an auxiliary frame ${\vec R}_2$, SO(3) can be extended around the basepoint ${\vec R}_2$.}
    \label{Fig_ExtensionAtDifferentBasepoints}
\end{figure}

Third, it is easy to consider local constraints at any basepoint by projecting demonstrations to a local frame. As shown in Fig. \ref{Fig_ExtensionAtDifferentBasepoints}(a), the original Angle-Axis Space is a local extension of SO(3) around the identity. In order to consider local constraints at ${\vec R}_2$, we need to obtain the tangent space extended around ${\vec R}_2$. To achieve this, we can choose ${\vec R}_2$ as an auxiliary frame and project all orientation data into this local frame by using:
\begin{align}
    \label{Eq_localTransformation}
    \begin{split}
        \bar{\vec R}_i = {\vec R}^\top_2 {\vec R}_i,
    \end{split}
\end{align}
where ${\vec R}_i$ is any given orientation. Then $\bar{\vec R}_i$ can be projected into the Angle-Axis Space to incorporate the local constraints, as shown in Fig. \ref{Fig_ExtensionAtDifferentBasepoints}(b). The reproduced orientations can be transformed back to the base frame by left-multiplying ${\vec R}_2$.

A similar operation is adopted in previous work (\cite{2020TRO-KMP-Orientation}), in which a logarithmic map\footnote{Please note that this logarithmic map is defined for quaternions, different from ours, which is defined for rotation matrices. Unless specified, the logarithmic map used in this paper is defined by Eq. (\ref{log-map}).} $\log\left(\cdot\right): \mathbb{R}^4 \longmapsto \mathbb{R}^3$ and an auxiliary quaternion ${\vec q}_a$ are used to project the quaternion demonstrations into a Euclidean space,
\begin{align}
    \label{eq-log-of-quaternion}
    \begin{split}
        \log\left({\vec q}_1 * \bar{\vec q}_a\right) &= \log\left({\vec q}\right) \\
                &= \left\{
                        {\begin{array}{*{20}{l}}
                            \arccos(v)\frac{\vec u}{\left\|{\vec u}\right\|}, &  {\vec u} \neq {\vec 0},\\ 
                            \left[0, 0, 0\right]^\top, &   {\text{otherwise}},
                       \end{array}}
              \right.
    \end{split}
\end{align}
where ${\vec q} = \left[v, {\vec u}\right]^\top$ is the quaternion product of a given quaternion ${\vec q}_1$ and the conjugation of the auxiliary quaternion, i.e., $\bar{\vec q}_a$. However, the physical meaning of the auxiliary quaternion ${\vec q}_a$ and the resultant Euclidean space are ambiguous. By using Angle-Axis Space, the auxiliary frame has better interpretability, i.e., the auxiliary frame is a new base frame to extend the tangent space at different basepoints. As shown in Fig. \ref{Fig_ExtensionAtDifferentBasepoints}(b), for any orientation $\bar{\vec R}_i = {\vec R}^\top_2 {\vec R}_i$, the point $\log\left( \bar{\vec R}_i \right)$  in $\vec{B}_\pi \backslash \mathcal{S}^2_\pi$ gives the rotation axis and angle from ${\vec R}_2$ to ${\vec R}_i$. In addition, although we have no strict proof, we believe that the resultant Euclidean space by using Eq. (\ref{eq-log-of-quaternion}) is equivalent to our proposed Angle-Axis Space.

% {\color{red}\sout{Compared with the quaternion representation, the Angle-Axis Space exhibits the following advantages, which make it more suitable for imitation learning:}}

% \begin{itemize}
%     \item {\color{red}\sout{The Angle-Axis Space is more concise. The rotation group SO(3) is a 3-dimensional manifold. However, 4 parameters is required by using quaternions. In contrast, the Angle-Axis Space, which is bijective to SO(3), is a 3-dimensional space.}}
%     \item {\color{red}\sout{The Angle-Axis Space has better interpretability. The Angle-Axis Space is a local coordinate chart of SO(3) extended around the identity $\vec{I}_3$. For any element ${\vec \psi} \in \vec{B}_\pi \backslash \mathcal{S}^2_\pi$, its physical meaning is clear: the orientation $\exp\left({\vec \psi}\right)$ can be obtained by rotating around the direction ${\vec \psi}$ by an angle $\left\|{\vec \psi}\right\|$ from $\vec{I}_3$. In contrast, all the unit quaternions form a 4-dimensional sphere, whose geometric nature is non-straightforward.}}
%     \item {\color{red}\sout{The 3 entries of an element ${\vec \psi} \in \vec{B}_\pi \backslash \mathcal{S}^2_\pi$ are independent. As a result, it is feasible to apply different variances to the 3 entries of ${\vec \psi}$. However, the quaternion has to be an unit vector. Its 4 entries are dependent. Thus it is intractable to apply different variances to its different entries.}}
% \end{itemize}

\section{Learning and Adaptation in Angle-Axis Space}
\label{Section_Methods}

% As well known, the rotation group SO(3) is a 3-dimensional manifold. To capture the motor skills encapsulated in the multiple demonstrations, we propose to transform the orientations into the Angle-Axis Space, which is a subset of the $\mathbb{R}^3$. Then the probabilistic modeling of the transformed trajectories can be retrieved via Gauss Mixture Model (GMM) and Gauss Mixture Regression (GMR). Subsequently, the distribution of transformed trajectories is fitted by a kernelized parametric trajectory based on KMP. Finally, the orientation adaptation with Via-Points/Incomplete-Orientation-Via-Points and angular acceleration/jerk constraints are explained.

After a comprehensive introduction to the Angle-Axis Space, we would like to discuss how to achieve orientation learning and adaptation with simultaneous incorporation of multiple local constraints.
To learn the orientation skills from multiple demonstrations, we propose to transform orientations into the Angle-Axis Space and model the probabilistic distribution of the transformed trajectories by Gaussian Mixture Model (GMM) and Gaussian Mixture Regression (GMR) (Section~\ref{subsection_probModelling}). After that, we employ KMP to learn the distribution of transformed trajectories (Section~\ref{subsection_trajRepro}), followed by orientation adaptation towards desired points (Section~\ref{subsection_adaptation}) and minimization of angular acceleration (Section~\ref{subsec:acc:mini}). Then the local constraints, i.e., the incomplete orientation constraints (IOCs), are handled in Section~\ref{subsec:incomplete}, Section~\ref{subsec:multiIOVPs} and Section~\ref{subsec-memorybasedAlgorithm}, in which our proposed Gaussian-based weighted average mechanism would be introduced.

\subsection{Probabilistic Modeling in Angle-Axis Space}
\label{subsection_probModelling}

Given $M$ demonstrations consisting of time and orientations, denoted by $\vec{D}_R = \{ \{t_{n,m}, \vec{R}_{n,m}\}^{N}_{n=1}\}^{M}_{m=1}$, where $N$ denotes the length of each demonstration. In order to model the distribution of demonstrations, we project them into the Angle-Axis Space and model the transformed trajectories instead.
Let us denote the transformed new trajectories as $D_\psi = \{ \{t_{n,m}, {\vec \psi}_{n,m}, \dot{{\vec\psi}}_{n,m}\}^{N}_{n=1} \}^{M}_{m=1}$, where
\begin{align}
    \label{Eq_LocalFrame}
    \begin{split}
        {\vec \psi}_{n,m} = \log\left( \vec{R}^\top_a \vec{R}_{n,m} \right),
    \end{split}
\end{align}
and $\dot{{\vec\psi}} \in \mathbb{R}^3$ is the derivative of ${\vec \psi}_{n,m} \in \vec{B}_\pi \backslash \mathcal{S}^2_\pi$. Here, a local frame $\vec{R}_a$ is introduced as an auxiliary frame. 
As stated in the previous section, the auxiliary frame can be used to obtain the tangent space extended at different basepoints. As revealed in \cite{2023TASE-GaofengLi-TrajIOC, Gaofeng-RAL2018, 2021TMech-GaofengLi-IOC}, for points in $\vec{B}_\pi \backslash \mathcal{S}^2_\pi$, the closer the points are to the boundary of $\vec{B}_\pi$, the larger the distance distortion between the manifold distance metric and the Euclidean distance metric. Therefore, except for considering local constraints, we can also choose a proper auxiliary frame to alleviate the distance distortion problem. To make all demonstration data closer to the center of $\vec{B}_\pi$, ${\mat R}_a$ can be chosen as the geometrical center of all demonstrations. But it may be time-consuming to find the geometrical center in real applications. Therefore,
%{\color{red}\sout{In order to alleviate the possible distortion problem, $\vec{R}_a$ can be chosen as any point in the demonstrations, e.g., the starting point of the first demonstration, such that all the transformed points are closer to the center of $\vec{B}_\pi$.}} 
$\mat{R}_a$ can be chosen as a point in the demonstrations, e.g., the starting point of the first demonstration, to achieve less distortion around $\mat{R}_a$. Multiple trajectories can be generated by choosing multiple auxiliary frames and can finally be fused into a smooth trajectory by using our proposed Gaussian-based weighted average mechanism. In this way, we do not need to worry about the distance distortion problem since each trajectory is only effective at a local time domain.

Similar to KMP, we use GMM to model the distribution of demonstrations, i.e., $\mathcal{P}(t,\vec{\psi},\dot{\vec{\psi}})$. Furthermore, we write ${\vec \eta} = [{\vec \psi}^\top \, \dot{{\vec\psi}}^\top]^\top \in \mathbb{R}^6$ and use GMR to retrieve a probabilistic reference trajectory $\vec{D}_r = \{t_n, \hat{\vec \mu}_n, \hat{\vec \Sigma}_n\}^N_{n=1}$ with $\mathcal{P}(\vec{\eta}_n|t_n)=\mathcal{N}(\vec{\mu}_n,\vec{\Sigma}_n)$. The reference trajectory in essence captures the probabilistic features of demonstrations, which can be exploited in the orientation learning. The details on the probabilistic modelling using GMM and GMR can be found in \cite{2016IntellSerRob-Tutorial-TPGMM, 2019IJRR-KMP}.

\subsection{Trajectory Reproduction using a Kernelized Approach}
\label{subsection_trajRepro}

Considering a parametric trajectory in the form of
\begin{align}
    \label{Eq_ParameterizedTraj}
    \begin{split}
        {\vec \eta}\left(t\right) = \vec{\Theta}^\top\left(t\right){\vec{\textmd{w}}} = \left[{\begin{array}{*{20}{l}}
                                                             {\vec \phi}^\top\left(t\right) & {\vec 0} & {\vec 0}\\
                                                                                {\vec 0} & {\vec \phi}^\top\left(t\right) & {\vec 0}\\
                                                                                {\vec 0} & {\vec 0} & {\vec \phi}^\top\left(t\right)\\ \hdashline
                                                                                \dot{\vec \phi}^\top\left(t\right) & {\vec 0} & {\vec 0}\\
                                                                                {\vec 0} & \dot{\vec \phi}^\top\left(t\right) & {\vec 0}\\
                                                                                {\vec 0} & {\vec 0} & \dot{\vec \phi}^\top\left(t\right)
                                                                            \end{array}}
                                                                        \right]{\vec{\textmd{w}}},
    \end{split}
\end{align}
where ${\vec \phi}\left(t\right) \in \mathbb{R}^B$ denotes the $B$-dimensional basis functions and ${\vec{\textmd{w}}} \in \mathbb{R}^{3B}$ is the weight vector. We can learn ${\vec{\textmd{w}}}$ by minimizing the sum of covariance weighted quadratic errors, i.e.,
\begin{align}
    \label{Eq_MeanObjective}
    \begin{split}
        J\left({\vec{\textmd{w}}}\right) =& \sum^{N}\limits_{n = 1} { \left( \vec{\Theta}^\top\left(t_n\right){\vec{\textmd{w}}} - \hat{\vec \mu}_n\right)^\top \hat{\vec \Sigma}^{-1}_n  \left( \vec{\Theta}^\top \left(t_n\right) {\vec{\textmd{w}}} - \hat{\vec\mu}_n\right) } \\
                                   & + \lambda ||{\vec{\textmd{w}}}||^2,
                                %   ^\top {\vec{w}},
    \end{split}
\end{align}
where $\lambda||{\vec{\textmd{w}}}||^2$ is a penalty term to mitigate overfitting, and $\lambda > 0$ is a positive constant.

By solving the cost function in (\ref{Eq_MeanObjective}) and employing the kernel trick to alleviate the basis functions $\vec{\phi}(t)$, we can predict $\vec{\eta}(t^*)$ for an arbitrary input $t^*$ using (see \cite{2019IJRR-KMP} for details)
\begin{align}
    \label{Eq_TrajPred}
    \begin{split}
        {\vec \eta}\left(t^*\right) = {\vec {\textmd{k}}}^* \left( {\vec {\textmd{K}}} + \lambda{\vec \Sigma}\right)^{-1}{\vec \mu},
    \end{split}
\end{align}
where
\begin{align}
    \label{equ:mu}
    \begin{split}
        {\vec \mu} &= \left[\hat{\vec \mu}^\top_1 \, \hat{\vec \mu}^\top_2 \, \cdots \, \hat{\vec \mu}^\top_N\right],
            \end{split}
\end{align}
\begin{align}
    \label{equ:Sigma}
    \begin{split}
        {\vec \Sigma} &= \text{blockdiag}\left(\hat{\vec \Sigma}_1, \hat{\vec \Sigma}_2, \cdots, \hat{\vec \Sigma}_N\right),
                    \end{split}
\end{align}
\begin{align}
    \label{equ:k:star}
    \begin{split}
           {\vec {\textmd{k}}}^* &= \left[{\vec {\textmd{k}}}\left(t^*, t_1\right) \, {\vec {\textmd{k}}}\left(t^*, t_2\right) \, \cdots \, {\vec {\textmd{k}}}\left(t^*, t_N\right)\right],
                \end{split}
\end{align}
\begin{align}
    \label{equ:K}
    \begin{split}
        {\vec {\textmd{K}}} &= \left[{\begin{array}{*{20}{c}}
                                          {\vec {\textmd{k}}}\left(t_1, t_1\right) & {\vec {\textmd{k}}}\left(t_1, t_2\right) & \cdots & {\vec {\textmd{k}}}\left(t_1, t_N\right) \\
                                          {\vec {\textmd{k}}}\left(t_2, t_1\right) & {\vec {\textmd{k}}}\left(t_2, t_2\right) & \cdots & {\vec {\textmd{k}}}\left(t_2, t_N\right) \\
                                          \vdots & \vdots & \ddots & \vdots \\
                                          {\vec {\textmd{k}}}\left(t_N, t_1\right) & {\vec {\textmd{k}}}\left(t_N, t_2\right) & \cdots & {\vec {\textmd{k}}}\left(t_N, t_N\right)
                                    \end{array}}
                              \right].
    \end{split}
\end{align}

In (\ref{equ:k:star})--(\ref{equ:K}), ${\vec {\textmd{k}}}(\cdot,\cdot)$ is defined by the kernel function, i.e., ${\vec {\textmd{k}}}\left(t_i, t_j\right) = k\left(t_i, t_j\right)\vec{I}_B$. Here, $\vec{I}_B$ is the $B$-dimensional identity matrix.

The predicted trajectory point ${\vec \eta}\left(t^*\right)$ includes both ${\vec \psi}\left(t^*\right)$ and ${\dot{\vec{\psi}}}\left(t^*\right)$. In fact, we can recover the orientation $\vec{R}\left(t^*\right)$  at time $t^*$ only using the predicted ${\vec \psi}\left(t^*\right)$, i.e.,
\begin{align}
    \label{Eq_OriRecover}
    \begin{split}
        \vec{R}\left(t^*\right) = \vec{R}_a \exp\left({\vec \psi}\left(t^*\right)\right).
    \end{split}
\end{align}
Here it does not matter that Eq. (\ref{Eq_TrajPred}) cannot guarantee ${\vec \psi}\left(t^*\right)$ to be within $\vec{B}_\pi \backslash \mathcal{S}^2_\pi$, since the exponential mapping ``$\exp\left(\cdot\right): \mathbb{R}^3 \rightarrow SO\left(3\right)$" is effective for all points in $\mathbb{R}^3$, as shown in Fig. \ref{fig-mapping}.

\subsection{Orientation Adaptation Towards Desired Via-Points}
\label{subsection_adaptation}

As discussed in many previous works (\cite{ude2014orientation, abu2015adaptation, 2020TRO-KMP-Orientation}), the capability of adapting the learned orientation skills towards desired points is important for robots in many settings. For example, in a table tennis robot, the success of returning a ball largely depends on the racket orientation. We now discuss orientation adaptations in Angle-Axis Space, where both orientation and angular velocity profiles can be modulated to pass through arbitrary desired points.

Given $H$ desired points consisting of orientations and angular velocities, denoted as $\tilde{\vec{D}}_d := \{\tilde{t}_h, \tilde{\vec{R}}_h, \tilde{\vec \omega}_h\}^H_{h=1}$, where $\tilde{\vec{R}}_h$ and $\tilde{\vec \omega}_h$ represent the desired orientation and angular velocity at time $\tilde{t}_h$. We first need to transform them into the Angle-Axis Space. Similar to the transformation of demonstrations in Section~\ref{subsection_probModelling}, for each desired orientation $\tilde{\vec R}_h$ we can transform it as
\begin{align}
    \label{Eq_desiredOriTransformation}
    \begin{split}
        \tilde{\vec \psi}_h = \log\left( \vec{R}^\top_a \tilde{\vec{R}}_h \right).
    \end{split}
\end{align}
In order to obtain the desired derivative $\dot{\tilde{\vec \psi}}_h$, we compute the desired orientation $\tilde{\vec{R}}_{\tilde{t}_h + \delta_t}$ at $\tilde{t}_h + \delta_t$ with $\delta_t > 0$ being a small constant. Note that the desired angular velocity $\tilde{\vec \omega}_h$ is a vector expressed in the world frame. We have
\begin{align}
    \label{Eq_newOri}
    \begin{split}
        \tilde{\vec{R}}_{\tilde{t}_h + \delta_t} = \tilde{\vec{R}}_h \exp\left( \tilde{\vec{R}}^\top_h \tilde{\vec \omega}_h \delta_t \right),
    \end{split}
\end{align}
where $\tilde{\vec{R}}^\top_h \tilde{\vec \omega}_h$ is the desired velocity expressed in the local frame $\tilde{\vec{R}}_h$. Then, the derivative $\dot{\tilde{\vec \psi}}_h$ is approximated by
\begin{align}
    \label{Eq_DesDeriv}
    \begin{split}
        \dot{\tilde{\vec \psi}}_h \approx \frac{ \log\left( \vec{R}^\top_a \tilde{\vec{R}}_{\tilde{t}_h + \delta_t}\right) - \tilde{\vec \psi}_h }{\delta_t}
    \end{split}
\end{align}
Until now, we have transformed the dataset $\tilde{\vec{D}}_d$ into $\tilde{\vec{D}}_\eta = \{\tilde{t}_h, \tilde{\vec \eta}_h\}^H_{h = 1}$ with $\tilde{\vec \eta}_h = [\tilde{\vec \psi}^\top_h \, \dot{\tilde{\vec \psi}}^\top_h]^\top \in \mathbb{R}^6$. In addition, for each desired point $\tilde{\vec \eta}_h$, we can design a covariance $\tilde{\vec \Sigma}_h \in \mathbb{R}^{6 \times 6}$ to control the precision of adaptations. Thus, we can obtain an additional probabilistic reference trajectory $\tilde{\vec{D}}_r = \{\tilde{t}_h, \tilde{\vec \eta}_h, \tilde{\vec \Sigma}_h\}^H_{h = 1}$ to indicate the transformed desired states.%Note that the covariance matrix is inversely proportional to the adaptation precision.

As suggested in \cite{2019IJRR-KMP}, the adaptation problem can be transformed into the problem of learning an extended reference trajectory corresponding to the concatenation of the initial reference trajectory and the desired points. Thus, in our case, we concatenate $\vec{D}_r$ and $\tilde{\vec{D}}_r$ to generate an extended reference database $\vec{D}^U_r = \vec{D}_r \cup \tilde{\vec{D}}_r$, where `$\cup$' stands for the union operation. By learning the extended reference trajectory $\vec{D}^U_r$ using the approach discussed in Section~\ref{subsection_trajRepro}, we can ensure that the adapted orientation trajectory goes through the desired orientation $\tilde{\vec R}_h$ with desired angular velocity $\tilde{\vec \omega}_h$ at time $\tilde{t}_h$, while the precision depends on the covariance $\tilde{\vec \Sigma}_h$ associated with each desired point.

\subsection{Angular Acceleration Constraints}
\label{subsec:acc:mini}

In many applications, it is required to make the reproduced trajectory as smooth as possible. Therefore, a practical orientation learning method is desirable to be able to consider the angular acceleration/jerk constraints. For our proposed method, it is straightforward to take into account the angular acceleration or jerk constraints, owing to the kernelized representation. For example, for the case of minimizing angular accelerations, we can simply replace $\vec{\mu}$ in (\ref{equ:mu}) with $\bar{\vec \mu} = \left[\bar{\vec \mu}^\top_1 \, \bar{\vec \mu}^\top_2 \, \cdots \, \bar{\vec \mu}^\top_N\right]$,
where  $\bar{\vec \mu}_n = \left[\hat{\vec \mu}^\top_n, {\vec 0}_{1\times3}\right]^\top$, $\hat{\vec \mu}^\top_n$ represents the mean of the reference trajectory extracted by using GMM/GMR, and ${\vec 0}_{1\times3}$ represents the additional components to consider angular acceleration. Then we can replace $\vec{\Sigma}$ in (\ref{equ:Sigma}) with $\bar{\vec \Sigma}= \text{blockdiag}\left(\bar{\vec \Sigma}_1, \bar{\vec \Sigma}_2, \cdots, \bar{\vec \Sigma}_N\right)$, where $\bar{\vec \Sigma}_n = \text{blockdiag}\left(\hat{\vec \Sigma}_n, \frac{1}{\lambda_a}I_3\right)$, $\hat{\vec \Sigma}_n$ represents the covariance of the reference trajectory extracted by using GMM/GMR, and $\frac{1}{\lambda_a}I_3$ represents the additional components to consider angular acceleration.
Here $\lambda_a > 0$ is a user-specified weight to balance the contradiction between the trajectory reproduction and the angular acceleration minimization. Note that the definition of ${\vec {\textmd{k}}}(\cdot,\cdot)$ would be altered since the parametric form in (\ref{Eq_ParameterizedTraj}) includes the second-order derivative of $\vec{\phi}(t)$ in order to account for the acceleration constraints. More details on angular acceleration minimization can be found in \cite{2020TRO-KMP-Orientation}.

%\left[{\begin{array}{*{20}{c}}
%                                        \hat{\vec \Sigma}_n & {\vec 0}\\
%                                        {\vec 0} & \frac{1}{\lambda_a}I
%                                    \end{array}}
%                            \right]

%\left[\begin{array}{c}
%                                        \hat{\vec \mu}_n\\
%                                        {\vec 0}
%                                    \end{array}\right]

\section{Adaptation with Incomplete Orientation Via Points (IOVPs) }
\label{sec-IOVPs}

In Section~\ref{subsect-discussion-angle-axis}, the benefits of the Angle-Axis Space in imitation learning with the incorporation of local constraints have been discussed. In this section, we would like to discuss the motivations and the way of incorporating local constraints.

\subsection{Adaptation with Single IOVP}
\label{subsec:incomplete}

In many tasks, it is expected to apply different precisions on different DoFs for a given via-point\footnote{In previous works, the desired points are classified as start-/via-/end-points, while in this paper we refer to them as via-points.}. We refer to this kind of via-point as an {\bf Incomplete-Orientation-Via-Point (IOVP)}. Taking the earphone's peg-in-hole task as an example, to fulfill the plugging task, the pin's orientation is expected to pass through a desired point $\vec{R}_d$ with a desired angular velocity ${\vec \omega}_d$ at $t_d$. However, as shown in Fig. \ref{Fig_Pingpong-IOVP}, we can also choose to plug into the hole at a different orientation $\vec{R}^{'}_{d}$ with the same angular velocity ${\vec \omega}_d$, in which $\vec{R}^{'}_{d}$ is obtained by rotating an angle around the pin's direction. More similar examples can also be found in the arc-welding tasks, drilling tasks, wheel-valve turning tasks, and so on. In these scenarios, it is easy to find a feasible orientation (e.g., $\vec{R}_d$) according to the task requirements. However, it is difficult and non-intuitive to determine which orientation is optimal. In other words, these tasks require the users/algorithm to manually/automatically specify a desired via-point with strict constraints on the rotation around $x$ and $y$ directions, but with a loose constraint on the rotation around $z$. In summary, the via-point is expected to have high precision along the $x$ and $y$ directions, while having a low precision in the $z$ direction.

The incorporation of IOVP can bring many benefits. For example, the given desired orientation $\vec{R}_d$ may become unreachable due to joint limits or emerging obstacles. By considering IOVP, it is also possible to find a feasible solution when $\vec{R}_d$ is unreachable, or we can generate a trajectory with smaller acceleration/jerk costs.

\begin{figure}[t]
    \centering
    \includegraphics[width = \hsize]{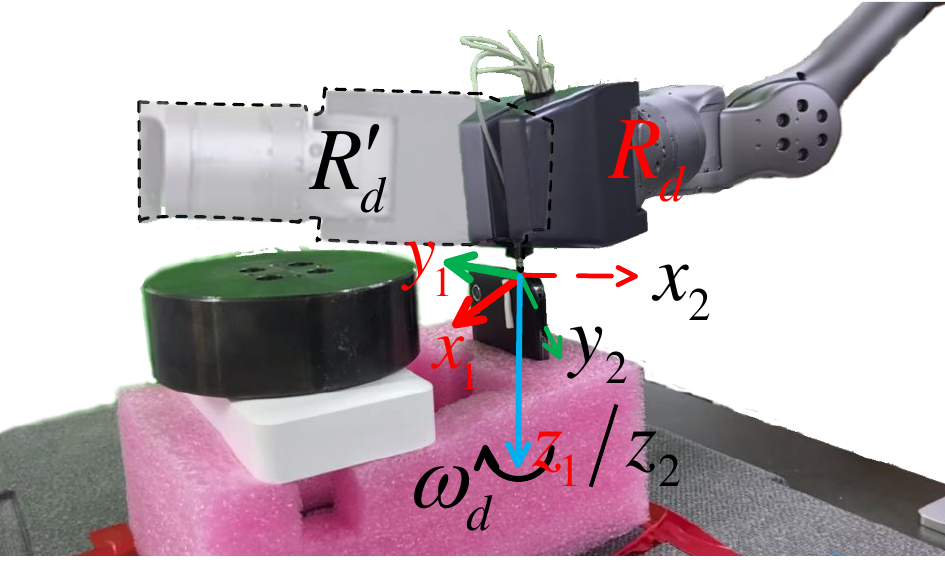}
    \caption{Illustration of the Incomplete-Orientation-Via-Points for a peg-in-hole task.}
    \label{Fig_Pingpong-IOVP}
\end{figure}

However, most existing methods are incapable of handling the IOVP. The quaternion-based method proposed in \cite{2020TRO-KMP-Orientation} is capable of adapting orientations towards desired via-points. But it is unable to assign different levels of precision to each DoF of the via-point, since the quaternion representation is subject to the constraint of unit norm and accordingly its four elements are coupled. Although quaternions are transformed into Euclidean space in \cite{2020TRO-KMP-Orientation}, the physical meaning of the transformed orientations in Euclidean space is ambiguous, making the individual precision control difficult. As a result, the release on one DoF would also cause a big error on other DoFs, which is contrary to our expectation for strict constraints on these DoFs, as illustrated in Fig. \ref{Fig_IllustrationForDistanceDistortion}(a).

% apply different precision to different DoFs of the same Via-Point, i.e. all the diagonal entries of $\tilde{\vec \Sigma}_h$ ($\tilde{\vec \Sigma}_h$ is a diagonal matrix) have to share the same values. This is because the 4 entries of a quaternion is non-independent and their sum of squares has to be 1. To guarantee the normalization of the quaternion, all quaternions are transformed into a 3-dimensional Euclidean space by introducing an auxiliary quaternion. However, the physical meaning of this Euclidean space is ambiguous. The 3 dimensions in the transformed Euclidean space cannot be associated with the 3 DoFs of orientations.

Based on our Angle-Axis Space-based framework, it is straightforward to deal with the IOVP. Let us denote the desired IOVP as $\{t_d, \vec{R}_d, {\vec \omega}_d\}$. Without loss of generality, we can assume that the end-effector rotating around the $z$ direction has no effects on the performance. As stated in Section \ref{Section_MathematicPre}, the Angle-Axis Space is the local chart of a local linear extension of SO(3) around the identity $\vec{I}_3$. In order to incorporate the IOVP, we extend the rotation group around $\vec{R}_d$. Thus, the local frame $\vec{R}_d$ can be chosen as the auxiliary frame, i.e., $\vec{R}_a = \vec{R}_d$. Then all demonstrations can be transformed into the auxiliary frame according to (\ref{Eq_LocalFrame}). The $\{t_d, \vec{R}_d, {\vec \omega}_d\}$ can also be transformed into the local frame via (\ref{Eq_desiredOriTransformation}) and (\ref{Eq_DesDeriv}). The covariance $\tilde{\vec \Sigma}_h$ can be designed as: $\tilde{\vec \Sigma}_h = \text{diag}\left(\varepsilon_s, \varepsilon_s, \varepsilon_b, \varepsilon_s, \varepsilon_s, \varepsilon_s\right)$, where $\varepsilon_s > 0$ is a very small constant to achieve high precision along the $x$ and $y$ directions, while $\varepsilon_b > 0$ is a large constant to release the constraint along the $z$ direction. Then the IOVP can be treated as a normal via-point and handled by following similar steps shown in Section \ref{subsection_adaptation}. As we can tell, the proposed approach is simple, straightforward, and effective, owing to the Angle-Axis Space-based representation method.

\subsection{Incorporation of Multiple IOVPs by using Weighted Average Mechanism in SO(3)}
\label{subsec:multiIOVPs}

% {\color{red}According to [1] Fallacy 1 and 2, "Off-the-shelf Euclidean learning
% approaches applied on a single tangent space disregard the injectivity
% radius", and "orientation kernelized movement primitives based on a
% single tangent space learn distorted models that can only be applied to
% data closely grouped on $S^3$". The proposed approach in this paper
% exactly falls into these categories, where one single tangent space
% (attached to the selected auxiliary frame) is used for the entire
% motion regardless of whether all points in a trajectory are close by or
% far away. The computation of velocity in such a single tangent space
% (eq 15) can introduce large distortion if some points in the trajectory
% are far away from the selected frame. I suggest the authors revise the
% method with the correct Riemannian methods and redo the experiments
% with the proposed IOC and IOVP.}

In Section~\ref{subsec:incomplete}, we have shown that the IOVP can be naturally handled by choosing ${\mat R}_d$ as the auxiliary frame and projecting all demonstrations and via-points into the auxiliary frame.  In order to incorporate multiple IOVPs, multiple trajectories are generated by projecting data into multiple tangent spaces. However, how to fuse these multiple trajectories into a smooth trajectory is still unexplored. The difficulties lie in two aspects. First, how to guarantee the dominance of each trajectory at the given local time domain. Second, how to guarantee the smoothness of the final reproduced trajectory. Although the KMP can guarantee the smoothness of each trajectory, the smoothness of the final trajectory is also affected by the fusing mechanism. In
%{\color{red}\sout{In previous subsection (Section~\ref{subsec:incomplete}), we have shown that the IOVP can be naturally handled by setting $\vec{R}_a = \vec{R}_d$ to transform all demonstrations and via-points into the local frame $\vec{R}_d$. However, how to incorporate multiple IOVPs is still unsolvable for existing approaches. While the simultaneous incorporation of multiple IOVPs is a very common requirements for many tasks. Hence in}}
this section, we would like to propose a Gaussian-based weighted average mechanism in SO(3) to smoothly fuse multiple trajectories.

Suppose the set of multiple IOVPs is given by $\bar{\vec{D}}_d := \{\bar{t}_k, \bar{\vec{R}}_k, \bar{\vec \omega}_k\}^K_{k=1}$, where $K$ is the number of IOVPs, $\bar{\vec{R}}_k$ and $\bar{\vec \omega}_k$ represent the desired IOVPs and angular velocity at time $\bar{t}_k$, respectively. By following Section~\ref{subsection_probModelling}$\sim$Section~\ref{subsec:incomplete}, it is easy to obtain $K$ reproduced orientation trajectories $\vec{R}_k \left(t\right)$, in which each $\vec{R}_k \left(\bar{t}_k\right)$ can satisfy the $k$-th IOC only at time $\bar{t}_k$, for $k = 1, 2, \cdots, K$. In order to reproduce a smooth trajectory that can satisfy all the IOCs simultaneously, we design a Gaussian-based weighted average mechanism to fuse the $K+1$ trajectories $\vec{R}_k \left(t\right), k = 0, 1, 2, \cdots, K$:
\begin{align}
    \begin{split}
    \label{Eq_FinalFusion}
        \vec{R}\left(t^*\right) &= W_0\left(t^*\right) \vec{R}_0\left(t^*\right) \oplus \left( \sum^{K}_{k=1}{W_k\left(t^*\right)} \right) {\vec R}_{I}\left(t^*\right),
    \end{split}\\
    \begin{split}
    \label{Eq_FusionAllIOVPs}
        {\vec R}_{I}\left(t^*\right) &=  W_1\left(t^*\right) \vec{R}_1\left(t^*\right) \oplus \cdots   \oplus W_K\left(t^*\right) \vec{R}_K\left(t^*\right),
    \end{split}\\
    \begin{split}
    \label{Eq_W0}
        W_0\left(t^*\right) &= 1-\sum^{K}_{k=1}{W_k\left(t^*\right)},
    \end{split}
\end{align}
where the time $t^*$ is a query input. $\vec{R}_0\left(t\right)$ is the reproduced baseline trajectory  and $W_0\left(t\right)$ is the weighted curve for the baseline trajectory, $\vec{R}_k\left(t\right)$ represents the $k$-th predicted trajectory by considering the $k$-th IOVP and $W_k\left(t\right)$ is the weighted curve for each $\vec{R}_k \left(t\right)$. The notation ``$W_i {\vec R}_i \oplus W_j {\vec R}_j$" represents the weighted average mechanism for any two points $\vec{R}_i$ and $\vec{R}_j$ in the rotation group SO(3), where $W_i$ and $W_j$ are the weights for $\vec{R}_i$ and $\vec{R}_j$, respectively. The baseline trajectory $\vec{R}_0\left(t\right)$ can be the predicted trajectory without considering any adaptations or by only considering the adaptation towards the starting via-point. Then the fused ${\vec R}_I\left(t\right)$ represents the weighted sum of the $K$ trajectories by considering the $k$-th IOVP, as shown in Eq. (\ref{Eq_FusionAllIOVPs}). Considering the non-associativity in non-Euclidean manifolds, the canonical ordering of the $K$ trajectories should be the same across the entire time domain. An easy and straightforward way is to fuse the multiple trajectories according to the IOVP timing $\bar{t}_k$, since the IOVP timing order is deterministic across the entire time domain. In addition, due to the local support property of the Gaussian weight curves $W_k\left(t\right)$ in the time domain, at any given time only one or two trajectories have significantly non-zero weights. Therefore, when fusing in chronological order, the contributions from trajectories far from the current time are minimal. Hence, chronological order is a reasonable and sufficiently good choice. Eq. (\ref{Eq_FinalFusion}) represents the weighted average of $\vec{R}_0\left(t\right)$ and ${\vec R}_I\left(t\right)$, where the weight for ${\vec R}_I\left(t\right)$ is the sum of weights for all the $K$ trajectories and the weight $W_0\left(t\right)$ for $\vec{R}_0\left(t\right)$ is given by Eq. (\ref{Eq_W0}).

\begin{figure}[b]
    \centering
    \includegraphics[width = \hsize]{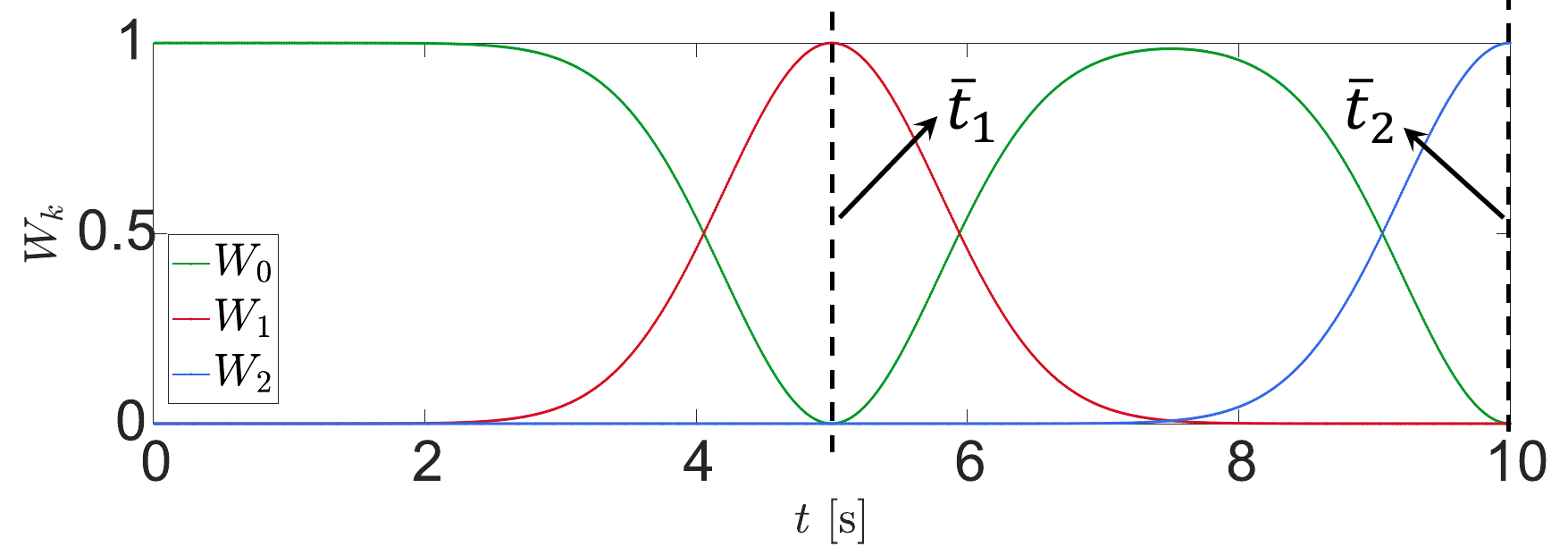}
    \caption{An example for the design of the weighted curve $W_k\left(t\right)$. In this example, $\Delta t = 2.4$s, $\bar{t}_1 = 5.0$s, $\bar{t}_2 = 10.0$s, $K = 2$.}
    \label{Fig_GaussianWeightCurve}
\end{figure}

Next, let's discuss the design of the weighted curve $W_k\left(t\right)$. In order to guarantee the smoothness of the final reproduced trajectory $\vec{R}\left(t\right)$, $W_k\left(t\right)$ is designed as a smooth Gaussian curve:
\begin{align}
    \label{Eq_GaussWeight}
    \begin{split}
        W_k\left(t\right) = \exp{\left( - \frac{ \left(t - \bar{t}_k \right)^2 }{2\sigma^2_k} \right)},
    \end{split}
\end{align}
where $\sigma_k = \Delta t_k / 3$ is a parameter to control the effective domain of the $k$-th trajectory $\vec{R}_k\left(t\right)$ and $\Delta t_k$ is the domain width. Please note that the curve $W_k\left(t\right)$ is designed for time-driven trajectories. However, the KMP framework can also support high-dimensional input-driven trajectory (\cite{2019IJRR-KMP, 2020TRO-KMP-Orientation}). In this case, the weighted curve can be replaced by:
\begin{align}
    \label{Eq-Wk-s-driven}
    \begin{split}
        W_k\left({\vec s}\right) = \exp\left( -\frac{ d^2\left({\vec s}, \bar{\vec s}_k\right)}{2\delta^2_k} \right),
    \end{split}
\end{align}
where the time query input $t$ is replaced by the high-dimensional input $\vec s$, $d\left({\vec s}, \bar{\vec s}_k\right)$ gives the distance between ${\vec s}$ and $\bar{\vec s}_k$ in the input space, and $\delta_k = \Delta d_{\vec s} / 3$ is a parameter to control the effective domain of the $k$-th trajectory $\vec{R}_k\left({\vec s}\right)$ with $\Delta d_{\vec s}$ being the domain width. For example, if the high-dimensional input is the robot's position, $d\left({\vec s}, \bar{\vec s}_k\right) := \left\|{\vec s} - \bar{\vec s}_k\right\|_{L_2}$ gives the Euclidean distance between two positions. In this paper, we would like to focus on the time-driven trajectory.

According to the property of the Gaussian function, $W_k\left(t\right)$ is close to zero when $t$ is out of the scope of $\left[\bar{t}_k - \Delta t_k, \bar{t}_k + \Delta t_k\right]$. The principle in choosing $\Delta t$ is to guarantee:
\begin{align}
    \label{Eq_DeltaTPrinciple}
    \begin{split}
        \bar{t}_k + \Delta t_k \leq \bar{t}_{k+1} \quad \& \quad  \bar{t}_{k} - \Delta t_{k} \geq  \bar{t}_{k-1},
    \end{split}
\end{align}
for all $k = 2, 3, \cdots, K-1$, such that the adjacent IOVPs do not interfere with each other. An example for the design of $W_k\left(t\right)$ is given in Fig. \ref{Fig_GaussianWeightCurve}. Certainly, this design is still imperfect, since the final reproduced orientation trajectory may encounter fluctuations when the two adjacent IOVPs are too close in the querying space while the geodesic distance between the two desired orientations is large. In this case, the final trajectory must satisfy two distinct orientation constraints within a very short interval. This can pose challenges for smooth robot execution. To address this issue, a post-processing step (e.g., increasing the time scale as a relaxation in non-critical applications) remains an interesting direction to improve our design. But the current design is good enough for tasks with sparse IOVPs.

Finally, let's discuss the weighted average mechanism in non-Euclidean space. Different from the Euclidean space, the weighted average mechanism in non-Euclidean spaces is non-trivial. Although several rotation average problems, e.g., the Fr\'{e}chet Mean in SO(3), have been reported in \cite{2013IJCV_Hartley_RotationAveraging, 2010MTNS_Hartley_RotationAveraging}, the weighted average in SO(3) is still missing in existing literature, according to our limited knowledge. Only in \cite{2021INFOTEH-Kapic-WeightedRotationAveraging}, an iterative algorithm, called KLW algorithm, was presented for weighted rotation average, by following the Karcher flow. These algorithms seek to find the Riemannian Center of Mass that achieves the minimal weighted distance to all given orientations. However, the iterative Karcher flow is time-consuming. It needs $K$ logarithmic operations and 1 exponential operation in each iteration, while the iteration count is non-deterministic. Its time complexity is given by $m*c*\left(K+1\right)$, where $m$ is the length of the trajectory, $c$ is the iteration count for each point in the trajectory, and $K$ is the total number of the IOCs. More importantly, the discontinuity problem, which would be discussed later in Section~\ref{subsec-memorybasedAlgorithm}, is not addressed. Our motivation is to fuse the multiple trajectories generated in different tangent spaces into a smooth trajectory. Therefore, we hope that the weighted averaging mechanism satisfies: 1) the dominance of each trajectory at the given local time domain, and 2) the smoothness of the final reproduced trajectory. Instead, the minimization of weighted distance to given orientations is not so important. 

Here we propose to achieve weighted average in a Riemannian space by utilizing the geodesic curve (as illustrated in Fig. \ref{Fig_weightedAverage}):
\begin{align}
    \label{Eq_weightSumInRiemannian}
    \begin{split}
        W_i {\vec R}_i \oplus W_j {\vec R}_j &= \vec{R}_i\exp\left(d {\bar {\vec \psi}} \right),
    \end{split}
\end{align}
where ${\bar {\vec \psi}} = \frac{1}{d\left({\vec R}_i, {\vec R}_j\right)}\log\left({\vec R}^\top_i {\vec R}_j\right)$ represents the traversal direction. This traversal direction indicates that the weighted average result of $\vec{R}_i$ and $\vec{R}_j$ lies on the geodesic curve from $\vec{R}_i$ to $\vec{R}_j$. And the distance of the averaging result to $\vec{R}_i$ is given by $d$, which is determined by $W_i$, $W_j$, and $d\left({\vec R}_i, {\vec R}_j\right)$:
\begin{align}
    \label{Eq_distance}
    \begin{split}
        d &= \frac{W_j}{W_i + W_j} d\left({\vec R}_i, {\vec R}_j\right),
    \end{split}
\end{align}
where  $d\left({\vec R}_i, {\vec R}_j\right) = \left\|\log\left({\vec R}^\top_i {\vec R}_j\right)\right\|$ is the geodesic distance from ${\vec R}_i$ to ${\vec R}_j$.

\begin{figure}[h]
    \centering
    \includegraphics[width = 0.85\hsize]{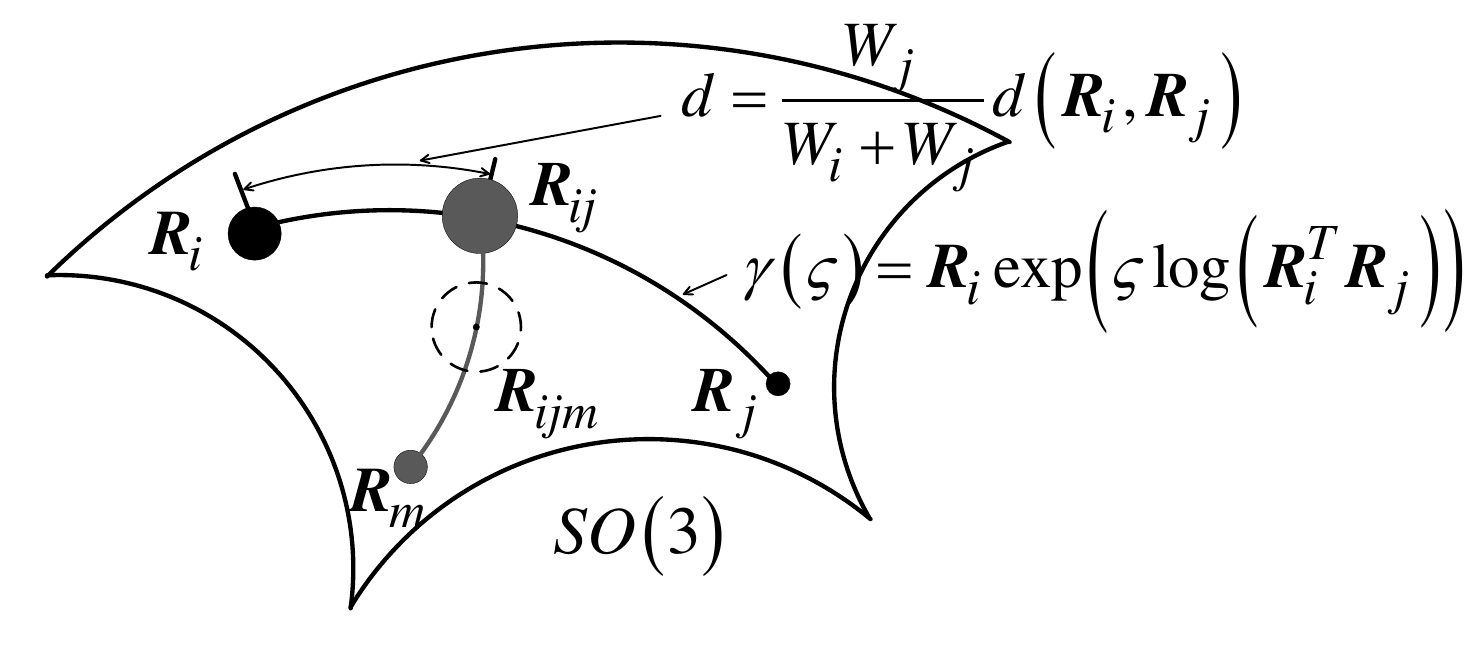}
    \caption{Illustration of the weighted average mechanism in SO(3). As shown, $\gamma\left(\varsigma\right)$ ($\varsigma \in \left[0,1\right]$ is a varying parameter) gives the geodesic curve from $\vec{R}_i$ to $\vec{R}_j$. The weighted average result $\vec{R}_{ij}$ is a point away from $\vec{R}_i$ with a distance $d$. Here the points with different weights are graphically represented as circles with different radii.}
    \label{Fig_weightedAverage}
\end{figure}

Let's denote $\vec{R}_{ij} := W_i {\vec R}_i \oplus W_j {\vec R}_j$. Obviously, when $W_i = W_j$, $\vec{R}_{ij}$ is in the middle of the geodesic curve; when $W_i = 0$, we have $\vec{R}_{ij} = \vec{R}_j$; when $W_j = 0$, we have $\vec{R}_{ij} = \vec{R}_i$. These results are consistent with our intuition.

To make the weighted sum of multiple points possible, we can assign the weight $W_i + W_j$ to $\vec{R}_{ij}$. Then $\vec{R}_{ij}$ can be regarded as a new point with a new weight $W_i + W_j$, such that it can be averaged with the next point $\vec{R}_m$. By continuing this process, the weighted sum of $K$ points, as shown in Eq. (\ref{Eq_FusionAllIOVPs}), can be achieved. Here we assume that the sum of all weights is equal to 1. If not, all the weights can be normalized to guarantee that the sum is 1.

\subsection{A Memory-based Algorithm to Address the Discontinuity Issue in  Weighted Rotation Average Mechanism}
\label{subsec-memorybasedAlgorithm}

\begin{figure}[b]
    \centering
    \includegraphics[width = 0.95\hsize]{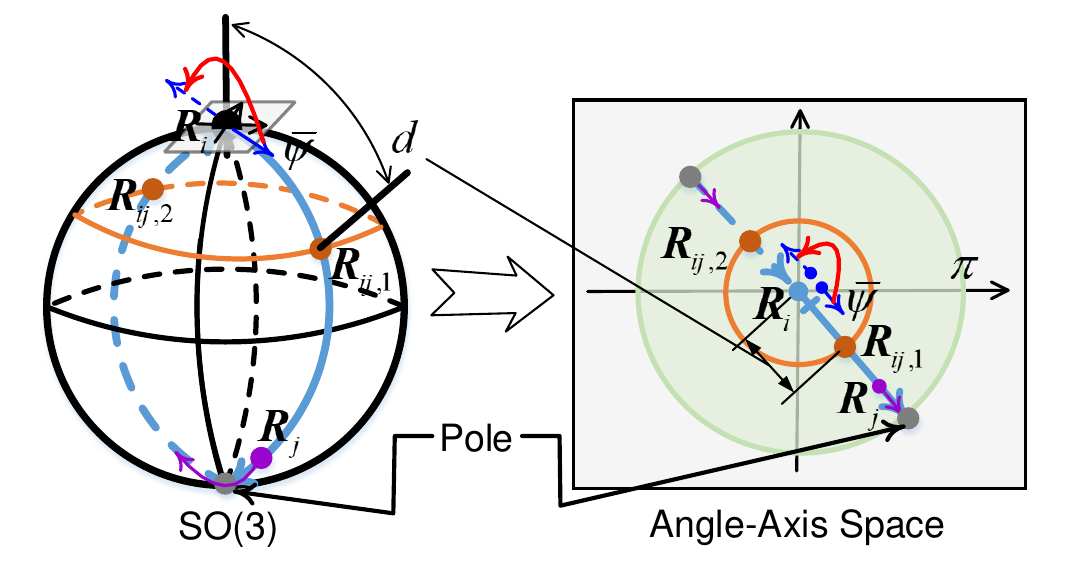}
    \caption{ Illustration of the discontinuity issue in weighted averaging mechanism. When ${\vec R}_j$ crosses over the pole, the direction ${\bar {\vec \psi}}$ would flip to the opposite direction, leading to the `jumping' of the averaging result from ${\vec R}_{ij, 1}$ to ${\vec R}_{ij, 2}$. }
    \label{Fig_FlippingProblem}
\end{figure}

\begin{figure*}[hb]
    \centering
    \subfigure[]{\includegraphics[width = 0.45\hsize]{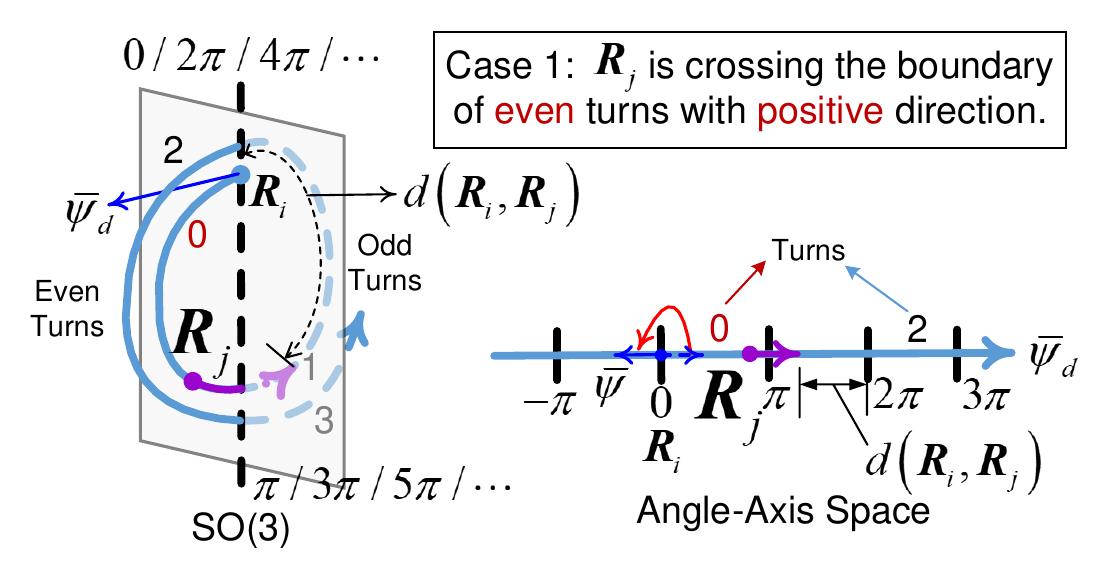}} \hspace{30pt}
    \subfigure[]{\includegraphics[width = 0.45\hsize]{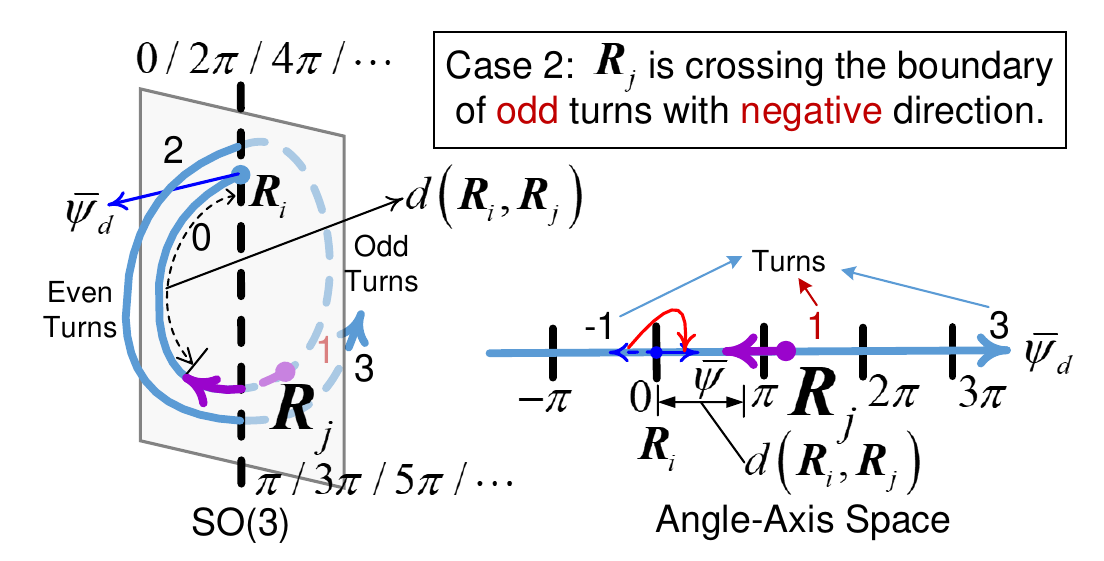}}
    \subfigure[]{\includegraphics[width = 0.45\hsize]{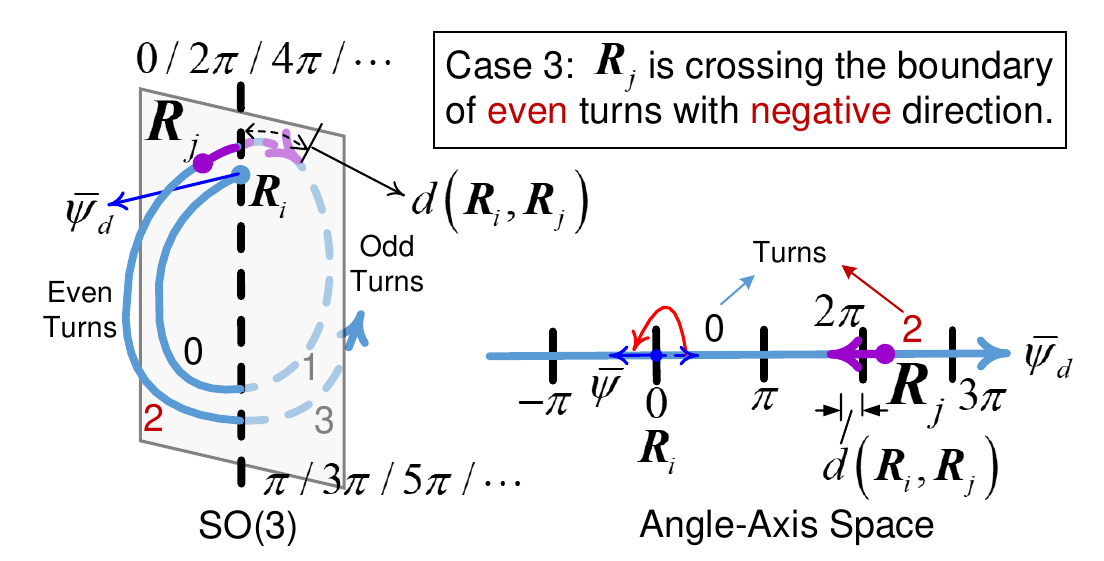}} \hspace{30pt}
    \subfigure[]{\includegraphics[width = 0.45\hsize]{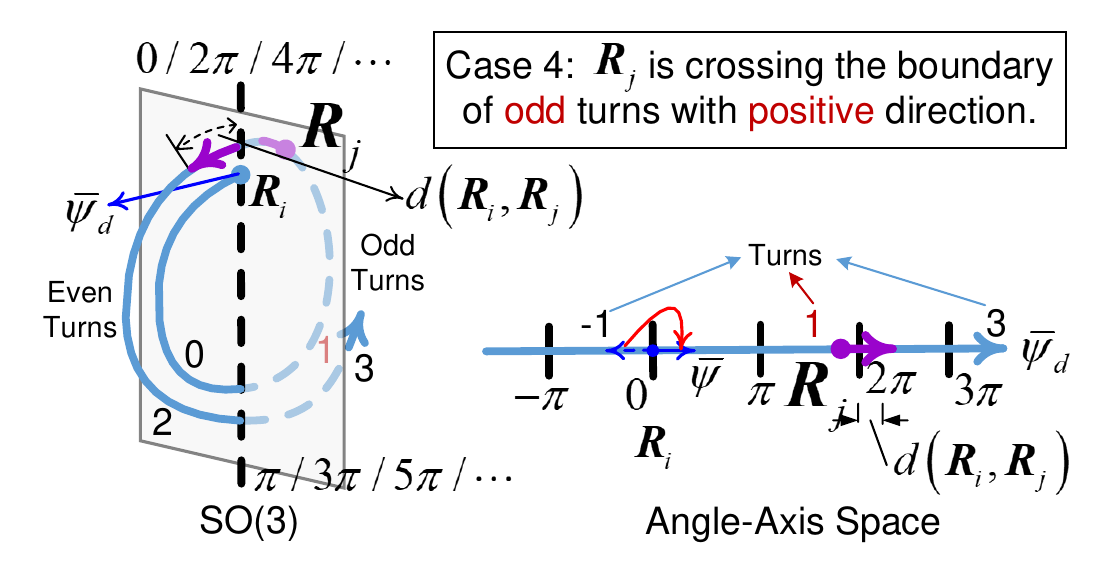}}
    \caption{(a) Case 1: The turn $N_{turns}$ is an even number and $d\left({\vec R}_i, {\vec R}_j\right) > d_{th}$. The default distance and averaging result are obtained according to lines 3-4 in {\bf Algorithm \ref{algorithm1:Update1}}. When the flipping is detected, i.e., the ${\bar {\vec \psi}}\left(t_c\right)$ (marked as ${\bar {\vec \psi}}$ in the figure) is opposite to the historical data ($-{\bar{\vec \psi}}^\top_p {\bar{\vec \psi}}\left(t_c\right) > e_\psi$), we set $N_{turns} = N_{turns}+1$ and the distance would be updated according to line 12 in {\bf Algorithm \ref{algorithm1:Update1}} and the traversal direction is changed to $-{\bar {\vec \psi}}\left(t_c\right)$, as given at line 13 in {\bf Algorithm \ref{algorithm1:Update1}}.   (b) Case 2: The turn $N_{turns}$ is an odd number and $d\left({\vec R}_i, {\vec R}_j\right) > d_{th}$. The default update strategy is given at lines 3-4 in {\bf Algorithm \ref{algorithm2:Update2}}. When the flipping is detected, we set $N_{turns} = N_{turns}-1$ and the update strategy is given at lines 12-13 in {\bf Algorithm \ref{algorithm2:Update2}}. (c) Case 3: The turn $N_{turns}$ is an even number and $d\left({\vec R}_i, {\vec R}_j\right) \leq d_{th}$. The default update strategy is given at lines 3-4 in {\bf Algorithm \ref{algorithm1:Update1}}. When the flipping is detected, we set $N_{turns} = N_{turns}-1$ and the update strategy is given at lines 12-13 in {\bf Algorithm \ref{algorithm1:Update1}}. (d) Case 4: The turn $N_{turns}$ is an odd number and $d\left({\vec R}_i, {\vec R}_j\right) \leq d_{th}$. The default update strategy is given at lines 3-4 in {\bf Algorithm \ref{algorithm2:Update2}}. When the flipping is detected, we set $N_{turns} = N_{turns}+1$ and the update strategy is given at lines 12-13 in {\bf Algorithm \ref{algorithm2:Update2}}.}
    \label{Fig_MemorybasedWeightedAverage}
\end{figure*}

Although the ${\vec R}_k\left(t\right)$ reproduced by KMP and the weighted curve $W_k\left(t\right)$ are both smooth, the weighted rotation average may also encounter the discontinuity issue due to the non-linear structure of a Riemannian manifold. As illustrated in Fig. \ref{Fig_FlippingProblem}, for a Riemannian manifold $\mathcal{M}$, the ``straight line" in this space is curved to a closed curve (the blue line). Without loss of generality, let's consider two points ${\vec R}_i \in \mathcal{M}$ and ${\vec R}_j\left(t\right) \in \mathcal{M}$. When ${\vec R}_j$ crosses over the pole, the shorter arc between ${\vec R}_i$ and ${\vec R}_j$ would be changed from the right hemisphere to the left hemisphere suddenly. Accordingly, to achieve the minimal weighted distance, the averaging result by using the Fr\'{e}chet mean would `jump' from ${\vec R}_{ij,1}$ to ${\vec R}_{ij,2}$, leading to the discontinuity issue. Please note that the above analysis does not involve which representation method is adopted. The discontinuity issue arises from the geometric structure of the Riemannian manifold. Therefore, the memory of past state knowledge is necessary to address the discontinuity issue, no matter what kind of representational method is adopted.
%When ${\vec R}_j$ crosses over the pole, the {\color{blue}traversal direction} ${\bar {\vec \psi}}$ obtained by the logarithm mapping (given in Eq. (\ref{log-map})) would flip to the opposite direction. As a result, the averaging result ${\vec R}_{ij,1}$ would flip to ${\vec R}_{ij,2}$ suddenly, leading to the discontinuity of the reproduced trajectory. More severely, when ${\vec R}_j$ is perturbated around the pole, the averaging result ${\vec R}_{ij}$ may jump to any points on the orange curve. Here the orange curve is the set of orientations that keep distances $d$ to ${\vec R}_i$, as shown in Fig. \ref{Fig_weightedAverage}. This is because the Angle-Axis Space is limited within a ball region ${\vec B}_\pi$ and the distance $d\left({\vec R}_i, {\vec R}_j\right)$ is bounded within $\left[0, \pi\right]$. This discontinuity issue is not specific to the angle-axis representation, but also exists in other representations, as it stems from the Riemannian structure of SO(3).

To address the discontinuity issue, we propose a memory-based algorithm to improve the robustness of our weighted rotation average mechanism. Our proposed algorithm is inspired by \cite{2022ICRA-Jeong-MemoryBasedSO3}. In \cite{2022ICRA-Jeong-MemoryBasedSO3}, an Extended Exponential Coordinates (EEC) by representing SO(3) with continuously many turns was proposed. Based on EEC, a radially unbounded potential function was designed to implement a 6D impedance controller. However, the rotation average is not involved in \cite{2022ICRA-Jeong-MemoryBasedSO3}. By borrowing the idea of EEC, we also extend the geodesic distance $d\left({\vec R}_i, {\vec R}_j\right)$ to many turns to avoid the flipping problem. 

Our basic idea is to update the distance $d$ and the traversal direction ${\bar {\vec \psi}}$, which are used in Eq. (\ref{Eq_weightSumInRiemannian}), according to the past trajectory of exponential coordinates to express multiple turns and keep the axis of rotation. In this way, except for the instantaneous state $\log\left({\vec R}^\top_i {\vec R}_j\right)$, the past state knowledge is also used. This is why we name the proposed algorithm the ``Memory-based Algorithm". 

%%%%%%%%%%%%%%%%%%%%%%%%%%%%%%%%%%%%%%%%%%%%%%%%%%%%%%%%%%%%%%%%%%%%%%%%%%%%%%%%%%%%%%%%%%%%%%%%%%
%%%       Algorithm: Memory-based Weighted Average
%%%%%%%%%%%%%%%%%%%%%%%%%%%%%%%%%%%%%%%%%%%%%%%%%%%%%%%%%%%%%%%%%%%%%%%%%%%%%%%%%%%%%%%%%%%%%%%%%%

\begin{algorithm}[b!]
  \caption{Update Strategy for Case 1 and Case 3} \label{algorithm1:Update1}
  \begin{algorithmic}[1]
        \Function {\tt Update1}{ } %${\vec R}_i$, $W_i$, $W_j$, $N_{turns}$, $\Psi$, $d\left({\vec R}_i, {\vec R}_j\right)$, ${\bar{\vec \psi}}\left(t_c\right)$, ${\bar{\vec \psi}}_p$
            \If {${\bar{\vec \psi}}^\top_p {\bar{\vec \psi}}\left(t_c\right) > e_\psi$} 
                \State $d = \frac{W_j \left(N_{turns}\pi + d\left({\vec R}_i, {\vec R}_j\right)\right)}{W_i + W_j}$;
                \State ${\vec R}_{ij} = {\vec R}_i \exp\left( d {\bar{\vec \psi}}\left(t_c\right) \right)$;
            \Else
                \If {$-{\bar{\vec \psi}}^\top_p {\bar{\vec \psi}}\left(t_c\right) > e_\psi$} \Comment{Flipping happens}
                    \If {$d\left({\vec R}_i, {\vec R}_j\right) > d_{th}$} \Comment{Case 1}
                        \State $N_{turns} =  N_{turns}+1$;
                    \Else \Comment{Case 3}
                        \State $N_{turns} =  N_{turns}-1$; 
                    \EndIf
                    \State $d = \frac{W_j \left( \left(N_{turns}+1\right)\pi - d\left({\vec R}_i, {\vec R}_j\right)\right)}{W_i + W_j} $;
                    \State ${\vec R}_{ij} = {\vec R}_i \exp\left( -d {\bar{\vec \psi}}\left(t_c\right) \right)$;
                    \State $\Psi \leftarrow \emptyset$; \Comment{historical data should be cleaned up after the flipping}
                \Else \Comment{${\bar{\vec \psi}}\left(t_c\right)$ is an outlier}
                    \State $d = \frac{W_j \left(N_{turns}\pi + d\left({\vec R}_i, {\vec R}_j\right)\right)}{W_i + W_j} $;
                    \State ${\vec R}_{ij} = {\vec R}_i \exp\left( d {\bar{\vec \psi}}_p \right)$; \Comment{historical data ${\bar{\vec \psi}}_p$ is used}
                \EndIf
            \EndIf
        \EndFunction
  \end{algorithmic}
\end{algorithm}

\begin{algorithm}[h!]
  \caption{Update Strategy for Case 2 and Case 4} \label{algorithm2:Update2}
  \begin{algorithmic}[1]
        \Function {\tt Update2}{ }
            \If {${\bar{\vec \psi}}^\top_p {\bar{\vec \psi}}\left(t_c\right) > e_\psi$}
                \State $d = \frac{W_j \left( \left(N_{turns}+1\right)\pi - d\left({\vec R}_i, {\vec R}_j\right)\right)}{W_i + W_j} $;
                \State ${\vec R}_{ij} = {\vec R}_i \exp\left( -d {\bar{\vec \psi}}\left(t_c\right) \right)$;
            \Else
                \If {$-{\bar{\vec \psi}}^\top_p {\bar{\vec \psi}}\left(t_c\right) > e_\psi$}
                    \If {$d\left({\vec R}_i, {\vec R}_j\right) > d_{th}$} \Comment{Case 2}
                        \State $N_{turns} =  N_{turns}-1$;
                    \Else \Comment{Case 4}
                        \State $N_{turns} =  N_{turns}+1$; 
                    \EndIf
                    \State $d = \frac{W_j \left(N_{turns}\pi + d\left({\vec R}_i, {\vec R}_j\right)\right)}{W_i + W_j} $;
                    \State ${\vec R}_{ij} = {\vec R}_i \exp\left( d {\bar{\vec \psi}}\left(t_c\right) \right)$;
                    \State $\Psi \leftarrow \emptyset$; 
                \Else \Comment{${\bar{\vec \psi}}\left(t_c\right)$ is an outlier}
                    \State $d = \frac{W_j \left( \left(N_{turns}+1\right)\pi - d\left({\vec R}_i, {\vec R}_j\right)\right)}{W_i + W_j} $;
                    \State ${\vec R}_{ij} = {\vec R}_i \exp\left( -d {\bar{\vec \psi}}_p \right)$; 
                \EndIf
            \EndIf
        \EndFunction
  \end{algorithmic}
\end{algorithm}

\begin{algorithm}[t!]
  \caption{Memory-based Weighted Rotation Average} \label{algorithm1:MemorybasedWeighted Average}
  \begin{algorithmic}[1]
        \State \textbf{Initialization:}
        \State $N_{turns} = 0$;
        \If {$d\left({\vec R}_i\left(t_0\right), {\vec R}_j\left(t_0\right)\right) == 0$}
            \State ${\bar{\vec \psi}}_d = {\vec 0}_{3 \times 1}$;
        \Else
            \State ${\bar{\vec \psi}}_d = \log\left( \left({\vec R}_i\left(t_0\right)\right)^\top {\vec R}_j\left(t_0\right)\right)/d\left({\vec R}_i\left(t_0\right), {\vec R}_j\left(t_0\right)\right)$;
        \EndIf
        \State $\Psi \leftarrow \emptyset$; 
        \State $d_{th} = 0.15$ radians;
        \State $e_\psi = \cos\left(\frac{50\pi}{180}\right)$;
        \State 
        
        \Function {\tt WeightedRotationAverage}{${\vec R}_i$, ${\vec R}_j$, $W_i$, $W_j$, $N_{turns}$, $\Psi$, $d_{th}$}
            \State ${\vec \psi} = \log\left({\vec R}^\top_i {\vec R}_j\right)$;
            \State $d\left({\vec R}_i, {\vec R}_j\right) = \left\|{\vec \psi}\right\|$;
            \If {$d\left({\vec R}_i, {\vec R}_j\right) == 0$}
                \State ${\bar{\vec \psi}}\left(t_c\right) = {\vec 0}_{3 \times 1}$;
            \Else
                \State ${\bar{\vec \psi}}\left(t_c\right) = {\vec \psi}/d\left({\vec R}_i, {\vec R}_j\right)$;
            \EndIf

            \State 
            \If { $size\left(\Psi\right) == 0$ } \Comment{$size\left(\Psi\right)$ gives the length of $\Psi$}
                \State ${\bar{\vec \psi}}_p = {\bar{\vec \psi}}\left(t_c\right)$;                
            \Else
                \State ${\bar{\vec \psi}}_p = {\text {avg}}\left({\bar {\vec \psi}} \left| {\bar {\vec \psi}} \in \Psi\right.\right)$;               
            \EndIf

            \State
            \If {$N_{turns}$ is even} 
                \State Update1(); \Comment{Case 1 and 3}
            \Else \Comment{$N_{turns}$ is odd}
                \State Update2(); \Comment{Case 2 and 4}
            \EndIf

            \State $\Psi \Leftarrow {\bar{\vec \psi}}\left(t_c\right)$; \Comment{Add ${\bar{\vec \psi}}\left(t_c\right)$ to $\Psi$}
        \EndFunction
  \end{algorithmic}
\end{algorithm}

%%%%%%%%%%%%%%%%%%%%%%%%%%%%%%%%%%%%%%%%%%%%%%%%%%%%%%%%%%%%%%%%%%%%%%%%%%%%%%%%%%%%%%%%%%%%%%%%%%
%%%       End of the Algorithm
%%%%%%%%%%%%%%%%%%%%%%%%%%%%%%%%%%%%%%%%%%%%%%%%%%%%%%%%%%%%%%%%%%%%%%%%%%%%%%%%%%%%%%%%%%%%%%%%%%

As shown in Fig. \ref{Fig_MemorybasedWeightedAverage}, an integer variable, denoted by $N_{turns}$, is used to memorize the turn of the historic trajectory of ${\vec R}_j$. $N_{turns}$ is initialized as 0 and would be increased/decreased every time ${\vec R}_j$ crosses the distances of $N_{turns}\pi$. For the sake of convenience, the following notations are defined:
\begin{itemize}
    \item ${\bar {\vec \psi}}\left(t_c\right) := \frac{1}{d\left({\vec R}_i, {\vec R}_j\right)} \log\left({\vec R}^\top_i {\vec R}_j\right)$ is the current traversal direction at time $t_c$. Obviously, ${\bar {\vec \psi}}\left(t_c\right)$ is a unit vector.
    
    \item ${\bar {\vec \psi}}_d := {\bar {\vec \psi}}\left(t_0\right)$ gives the default positive traversal direction determined by the initial ${\bar {\vec \psi}}\left(t\right)$.
    
    \item $\Psi := \left\{{\bar {\vec \psi}}\left(t_{c-5}\right), \cdots, {\bar {\vec \psi}}\left(t_{c-1}\right)\right\}$ records the past traversal directions of ${\bar {\vec \psi}}\left(t\right)$. In this paper, at most 5 past data are used.
    
    \item ${\bar {\vec \psi}}_p := {\text {avg}}\left({\bar {\vec \psi}} \left| \, {\bar {\vec \psi}} \in \Psi \right. \right)$ gives the average direction of the historical data. ${\bar {\vec \psi}}_p$ would be normalized to a unit vector.
    
    \item $d_{th}$ is a threshold to distinguish whether the flipping occurs at $d\left({\vec R}_i, {\vec R}_j\right) == 0$ or $d\left({\vec R}_i, {\vec R}_j\right) == \pi$. Our algorithm is not sensitive to the choice of $d_{th}$. $d_{th}$ can be any positive value between 0 and $\pi$. 
    
    \item $e_\psi$ is a threshold to evaluate whether the current traversal direction is aligned with history by using dot products.
\end{itemize}

According to ${\bar {\vec \psi}}\left(t_c\right)$, $d\left({\vec R}_i, {\vec R}_j\right)$, ${\bar {\vec \psi}}_p$, and $N_{turns}$, 4 cases and 2 update strategies for $d$ and ${\bar {\vec \psi}}$ are designed, as given in Fig. \ref{Fig_MemorybasedWeightedAverage}, {\bf Algorithm \ref{algorithm1:Update1}}, and {\bf Algorithm \ref{algorithm2:Update2}}.

{\bf Update Strategy 1: The turn $N_{turns}$ is an even number}

As shown in Fig. \ref{Fig_MemorybasedWeightedAverage}(a) and (c), when the turn $N_{turns}$ is an even number, the trajectory of ${\vec R}_j$ has traversed a length of $N_{turns}\pi$ with an additional distance $d\left({\vec R}_i, {\vec R}_j\right)$. Therefore, the default distance is given by:
\begin{align}
    \label{Eq_DefaultDistanceInEvenNumber}
    \begin{split}
        d = \frac{W_j \left(N_{turns}\pi + d\left({\vec R}_i, {\vec R}_j\right)\right)}{ W_i + W_j }.
    \end{split}
\end{align}

In this case, the ${\bar {\vec \psi}}\left(t_c\right)$ obtained by the logarithm mapping is aligned with the initial direction ${\bar {\vec \psi}}_d$. Therefore, the averaging result is obtained by:
\begin{align}
    \label{Eq_DefaultRijInEvenNumber}
    \begin{split}
        {\vec R}_{ij} = {\vec R}_i \exp\left( d {\bar {\vec \psi}}\left(t_c\right) \right).
    \end{split}
\end{align}

When $d\left({\vec R}_i, {\vec R}_j\right) > d_{th}$, ${\vec R}_j$ may traverse across the boundary of even turns at the points of $\pi/3\pi/5\pi/\cdots$ with positive direction, leading to the flipping of ${\bar {\vec \psi}}\left(t_c\right)$. Hence, we can monitor the dot product of ${\bar {\vec \psi}}\left(t_c\right)$ and the historical data ${\bar {\vec \psi}}_p$ to detect the flipping event. Considering that ${\vec R}_i$ and ${\vec R}_j$ are both time-varying, we set $50^\circ$ (i.e., $e_\psi = \cos\left(\frac{50\pi}{180}\right)$) as a threshold to evaluate whether the current traversal direction is aligned with history. Therefore, if ${\bar{\vec \psi}}^\top_p {\bar{\vec \psi}}\left(t_c\right) > e_\psi$ still holds, the default distance and averaging result are adopted, as given at lines 3-4 in {\bf Algorithm \ref{algorithm1:Update1}}. However, when $-{\bar{\vec \psi}}^\top_p {\bar{\vec \psi}}\left(t_c\right) > e_\psi$ is detected, it indicates that the flipping has happened. In this case (Case 1), $N_{turns}$ needs to be increased by 1 because ${\vec R}_j$ traverses past the even turns in a positive direction, as illustrated in Fig. \ref{Fig_MemorybasedWeightedAverage}(a) and line 8 in {\bf Algorithm \ref{algorithm1:Update1}}:
\begin{align}
    \label{Eq_TurnsIncreasedBy1}
    \begin{split}
        N_{turns} = N_{turns} + 1.
    \end{split}
\end{align}

Then the distance and the averaging result are given by (as shown at lines 12-13 in {\bf Algorithm \ref{algorithm1:Update1}}):
\begin{align}
    \label{Eq_FlippingUpdateInEvenNumber}
    \begin{split}
        &d = \frac{W_j \left( \left(N_{turns}+1\right)\pi - d\left({\vec R}_i, {\vec R}_j\right)\right)}{W_i + W_j}, \\
        &{\vec R}_{ij} = {\vec R}_i \exp\left( -d {\bar{\vec \psi}}\left(t_c\right) \right).
    \end{split}
\end{align}
Here $-{\bar{\vec \psi}}\left(t_c\right)$ is used as the traversal direction because ${\bar {\vec \psi}}\left(t_c\right)$ is opposite to the default positive direction ${\bar {\vec \psi}}_d$ for odd turns. In this way, the distance and the traversal direction are guaranteed to be smooth to avoid sudden jumps.

In some extreme cases (e.g., ${\bar{\vec \psi}}\left(t_c\right) = {\vec 0}_{3 \times 1}$), the current traversal direction ${\bar {\vec \psi}}\left(t_c\right)$ may be an outlier. In this case, the historical data ${\bar {\vec \psi}}_p$ would be used to obtain the averaging result, as shown at lines 16-17 in {\bf Algorithm \ref{algorithm1:Update1}}.

In contrast, when $d\left({\vec R}_i, {\vec R}_j\right) \leq d_{th}$, ${\vec R}_j$ may traverse across the boundary of even turns at the points of $0/2\pi/4\pi/\cdots$ with negative direction, as illustrated in Fig. \ref{Fig_MemorybasedWeightedAverage}(c). In this case (Case 3), $N_{turns}$ needs to be decreased by 1, as given at line 10 in {\bf Algorithm \ref{algorithm1:Update1}}:
\begin{align}
    \label{Eq_TurnsDecreasedBy1}
    \begin{split}
        N_{turns} = N_{turns} - 1.
    \end{split}
\end{align}
Then the distance and averaging result are obtained according to Eq. (\ref{Eq_FlippingUpdateInEvenNumber}).

This update strategy is summarized in {\bf Algorithm \ref{algorithm1:Update1}}.

{\bf Update Strategy 2: The turn $N_{turns}$ is an odd number}

Similar to Update Strategy 1, when the turn $N_{turns}$ is an odd number, the update strategy is summarized in {\bf Algorithm \ref{algorithm2:Update2}}. The default distance and averaging result are given by Eq. (\ref{Eq_FlippingUpdateInEvenNumber}), as given at lines 3-4 in {\bf Algorithm \ref{algorithm2:Update2}}. When the flipping is detected, $N_{turns}$ needs to be decreased/increased by 1, according to $d\left({\vec R}_i, {\vec R}_j\right) > d_{th}$ (Case 2) or $d\left({\vec R}_i, {\vec R}_j\right) \leq d_{th}$ (Case 4). Then the distance and averaging result are given by Eq. (\ref{Eq_DefaultDistanceInEvenNumber}) and (\ref{Eq_DefaultRijInEvenNumber}), as illustrated at lines 7-14 in {\bf Algorithm \ref{algorithm2:Update2}}. The situation that ${\bar {\vec \psi}}\left(t_c\right)$ is an outlier is handled at lines 16-17 in {\bf Algorithm \ref{algorithm2:Update2}}.

\vspace{10pt}
Finally, our proposed memory-based weighted rotation average algorithm is summarized in {\bf Algorithm \ref{algorithm1:MemorybasedWeighted Average}} and Fig. \ref{Fig_AlgorithmFlowchar}. Since our solution gives the closed-form expression, only 1 logarithmic operation and 1 exponential operation are needed to fuse any two points. To fuse the $K$ IOVPs, $2K$ logarithmic/exponential operations are needed. Therefore, our solution only needs $2mK$ logarithmic/exponential operations for a trajectory with $m$ points. Compared with the iterative Karcher flow, it is more efficient and the time consumption is deterministic.

\begin{figure}[t]
    \centering
    \includegraphics[width = \hsize]{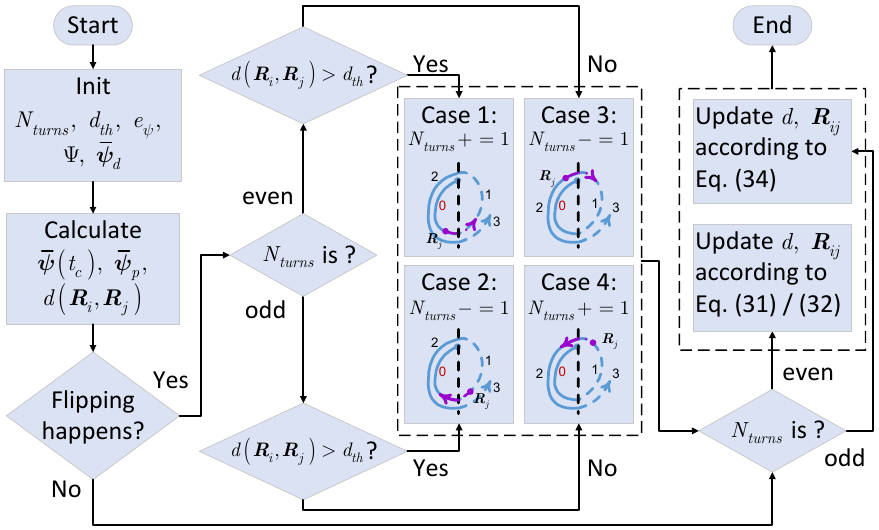}
    \caption{ The flowchart of the memory-based algorithm }
    \label{Fig_AlgorithmFlowchar}
\end{figure}

\begin{assumption} \label{Assumption-CorrectFlippingDetection}
The sampling rate of the proposed algorithm is fast enough, such that all the flipping events can be captured by the memory-based algorithm. 
\end{assumption}

We propose this assumption because there is a possibility that the memory-based algorithm fails to capture the flipping event when the rotation between two time steps exceeds $\pi$. However, in this case, the angular speed is larger than $\frac{\pi}{dt}$\,rad/s, where $dt$ is the time duration for the two time steps. Generally, $dt$ is very small to guarantee the smoothness of the trajectory. For example, in our simulations and experiments, $dt$ = 0.01s is chosen. This means that the angular velocity would surpass 100$\pi$ rad/s, which greatly exceeds the maximum angular velocity limits for almost all manipulators (For example, the maximum angular velocity for KUKA iiwa is $\pi$ rad/s). Therefore, although it is possible, this scenario is highly unlikely to occur. To make it more rigorous, we add this assumption in the manuscript.

\begin{lemma} \label{Theorem-continuous-d-psi}
Based on {\bf Assumption \ref{Assumption-CorrectFlippingDetection}}, and by memorizing the turns $N_{turns}$ and past data, the distance curve $d\left(t\right)$ and the traversal direction ${\bar {\vec \psi}}\left(t\right)$ vary continuously even when the flipping event happens, if ${\vec R}_i\left(t\right)$, ${\vec R}_j\left(t\right)$ and the weighted curves $W_i\left(t\right)$, $W_j\left(t\right)$ are all Lipschitz continuous.
\end{lemma}

\begin{proof}
Without loss of generality, we consider the case (Case 1 in Fig. 11(a)) that $N_{turns}$ is an even integer and $d\left({\vec R}_i, {\vec R}_j\right) > d_{th}$. Suppose ${\vec R}_i$, ${\vec R}_j$ move smoothly from $t^{-}$ to $t^{+}$, where $t$ is exactly the moment when the flipping event happens and $t^{-}$/$t^{+}$ denote the moment before/after $t$. 

Then according to our memory-based algorithm, the distance curves $d\left(t^{-}\right)$ and $d\left(t^{+}\right)$ at time $t^{-}$/$t^{+}$ are given as:
\begin{align}
    \label{Eq_proof_d_continuous}
    \begin{split}
        d_{t^{-}} &= \frac{W_j\left(t^{-}\right) \left( N^{-}_{turns}\pi + d\left({\vec R}_i\left(t^{-}\right), {\vec R}_j\left(t^{-}\right) \right) \right)} { W_j\left(t^{-}\right) + W_i\left(t^{-}\right) },\\
        d_{t^{+}} &= \frac{W_j\left(t^{+}\right) \left( \left(N^{+}_{turns}+1\right)\pi - d\left({\vec R}_i\left(t^{+}\right), {\vec R}_j\left(t^{+}\right)\right) \right)} { W_j\left(t^{+}\right) + W_i\left(t^{+}\right) }.\\
    \end{split}
\end{align}

Based on {\bf Assumption \ref{Assumption-CorrectFlippingDetection}}, we can have the relation:
\begin{align}
    \label{Eq_Nturns_relation}
    \begin{split}
        N^{+}_{turns} = N^{-}_{turns} + 1.
    \end{split}
\end{align}

Since ${\vec R}_i\left(t\right)$, ${\vec R}_j\left(t\right)$ and $W_i\left(t\right)$, $W_j\left(t\right)$ are all continuous, it is easy to verify that:
\begin{align}
    \label{Eq_Wj_Wi}
    \begin{split}
        \lim_{t^{-} \rightarrow t} {W_i\left(t^{-}\right)} = \lim_{t^{+} \rightarrow t} {W_i\left(t^{+}\right)} = W_i\left(t\right),\\
        \lim_{t^{-} \rightarrow t} {W_j\left(t^{-}\right)} = \lim_{t^{+} \rightarrow t} {W_j\left(t^{+}\right)} = W_j\left(t\right),\\
        \lim_{t^{-} \rightarrow t} {{\vec R}_j\left(t^{-}\right)} = \lim_{t^{+} \rightarrow t} {{\vec R}_j\left(t^{+}\right)} = {\vec R}_j\left(t\right),\\
        \lim_{t^{-} \rightarrow t} {{\vec R}_i\left(t^{-}\right)} = \lim_{t^{+} \rightarrow t} {{\vec R}_i\left(t^{+}\right)} = {\vec R}_i\left(t\right).\\
    \end{split}
\end{align}

Since $t$ is exactly the moment that the flipping event happens, the distance between ${\vec R}_i\left(t\right)$ and ${\vec R}_j\left(t\right)$ equals $\pi$, i.e., 
\begin{align}
    \label{Eq_dRiRj_pi}
    \begin{split}
        d\left({\vec R}_i\left(t\right), {\vec R}_j\left(t\right)\right) = \pi.
    \end{split}
\end{align}

Therefore, we have:
\begin{align}
    \label{Eq_dRiRj_limits}
    \begin{split}
        \lim_{t^{-} \rightarrow t} {d\left({\vec R}_i\left(t^{-}\right), {\vec R}_j\left(t^{-}\right)\right)} &= \lim_{t^{+} \rightarrow t} {d\left({\vec R}_i\left(t^{+}\right), {\vec R}_j\left(t^{+}\right)\right)} \\
        &= \pi.
    \end{split}
\end{align}

According to Eq. (\ref{Eq_Nturns_relation}), (\ref{Eq_Wj_Wi}), (\ref{Eq_dRiRj_pi}), and (\ref{Eq_dRiRj_limits}), we can verify that:
\begin{align}
    \label{Eq_distance_limits}
    \begin{split}
        \lim_{t^{-} \rightarrow t} {d_{t^-}} &= \frac{W_j\left(t\right) \left( N^{-}_{turns}\pi + \pi \right)} { W_j\left(t\right) + W_i\left(t\right) },\\
        \lim_{t^{+} \rightarrow t} {d_{t^+}} &= \frac{W_j\left(t\right) \left( \left(N^{-}_{turns} + 1 + 1\right)\pi - \pi \right)} { W_j\left(t\right) + W_i\left(t\right) } \\
        &= \frac{W_j\left(t\right) \left( N^{-}_{turns}\pi + \pi \right)} { W_j\left(t\right) + W_i\left(t\right) }.\\
    \end{split}
\end{align}

Therefore, we verify that $\lim_{t^{-} \rightarrow t} {d_{t^-}} = \lim_{t^{+} \rightarrow t} {d_{t^+}}$. Hence the curve $d\left(t\right)$ is continuous.

To prove the continuity of the traversal direction ${\bar {\vec \psi}}\left(t\right)$, we have to add a stricter condition that the ${\vec R}_i\left(t\right)$ and ${\vec R}_j\left(t\right)$ are Lipschitz continuous, i.e., the rate of change between any two points can be globally bounded.

%To better illustrate the proof, we denote ${\vec \psi}\left(t\right)$ as the error vector between ${\vec R}_i\left(t\right)$ and ${\vec R}_j\left(t\right)$, and ${\hat {\vec \psi}}\left(t\right)$ as the unit vector of ${\vec \psi}\left(t\right)$. 
Then the traversal direction ${\bar {\vec \psi}}$ can be given as:
\begin{align}
    \label{Eq_psi_definition}
    \begin{split}
        {\bar {\vec \psi}}_{t^{-}} &= f\left( \log\left( {\vec R}^\top_i\left(t^{-}\right) {\vec R}_j\left(t^{-}\right) \right) \right),\\
        {\bar {\vec \psi}}_{t^{+}} &= f\left( \log\left( {\vec R}^\top_i\left(t^{+}\right) {\vec R}_j\left(t^{+}\right) \right) \right),\\
    \end{split}
\end{align}
here $f\left(\cdot\right)$ represents the operations including normalization and direction flips (if $N_{turns}$ is an odd number).

As illustrated in Fig. \ref{Fig_psi-continuous}(a), as $t^{-}$ and $t^{+}$ both converge to $t$, ${\bar {\vec \psi}}_{t^{-}}$ and ${\bar {\vec \psi}}_{t^{+}}$ converge to the same point, if ${\vec R}_i\left(t\right)$ and ${\vec R}_j\left(t\right)$ are Lipschitz continuous. However, as illustrated in Fig. \ref{Fig_psi-continuous}(b), if ${\vec R}_i\left(t\right)$ and ${\vec R}_j\left(t\right)$ are only continuous, the traversal direction ${\bar {\vec \psi}}\left(t\right)$ may jump and break the continuity. This is because the single point ${\vec R}^\top_i {\vec R}_j$ on SO(3) is mapped to a `circle curve/spherical surface' on $\mathbb{R}^3$. Without the Lipschitz assumption, the traversal direction ${\bar {\vec \psi}}\left(t\right)$ is able to jump to a different direction via instantaneous teleportation along the `circle curve/spherical surface'.

The continuity for the other cases can also be validated by following a similar procedure.

\begin{figure}[t]
    \centering
    \subfigure[]{\includegraphics[width = \hsize]{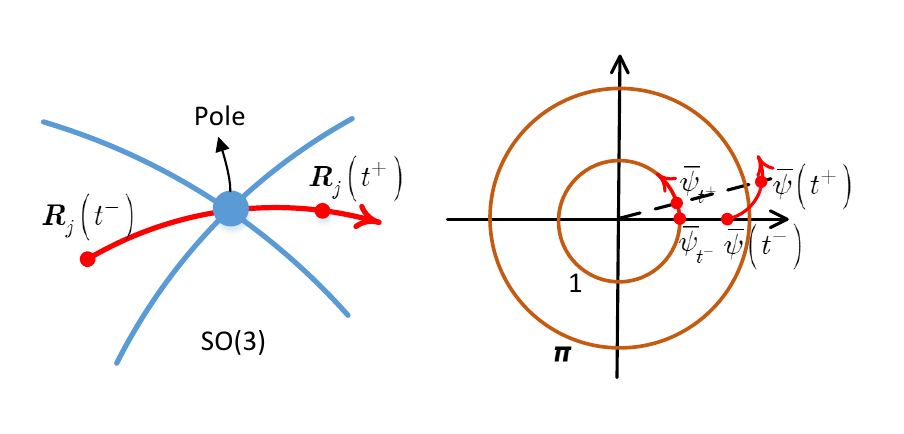}}\\
    \subfigure[]{\includegraphics[width = \hsize]{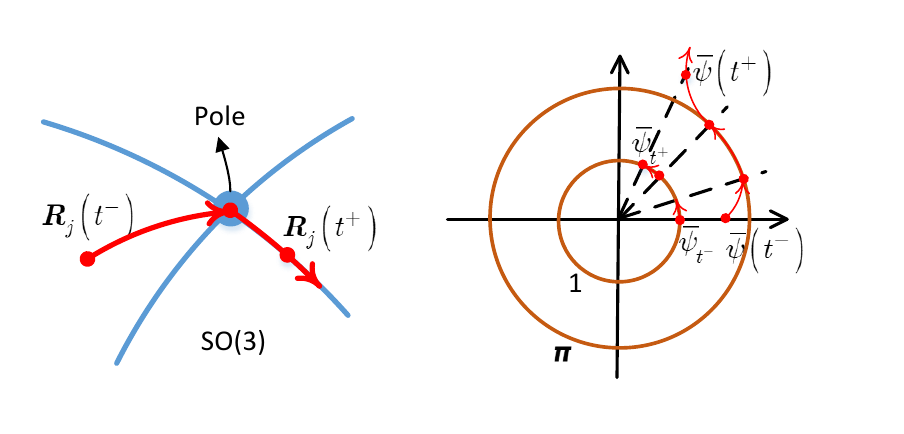}}
    \caption{Illustration of the continuity of ${\bar {\vec \psi}}\left(t\right)$ at the transition state. (a) The Lipschitz condition can ensure the continuity of ${\bar {\vec \psi}}\left(t\right)$. (b) Without Lipschitz condition, the traversal direction ${\bar {\vec \psi}}\left(t\right)$ jumps to a different direction via instantaneous teleportation along the `circle curve/spherical surface'.}
    \label{Fig_psi-continuous}
\end{figure}

 \hfill $\square$
\end{proof}

\begin{remark} \label{Remark}
Please note that the stricter Lipschitz continuity for ${\vec R}_i\left(t\right)$ and ${\vec R}_j\left(t\right)$ is proposed to ensure mathematical rigor. Actually the continuity of ${\bar {\vec \psi}}\left(t\right)$ would be broken only when the trajectory of ${\vec R}^\top_i {\vec R}_j$ breaks the Lipschitz continuity exactly at the pole point. However, this event has little chance of being encountered in practical engineering. Therefore, the Lipschitz continuity is not necessary in real applications. But the Lipschitz continuity is also easy to guarantee, since the orientation trajectories are typically required to be smooth in real applications.
\end{remark}

\begin{lemma} \label{Theorem-continuity}
Our proposed memory-based algorithm can guarantee the continuity of the weighted averaging results, if the trajectories ${\vec R}_i\left(t\right)$, ${\vec R}_j\left(t\right)$ are Lipschitz continuous and the weighted curves $W_i\left(t\right)$, $W_j\left(t\right)$ are continuous.
\end{lemma}

\begin{proof}
Denote $t_n$ and $t_{n+1}$ as two time moments, then the weighted averaging results ${\vec R}_{ij}\left(t_n\right)$ and ${\vec R}_{ij}\left(t_{n+1}\right)$ are given as:
\begin{align}
    \label{Eq_proof_Rij}
    \begin{split}
        {\vec R}_{ij}\left(t_n\right) &= {\vec R}_i \left(t_n\right) \exp\left(d_{t_n} {\bar {\vec \psi}}_{t_n} \right),\\
        {\vec R}_{ij}\left(t_{n+1}\right) &= {\vec R}_i\left(t_{n+1}\right) \exp\left(d_{t_{n+1}} {\bar {\vec \psi}}_{t_{n+1}} \right),
    \end{split}
\end{align}
where ${\vec R}_i \left(t_n\right)$ and ${\vec R}_{i}\left(t_{n+1}\right)$ are the points of the trajectory ${\vec R}_i\left(t\right)$ at time $t_n$ and $t_{n+1}$, respectively. $d_{t_n}$, ${\bar {\vec \psi}}_{t_n}$ and $d_{t_{n+1}}$, ${\bar {\vec \psi}}_{t_{n+1}}$ are the distances and traversal directions generated by our proposed memory-based algorithm according to ${\vec R}_i\left(t_n\right), {\vec R}_j\left(t_n\right)$ and ${\vec R}_i\left(t_{n+1}\right), {\vec R}_j\left(t_{n+1}\right)$, respectively.

Then the distance between ${\vec R}_{ij}\left(t_n\right)$ and ${\vec R}_{ij}\left(t_{n+1}\right)$ is given by:
\begin{align}
    \label{Eq_proof_distance}
    \begin{split}
        &d\left({\vec R}_{ij}\left(t_n\right), {\vec R}_{ij}\left(t_{n+1}\right)\right) = \left\| \log\left( {\vec R}^\top_{ij}\left(t_n\right) {\vec R}_{ij}\left(t_{n+1}\right)  \right)\right\| \\
        &= \left\| \log\left( \exp^\top\left(d_{t_n} {\bar {\vec \psi}}_{t_n} \right) {\vec R}^\top_i\left(t_n\right) {\vec R}_{ij}\left(t_{n+1}\right) \right)\right\| \\
        &\leq \left\| -d_{t_n} {\bar {\vec \psi}}_{t_n} + \log\left( {\vec R}^\top_i\left(t_n\right) {\vec R}_i\left(t_{n+1}\right)\right) + d_{t_{n+1}} {\bar {\vec \psi}}_{t_{n+1}} \right\| \\
        &\leq \left\| d_{t_{n+1}} {\bar {\vec \psi}}_{t_{n+1}}-d_{t_n} {\bar {\vec \psi}}_{t_n} + \log\left( {\vec R}^\top_i\left(t_n\right) {\vec R}_i\left(t_{n+1}\right)\right) \right\| \\
    \end{split}
\end{align}

Since ${\vec R}_i\left(t\right)$ is continuous, for any given positive real number $\varepsilon_{\vec R} > 0$, however small, we can find a positive real number $\delta_{\vec R} > 0$ such that for $\left|t_{n+1} - t_n\right| < \delta_{\vec R}$, we have $d\left( {\vec R}_i\left(t_n\right), {\vec R}_i\left(t_{n+1}\right) \right) < \varepsilon_{\vec R}$. 

%Since the trajectories ${\vec R}_i\left(t\right)$, ${\vec R}_j\left(t\right)$ and the weighted curves $W_i\left(t\right)$, $W_j\left(t\right)$ are all continuous, our weighted average algorithm by memorizing the turn $N_{turns}$, the distance $d\left(t\right)$ and the {\color{blue}traversal direction} ${\bar {\vec \psi}}\left(t\right)$ are also continuous. 
According to {\bf Lemma \ref{Theorem-continuous-d-psi}}, the distance $d\left(t\right)$ and the traversal direction ${\bar {\vec \psi}}\left(t\right)$ are also continuous.
Therefore, for any given positive real number $\varepsilon - \varepsilon_{\vec R} > 0$, however small, we can find a positive real number $\delta_d > 0$ such that for $\left|t_{n+1} - t_n\right| < \delta_d$, we have $\left\| d_{t_{n+1}} {\bar {\vec \psi}}_{t_{n+1}}-d_{t_n} {\bar {\vec \psi}}_{t_n}\right\| < \varepsilon - \varepsilon_{\vec R}$.

Hence for any given positive number $\varepsilon > 0$, there exists $\delta = \min\left(\delta_{\vec R}, \delta_d\right)$ such that for $\left|t_{n+1} - t_n\right| < \delta$, we have
\begin{align}
    \label{Eq_proof_distance}
    \begin{split}
        &d\left({\vec R}_{ij}\left(t_n\right), {\vec R}_{ij}\left(t_{n+1}\right)\right) \\
        &\leq \left\| d_{t_{n+1}} {\bar {\vec \psi}}_{t_{n+1}}-d_{t_n} {\bar {\vec \psi}}_{t_n} \right\| + d\left( {\vec R}_i\left(t_n\right), {\vec R}_i\left(t_{n+1}\right)\right) \\
        &\leq \varepsilon - \varepsilon_{\vec R} + \varepsilon_{\vec R} = \varepsilon.
    \end{split}
\end{align}

Therefore, the continuity of the weighted averaging results is proved.

 \hfill $\square$
\end{proof}

\section{Simulation Evaluations}
\label{Section_Simulations}

% In this section, we report several evaluations to illustrate the
We test the performance of our approach via the following evaluations: 1) Orientation adaptation towards arbitrary desired points in terms of orientations and angular velocities (Section \ref{subsection_evaluation_a}); 2) Orientation adaptations with angular acceleration constraints (Section \ref{subsection_evaluation_b}); 3) Orientation adaptation with a single IOVP (Section \ref{subsection_evaluation_c}); and 4) Orientation adaptation with multiple IOVPs (Section \ref{subsection_evaluation_D}). The first and second evaluations are used to demonstrate that our proposed method inherits the KMP's capability of adapting to any via-points and angular acceleration constraints, whose performance is equivalent to the previous work shown in \cite{2020TRO-KMP-Orientation}. The third and fourth evaluations demonstrate the extra benefit of our method that we can incorporate multiple IOVPs easily. All demonstrations are time-driven trajectories.
% i.e., the input $s$ is time. But please
% Note that our method is not limited to the time-driven trajectories.
The Gaussian kernel $k\left(t_i, t_j\right) = \exp\left(-l \left(t_i - t_j\right)^2\right)$ with $l = 0.01$ and the regularization factor $\lambda = 1$ are used.

\subsection{Evaluations on Adaptation with Via-Points}
\label{subsection_evaluation_a}

In this subsection, we consider the orientation adaptation towards multiple desired via-points. We collected five simulated orientation trajectories, all of which have a duration of 10 seconds, as depicted in Fig. \ref{Fig_s1_Demonstrations}. For better comparison, the simulated demonstrations are the same as the ones adopted in \cite{2020TRO-KMP-Orientation}.

To alleviate possible distortion, the auxiliary frame $\vec{R}_a$ is set as the initial value of the simulated trajectories. Then all demonstrations are transformed into this local frame and projected into the Angle-Axis Space. The results of the probabilistic modeling are depicted in Fig. \ref{Fig_s2_GMM+GMR}. GMM and GMR methods (\cite{2016IntellSerRob-Tutorial-TPGMM, 1996JAIR-GMR}) are used to extract the reference trajectory $D_r$.  Here 5 Gaussian components are used to estimate the joint probability distribution of the demonstrations. As shown in Fig. \ref{Fig_s2_GMM+GMR}, the extracted reference trajectory is able to encapsulate the distribution of the demonstrations.

% \begin{table}[b]
%   \centering
%   \caption{The Two Groups of Via-Points}
%   \label{Table-ViaPoints}
%   \setlength{\tabcolsep}{1.4mm}{
%   \begin{tabular}{c|c|c|c}
%     \hline
%     \hline
%                                      & $\tilde{t}_h$  & $\tilde{\vec \psi}_h$ & $\tilde{\vec \omega}_h$  \\ \hline
%        \multirow{3}{*}{Adaption 1} & 0.0            & $[1.2614, 1.0512, 1.5767]^\top$            & $[0.0, 0.0, 0.0]^\top$    \\ \cline{2-4}
%                                      & 4.0            & $[1.5456, 1.0304, 2.0608]^\top$            & $[0.1, 0.1, 0.0]^\top$    \\ \cline{2-4} %$[1.6924, 0.8462, 1.9745]^\top$
%                                      & 10.0           & $[0.9137, 1.3705, 0.9137]^\top$            & $[0.0, 0.3, 0.3]^\top$    \\ \hline
%        \multirow{3}{*}{Adaption 2} & 0.0            & $[1.2614, 1.0512, 1.5767]^\top$            & $[0, 0, 0]^\top$    \\ \cline{2-4}
%                                      & 6.0            & $[0.7028, 1.1713, 0.4685]^\top$            & $[0.1, 0.2, 0.0]^\top$    \\ \cline{2-4}
%                                      & 10.0           & $[0.9137, 1.3705, 0.9137]^\top$            & $[0.0, 0.3, 0.3]^\top$    \\ \hline
%     \hline
%   \end{tabular}
%   }
% \end{table}

%\begin{figure}[t]
%    \centering
%    \includegraphics[width=\hsize]{figs/s3_ViaPoint1-new-quaternion.pdf}
%    \caption{The first adaptation results shown in quaternion representation.}
%    \label{Fig_s3_ViaPoint1-Quaternion}%, viewport={0 0 750 300}, clip=true
%\end{figure}

In order to validate the capability of adapting orientation trajectories towards new targets with desired angular velocities, two adaptation tasks are carried out. In each adaptation task, three desired via-points are given as:
\begin{align}
    \label{Eq_via-point-NormalViaPoints}
    \begin{split}
        \tilde{t}_1 = 0.0, \tilde{\vec \psi}_1 &= [1.2614, 1.0512, 1.5767]^\top, \\ \tilde{\vec \omega}_1 &= [0.0, 0.0, 0.0]^\top;\\
        \tilde{t}_2 = 4.0, \tilde{\vec \psi}_2 &= [1.5456, 1.0304, 2.0608]^\top, \\ \tilde{\vec \omega}_2 &= [0.1, 0.1, 0.0]^\top; \\
        \tilde{t}_3 = 10.0, \tilde{\vec \psi}_3 &= [0.9137, 1.3705, 0.9137]^\top,\\ \tilde{\vec \omega}_3 &= [0.0, 0.3, 0.3]^\top,
    \end{split}
\end{align}
and
\begin{align}
    \label{Eq_via-point-NormalViaPoints-2}
    \begin{split}
        \tilde{t}_1 = 0.0, \tilde{\vec \psi}_1 &= [1.2614, 1.0512, 1.5767]^\top, \\ \tilde{\vec \omega}_1 &= [0.0, 0.0, 0.0]^\top;\\
        \tilde{t}_2 = 6.0, \tilde{\vec \psi}_2 &= [0.7028, 1.1713, 0.4685]^\top, \\ \tilde{\vec \omega}_2 &= [0.1, 0.2, 0.0]^\top; \\
        \tilde{t}_3 = 10.0, \tilde{\vec \psi}_3 &= [0.9137, 1.3705, 0.9137]^\top,\\ \tilde{\vec \omega}_3 &= [0.0, 0.3, 0.3]^\top,
    \end{split}
\end{align}
respectively.

\begin{figure}[t]
    \centering
    \includegraphics[width=0.95\hsize]{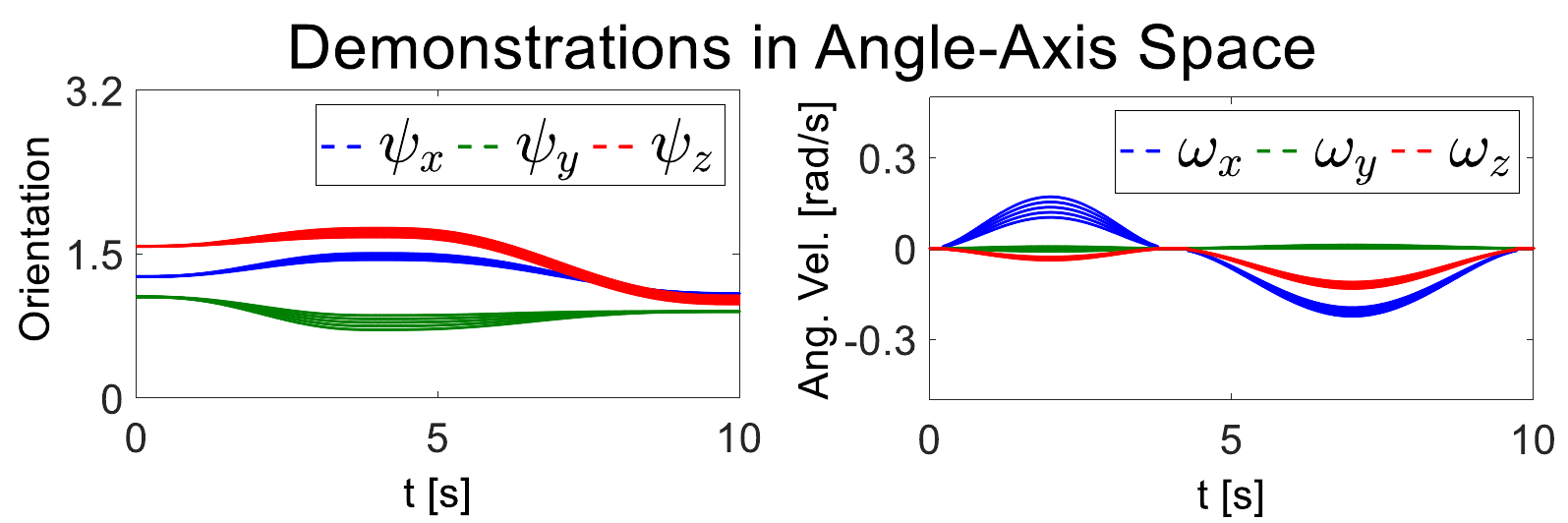}
    \caption{The simulated orientation trajectories in Angle-Axis Space and their corresponding angular velocities, expressed in the world frame.}
    \label{Fig_s1_Demonstrations}
\end{figure}

\begin{figure}[t]
    \centering
    \includegraphics[width = \hsize]{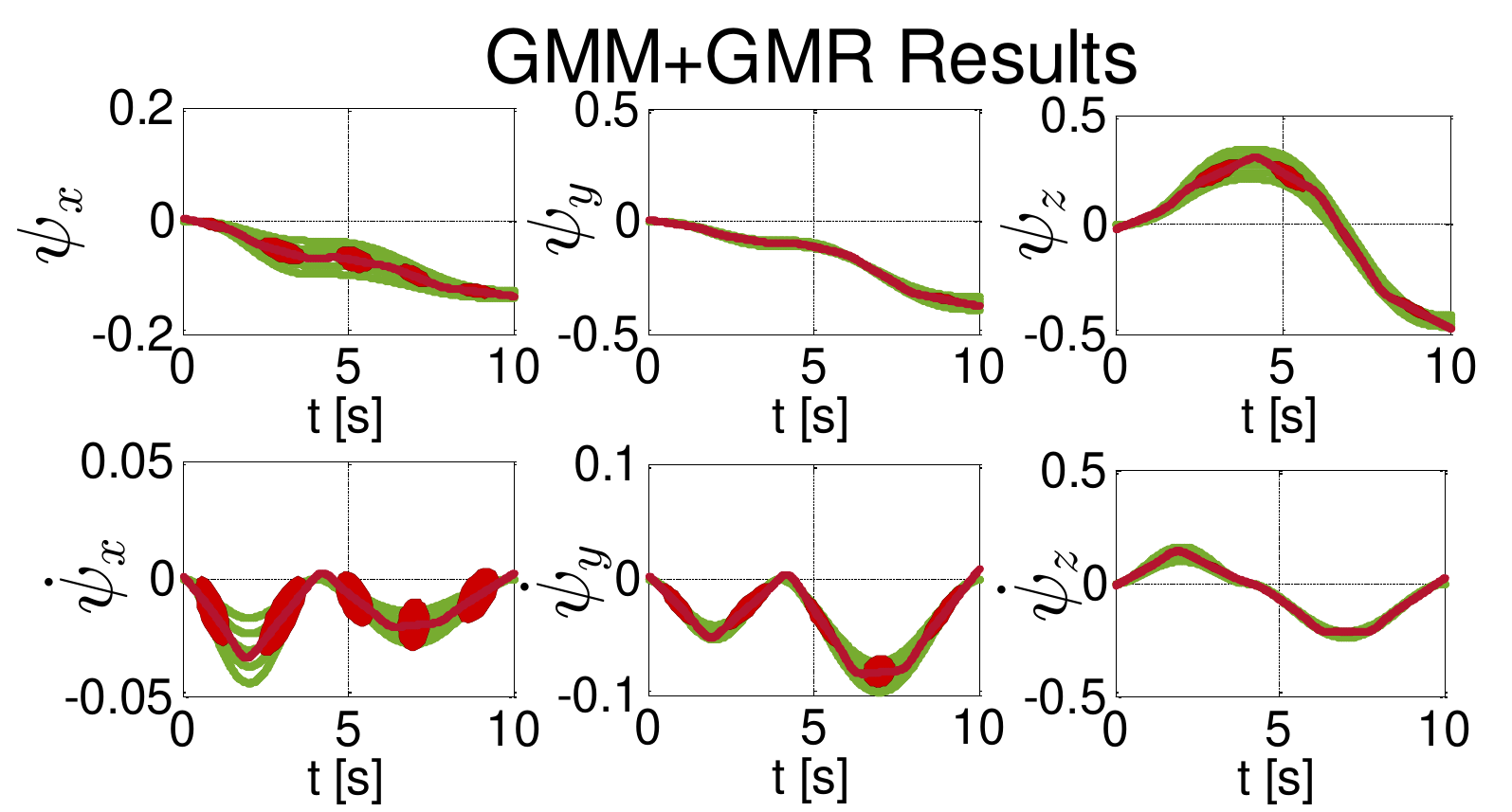}
    \caption{The probabilistic modeling of the demonstrations. Green curves represent the simulated orientation trajectories and their derivatives (not the angular velocities). Red ellipses are the 5 Gaussian components used to encapsulate the probabilistic distribution of demonstrations by using GMM. Red curves are the means ($\hat{\vec \mu}_n$) of the extracted reference trajectory $D_r$ (the extracted covariances $\hat{\vec \Sigma}_n$ are ignored for better readability). Please note that all the orientation trajectories and their derivatives have been transformed into the auxiliary frame and projected into the Angle-Axis Space. Hence the shapes of the demonstrations are different from the ones shown in Fig. \ref{Fig_s1_Demonstrations}. But the final reproduced trajectories can be transferred into the world frame by Eq. (\ref{Eq_OriRecover}).}
    \label{Fig_s2_GMM+GMR}
\end{figure}

\begin{figure}[!t]
    \centering
    \subfigure[]{\includegraphics[width = \hsize]{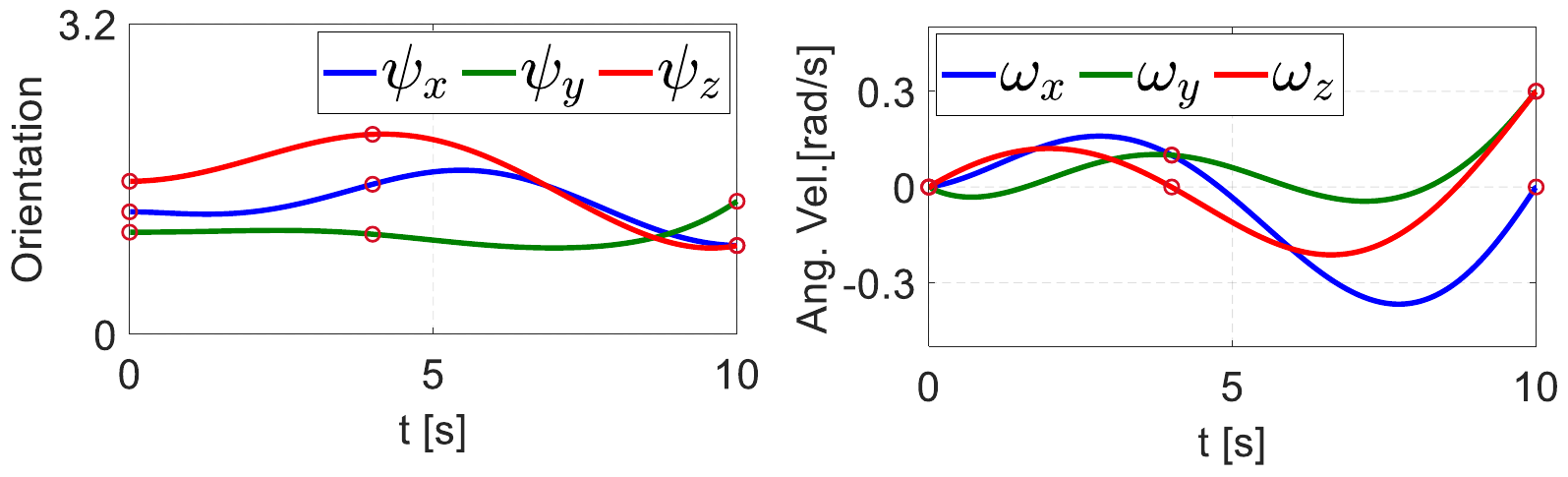}}
    \subfigure[]{\includegraphics[width = \hsize]{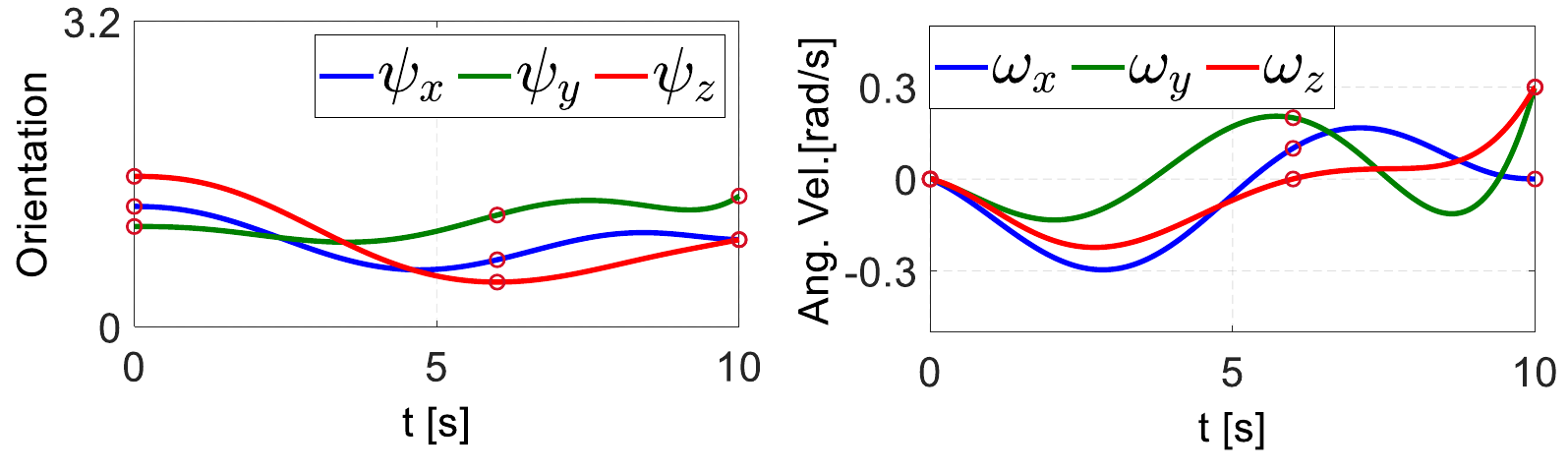}}
    \caption{Adaptations of orientations and angular-velocity profiles with various constraints of desired points (depicted by circles). (a) the first group of via-points. (b) the second group of via-points.}
    \label{Fig_s3_ViaPoints}
\end{figure}

The covariance $\tilde{\vec \Sigma}_h$ for each via-point is designed as: $\tilde{\vec \Sigma}_h = \text{blockdiag}\left(\tilde{\vec \Sigma}_{rh}, \tilde{\vec \Sigma}_{vh}\right)$, where $h = 1, 2, 3$, $\tilde{\vec \Sigma}_{rh}$ and $\tilde{\vec \Sigma}_{vh}$ represent the covariances for orientation and angular velocity, respectively. In this subsection, $\tilde{\vec \Sigma}_{rh}$ and $\tilde{\vec \Sigma}_{vh}$ are both set as $\text{diag}\left(10^{-10}, 10^{-10}, 10^{-10}\right)$, i.e., $10^{-10}*I_3$.

The adaptation results are given in Fig. \ref{Fig_s3_ViaPoints}. We can see that our approach modulates orientations and angular velocities simultaneously to go through all desired points. Moreover, the shapes of demonstrations including orientations (shown in Angle-Axis Space) and angular velocities are kept in the adaptation tasks. %For better comparison, the first adaptation task are the same as the adaptation task shown in Fig. 2(a) of \cite{2020TRO-KMP-Orientation}. The adaptation results are projected into the quaternion representation, as provided in Fig. \ref{Fig_s3_ViaPoint1-Quaternion}. As we can tell, it is the same as the Fig. 2(a2) of \cite{2020TRO-KMP-Orientation}, which is a solid proof for the effectiveness of our method.

\subsection{Evaluation on Angular Acceleration Constraints}
\label{subsection_evaluation_b}

In this subsection, we aim to adapt the orientation profile towards desired points while considering the angular acceleration constraints. We use the same demonstrations as shown in Fig. \ref{Fig_s1_Demonstrations} and the same auxiliary frame $\vec{R}_a$. The desired points are randomly selected as follows: 
\begin{align}
    \label{Eq_via-point-AccCost}
    \begin{split}
        \tilde{t}_1 = 0.0, \tilde{\vec \psi}_1 &= [1.2614, 1.0512, 1.5767]^\top, \\ \tilde{\vec \omega}_1 &= [0.0, 0.0, 0.0]^\top;\\
        \tilde{t}_2 = 5.0, \tilde{\vec \psi}_2 &= [1.7639, 0.7560, 2.0159]^\top, \\ \tilde{\vec \omega}_2 &= [0.1, 0.0, 0.0]^\top; \\
        \tilde{t}_3 = 10.0, \tilde{\vec \psi}_3 &= [0.7935, 1.3224, 0.0000]^\top,\\ \tilde{\vec \omega}_3 &= [-0.1, 0.0, 0.0]^\top.
    \end{split}
\end{align}

The covariance for each via-point is designed as: $\tilde{\vec \Sigma}_h = \text{blockdiag}\left(\tilde{\vec \Sigma}_{rh}, \tilde{\vec \Sigma}_{vh}, \tilde{\vec \Sigma}_{ah}\right)$, where the block matrix $\tilde{\vec \Sigma}_{rh} = \text{diag}\left(10^{-10}, 10^{-10}, 10^{-10}\right)$ represents the selected covariance for orientation, $\tilde{\vec \Sigma}_{vh} = \text{diag}\left(10^{3}, 10^{3}, 10^{3}\right)$ represents the selected covariance for angular velocity, and $\tilde{\vec \Sigma}_{ah} = \frac{1}{\lambda_a}I_3$ is a user-specified weight matrix to consider the angular acceleration minimization. The precisions on angular velocities are relaxed because they would be regulated to consider the angular acceleration cost. To better illustrate the impact of considering the angular acceleration constraints, 5 different $\lambda_a$ ($10, 10^2, 10^3, 10^4, 10^5$, respectively) are chosen to compare the smoothness of the reproduced trajectories.

% In order to quantitatively show the performance of our approach, the following angular acceleration cost \cite{2020TRO-KMP-Orientation} are used:
% \begin{align}
%     \label{Eq_AngCost}
%     \begin{split}
%         c_{wd} = \frac{1}{N}\sum^{N}_{n = 1} {\left\|\dot{\vec \omega} \left(t_n\right) \right\|^2},
%     \end{split}
% \end{align}
% and meanwhile a group of penalty parameters $\lambda_a$ are used.

\begin{figure}[!b]
    \centering
    \includegraphics[width = \hsize]{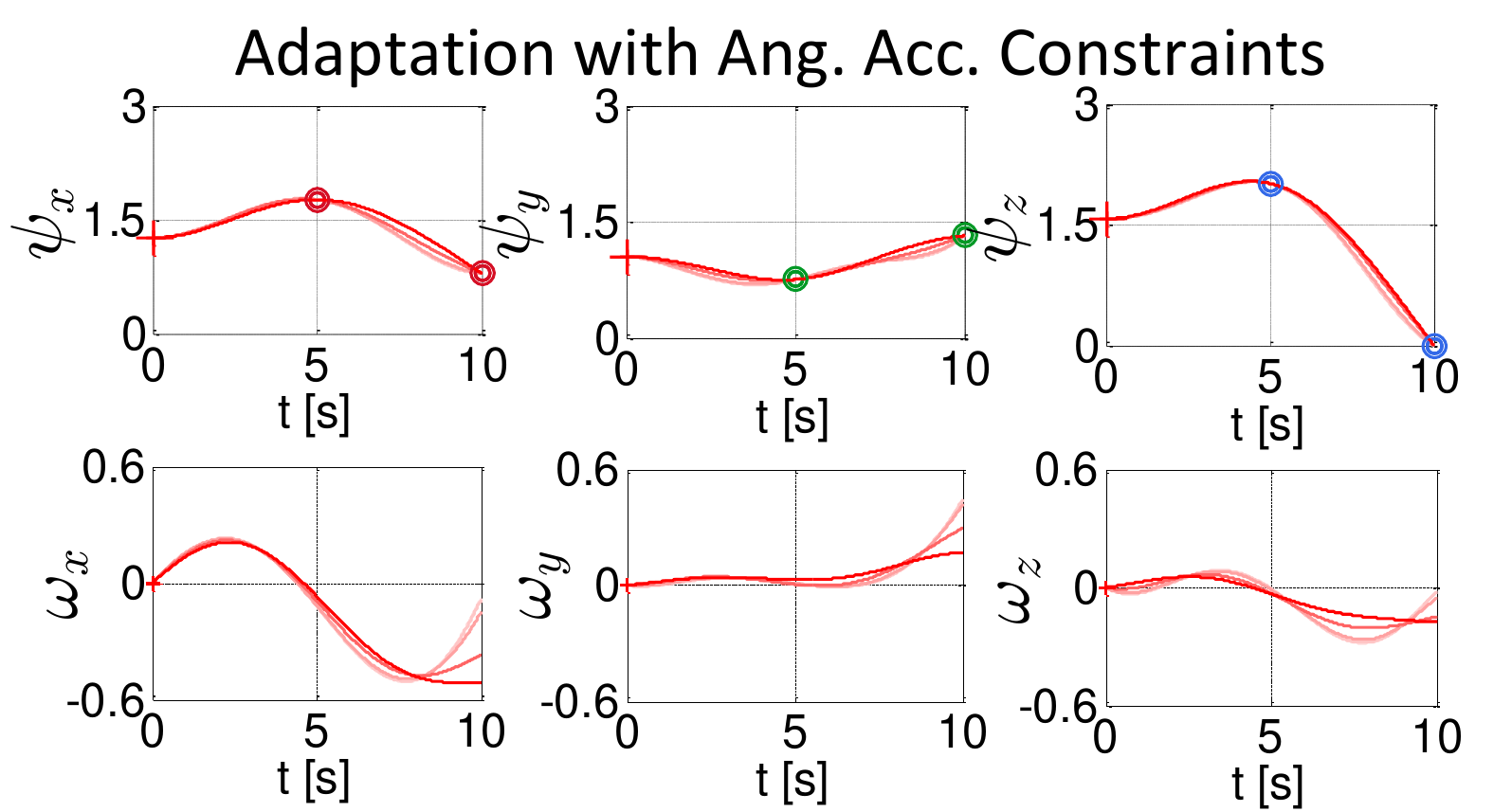}
    \caption{Adapting orientations while considering angular acceleration constraints. Circles plot the desired points. `+' marks the initial points of trajectories. The curves evolving from light red to dark red correspond to the increasing direction of $\lambda_a$.}
    \label{Fig_s5_PsiVel_Acc}
\end{figure}

\begin{figure}[!t]
    \centering
    \includegraphics[width = \hsize]{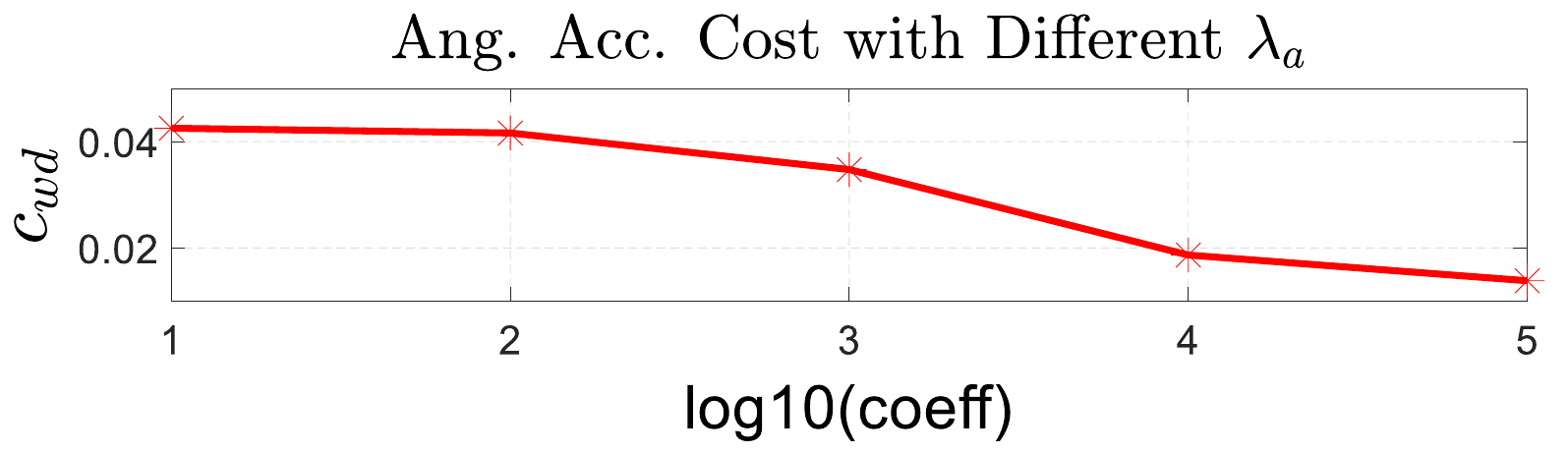}
    \caption{The angular acceleration costs under different $\lambda_a$. The abscissa is shown in log-scale, i.e., $\log_{10}\left(\lambda_a\right)$.}
    \label{Fig_s8_jerkCost}
\end{figure}

Figure \ref{Fig_s5_PsiVel_Acc} plots the adapted orientation trajectories as well as their corresponding angular velocities. The change of color from light to dark indicates the update direction as the coefficient $\lambda_a$ increases.
We can see in Fig. \ref{Fig_s5_PsiVel_Acc} that the adapted orientation trajectories can pass through desired points precisely. Meanwhile, the angular velocities become smoother when $\lambda_a$ is larger. The corresponding angular acceleration costs are reported in Fig. \ref{Fig_s8_jerkCost}, where the acceleration cost is defined as
\begin{align}
    \label{Eq_AngCost}
    \begin{split}
        c_{wd} = \frac{1}{N}\sum^{N}_{n = 1} {\sqrt{ \left\|\dot{\vec \omega} \left(t_n\right) \right\|^2 }}.
    \end{split}
\end{align}
Here the angular velocity is calculated on the manifold by:
\begin{align}
    \label{Eq_AngularVelocity}
    \begin{split}
        {\vec \omega}\left({t_n}\right) &= \frac{ {\vec R}_{t_{n}} \log\left( {\vec R}^\top_{t_{n}} {\vec R}_{t_{n+1}} \right) } {dt}, \\
        \dot {\vec \omega}\left({t_n}\right) &= \frac{d}{dt} {\vec \omega}\left({t_n}\right),
    \end{split}
\end{align}
with ${\vec R}_{t_{n}}$ and ${\vec R}_{t_{n+1}}$ being the last and current orientations, respectively, and $dt$ is the time duration for two adjacent points in the orientation trajectory.
The decreasing trend of $c_{wd}$ in Fig.~\ref{Fig_s8_jerkCost} coincides with our expectation as the larger the $\lambda_a$ is, the larger the impact of acceleration minimization is.
% Thus, we can conclude that our approach inherits the capability of adapting orientations towards various desired points, while incorporating angular acceleration constraints.
% The orientation trajectories (in Angle-Axis Space) and the corresponding velocities are shown in Fig. \ref{Fig_s5_PsiVel_Acc}, where the color changes from light to dark as $\lambda_a$ increases. As shown in Fig. \ref{Fig_s5_PsiVel_Acc}, the reproduced orientation trajectories pass through various desired points precisely. Also, the angular velocities become more smooth with the increasing of $\lambda_a$. The corresponding angular acceleration costs with different $\lambda_a$ are also provided in Fig. \ref{Fig_s8_jerkCost}, representing that $c_{wd}$ decreases as $\lambda_a$ increases. Thus, we can conclude that our approach inherits the capability of adapting orientations towards various desired points, while incorporating angular acceleration constraints.

\subsection{Evaluation on Adaptation with Single IOVP}
\label{subsection_evaluation_c}

In this subsection, we aim to validate the capability and benefits of incorporating the IOVP, which is inapplicable for previous methods. The same demonstrations, as depicted in Fig. \ref{Fig_s1_Demonstrations}, are employed. The desired points, which are randomly picked,  are given as: 
\begin{align}
    \label{Eq_via-point-SingleIOVP}
    \begin{split}
    \tilde{t}_0 = 0.0, \tilde{\vec \psi}_0 &= [1.2614, 1.0512, 1.5767]^\top, \\ \tilde{\vec \omega}_0 &= [0.0, 0.0, 0.0]^\top; \\
    \tilde{t}_1 = 5.0, \tilde{\vec \psi}_1 &= [0.7028, 1.1713, 0.4685]^\top, \\ \tilde{\vec \omega}_1 &= \vec{R}^\top_{d1}[0.0, 0.0, 0.3]^\top;\\
    \tilde{t}_2 = 10.0, \tilde{\vec \psi}_2 &= [0.9137, 1.3705, 0.9137]^\top, \\ \tilde{\vec \omega}_2 &= [0.0, 0.3, 0.3]^\top. 
    \end{split}
\end{align}
Similar to Section \ref{subsection_evaluation_b}, the angular acceleration constraints are also taken into account, with 5 different values of $\lambda_a$. In order to incorporate the IOVP, $\vec{R}_{d1} = \exp\left( \tilde{\vec \psi}_1 \right)$ is set as the auxiliary frame.

\begin{figure}[!b]
    \centering
    \includegraphics[width = 0.98\hsize]{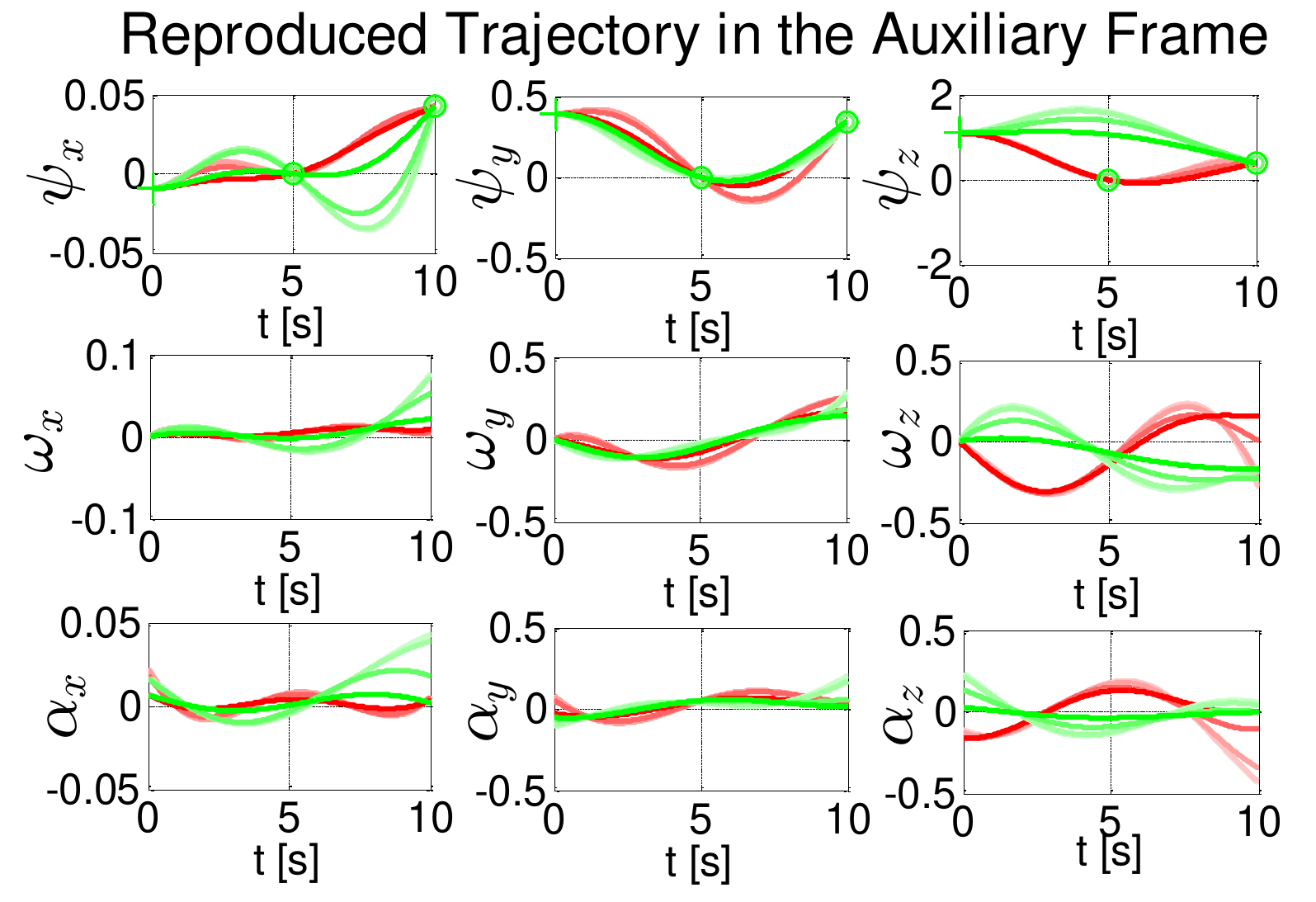}
    \caption{The reproduced trajectories expressed in the auxiliary frame, which have been transformed into the Angle-Axis Space. Here the red curves are the results without incorporating the IOVP. The green curves are the results obtained by incorporating IOVP. The color changes from light to dark as $\lambda_a$ increases.}
    \label{Fig_s9_IOCPreAux}
\end{figure}

\begin{figure}[!t]
    \centering
    \includegraphics[width = 0.65\hsize]{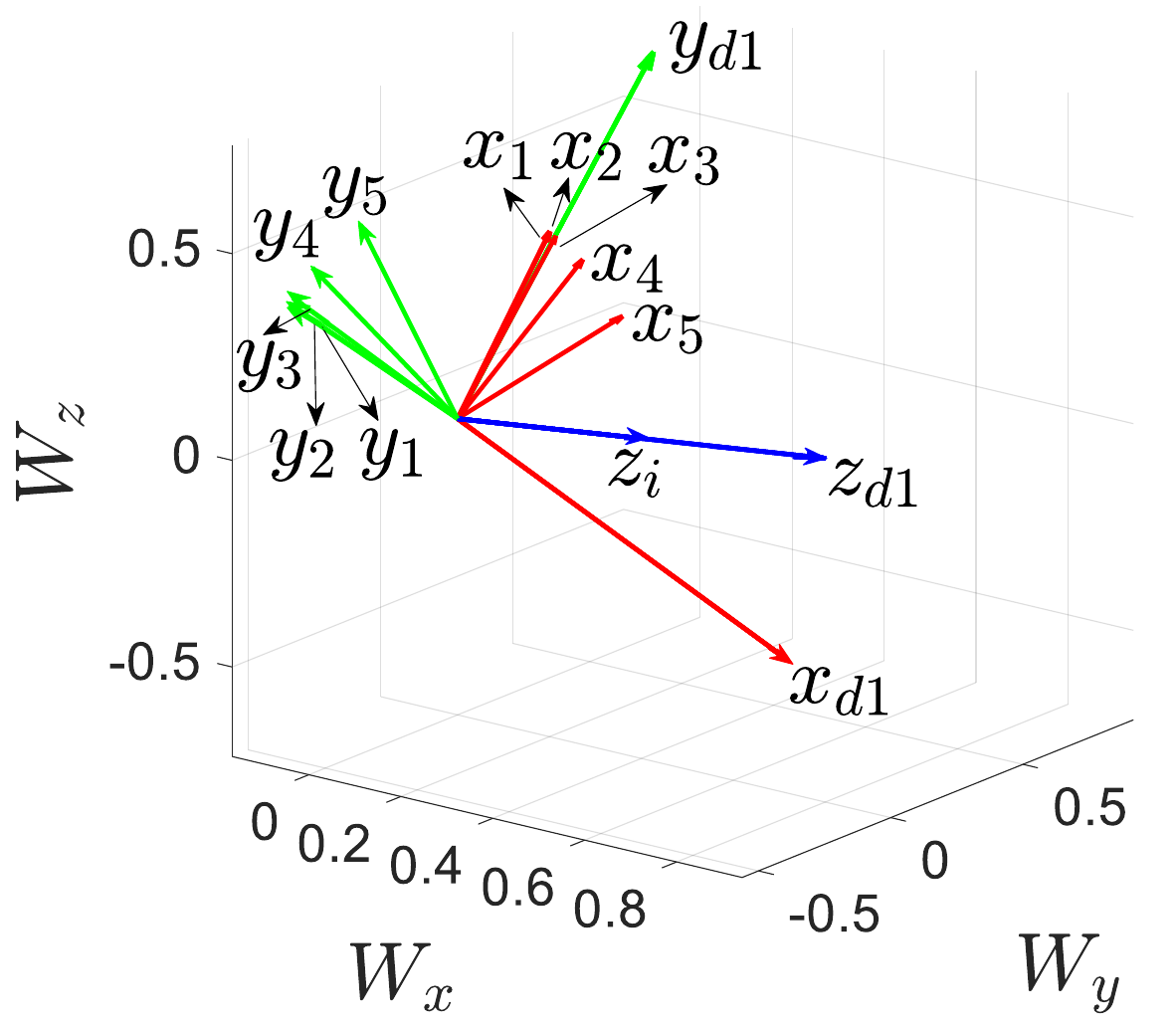}
    \caption{The desired orientation $R_{d1}$ ($\{x_{d1} y_{d1} z_{d1}\}$) and the reproduced orientations ($\{x_i y_i z_i\}, \, i = 1, 2, \cdots, 5$) at $\tilde{t}_1$ with 5 different $\lambda_a$. All $z_i$ share the same $z$ direction with $z_{d1}$, which means that the strict constraints in $x$ and $y$ are guaranteed. However, the rotation around the $z$ direction is released, indicated by the different $x_i$ and $y_i$.}
    \label{Fig_s11_frames_IOVP}
\end{figure}

\begin{figure}[!t]
    \centering
    \includegraphics[width = \hsize]{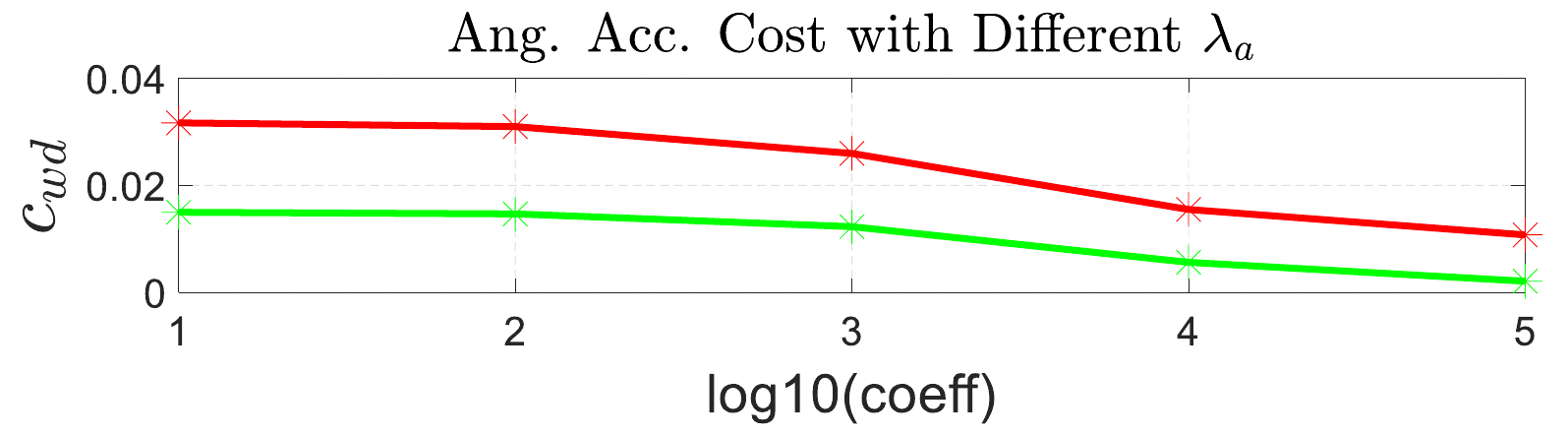}
    \caption{The angular acceleration costs with (green) and without (red) considering IOVP. By incorporating IOVP, smaller angular acceleration costs are achieved.}
    \label{Fig_s10_jerkCost_IOVP}
\end{figure}

To better demonstrate the benefits by considering the IOVP, two adaptation tasks are carried out. In the first one, $\vec{R}_{d1}$ is regarded as a normal via-point. Therefore, all the covariances $\tilde{\vec \Sigma}_{rk}$ ($k = 0, 1, 2$) are set as the same as those adopted in Section \ref{subsection_evaluation_b}. In the second one, $\vec{R}_{d1}$ is regarded as an IOVP. Therefore, $\tilde{\vec \Sigma}_{r1}$ is adjusted as $\tilde{\vec \Sigma}_{r1} = \text{diag}\left(10^{-10}, 10^{-10}, 10^{3}\right)$ to release the constraint on $z$ direction, while the other parts keep the same.

The reproduced trajectories expressed in the auxiliary frame are provided in Fig. \ref{Fig_s9_IOCPreAux}. As shown in the first row of Fig. \ref{Fig_s9_IOCPreAux}, the red curves pass through all the desired circles since $\vec{R}_{d1}$ is regarded as a normal via-point. In contrast, the green curves in $x$ and $y$ directions pass through the desired circles since a high level of precision is set in the $x$ and $y$ directions, while the constraint on the $z$ direction is released.

To better validate the effectiveness, the orientation $\vec{R}_{d1}$ and the reproduced orientations at $\tilde{t}_1$ with 5 different $\lambda_a$ are plotted in Fig. \ref{Fig_s11_frames_IOVP}. The $z$ directions of all the reproduced orientations remain the same as $z_{d1}$, proving that the levels of precision on the $x$ and $y$ directions are high enough. However, the rotation around the $z$ direction is free. The comparisons of angular acceleration costs between the two tasks are provided in Fig. \ref{Fig_s10_jerkCost_IOVP}, showing that we can achieve smaller angular acceleration costs by incorporating the IOVP. These results validate that releasing constraints on the unnecessary direction gives us more flexibility. By incorporating the IOVP, we can achieve extra benefits,  e.g., achieving smaller angular acceleration costs. It is also possible to find a feasible solution when $\vec{R}_{d1}$ is unreachable.

\subsection{Evaluation on Adaptation with Multiple IOVPs}
\label{subsection_evaluation_D}

In this subsection, we aim to validate the capability of our proposed approaches to incorporate multiple IOVPs. In addition, we would like to validate the smoothness of the final reproduced trajectory by using our proposed weighted rotation average mechanism. The same demonstrations, as depicted in Fig. \ref{Fig_s1_Demonstrations}, are employed. 

\begin{figure}[t]
    \centering
    \includegraphics[width = \hsize]{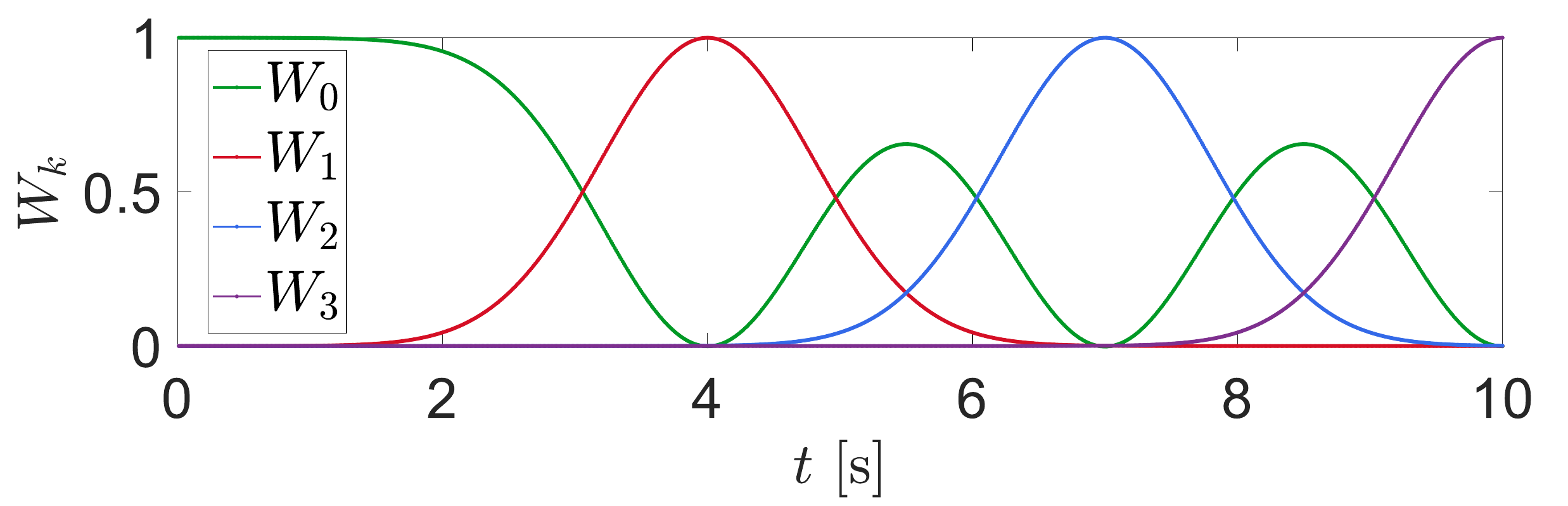}
    \caption{The design of the weighted curve $W_k\left(t\right)$. In this example, $\Delta t = 2.4$s, $\bar{t}_1 = 4.0$s, $\bar{t}_2 = 7.0$s, $\bar{t}_3 = 10.0$s, $K = 3$. }
    \label{Fig_GaussianWeightCurve-Experiments}
\end{figure}

\begin{figure*}[t]
    \centering
    \subfigure[]{\includegraphics[width = 0.49\hsize]{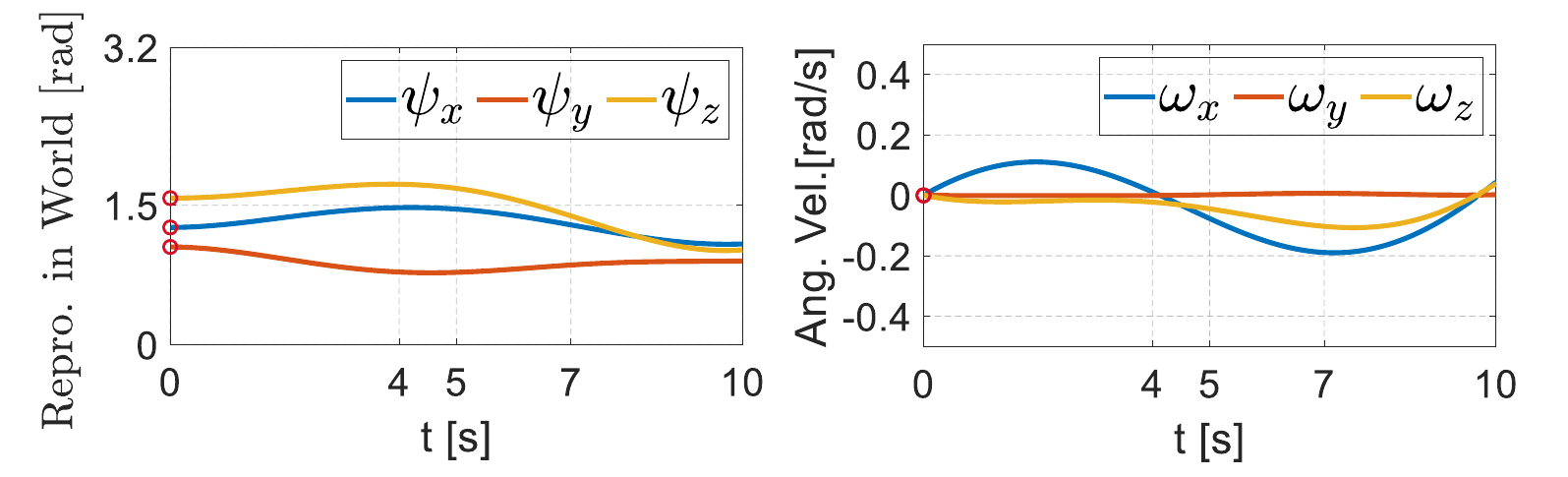}}
    \subfigure[]{\includegraphics[width = 0.49\hsize]{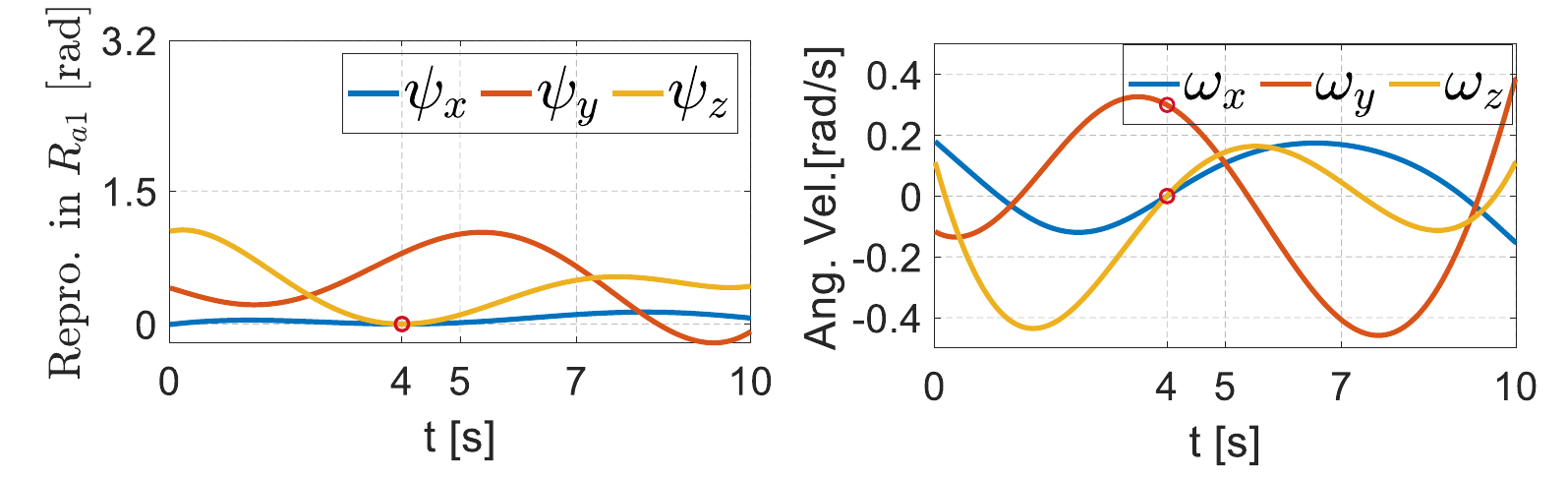}}
    \subfigure[]{\includegraphics[width = 0.49\hsize]{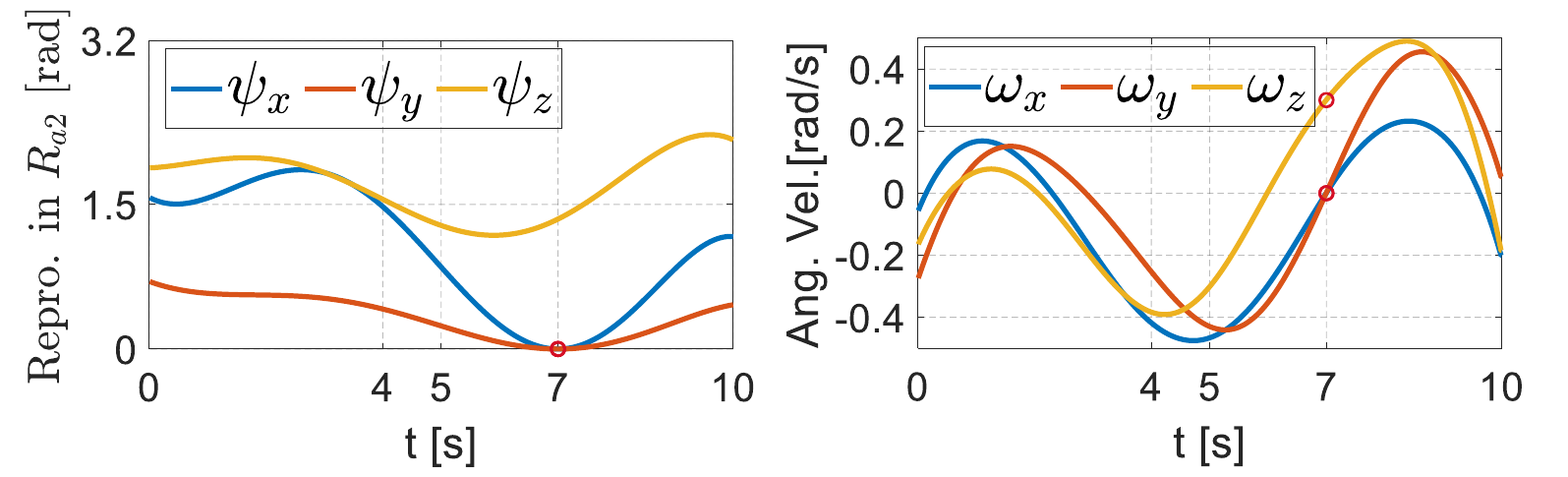}}
    \subfigure[]{\includegraphics[width = 0.49\hsize]{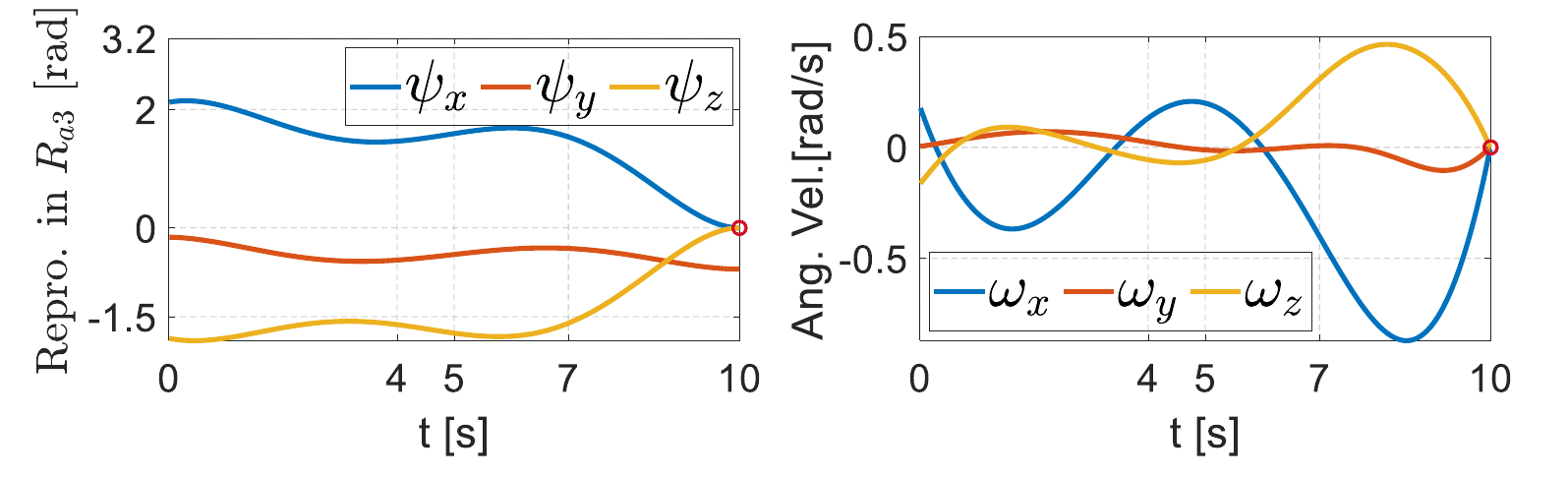}}
    \caption{ Reproduced trajectories expressed in different frames. (a) The base trajectory ${\vec R}_0\left(t\right)$ by considering the starting via-point. This trajectory is expressed in the world frame. (b) The trajectory ${\vec R}_1\left(t\right)$, (c) ${\vec R}_2\left(t\right)$, and (d) ${\vec R}_3\left(t\right)$ by considering the three IOVPs, respectively. These trajectories are expressed in different auxiliary frames. The desired via-points are depicted by circles. Since the desired via-points are chosen as auxiliary frames in (b), (c), and (d), the desired via-points are all transformed to identity ${\vec I}_3$. Therefore, in (b), (c), and (d), all red circles are located at position [0; 0; 0]. }
    \label{Fig_s-d1_ReproTraj-inAuxiliaryFrames}
\end{figure*}

Four desired via-points are given as follows:  
\begin{align}
    \label{Eq_via-point-MultiIOVPs}
    \begin{split}
        \tilde{t}_0 = 0.0, \tilde{\vec \psi}_0 &= [1.2614, 1.0512, 1.5767]^\top, \\ \tilde{\vec \omega}_0 &= [0.0, 0.0, 0.0]^\top;\\
        \tilde{t}_1 = 4.0, \tilde{\vec \psi}_1 &= [0.7028, 1.1713, 0.4685]^\top, \\ \tilde{\vec \omega}_1 &= [0.0069, 0.2103, 0.2138]^\top; \\
        \tilde{t}_2 = 7.0, \tilde{\vec \psi}_2 &= [-0.5236, 0, 0]^\top, \\ \tilde{\vec \omega}_2 &= [0.0, 0.1500, 0.2598]^\top; \\
        \tilde{t}_3 = 10.0, \tilde{\vec \psi}_3 &= [0.0, 2.2214, -2.2214]^\top, \\ \tilde{\vec \omega}_3 &= [0.0, 0.0, 0.0]^\top
    \end{split}
\end{align}
For the first and the fourth via-points, the desired velocity is set as zero such that the robot can be started/stopped more smoothly. In addition, to validate the generalization to large variation, the three via-points $\tilde{\vec \psi}_k$ for $k = 1, 2, 3$ are selected to be far from the demonstration data. More specifically, $\tilde{\vec \psi}_3 = [0.0, 2.2214, -2.2214]^\top$ is selected to be on the boundary of ${\vec B}_\pi$, reflected by the fact that its norm equals $\pi$. By using these given via-points, the performance of our proposed imitation learning framework and our memory-based weighted rotation average mechanism can be validated.

Among them, the first via-point is set as a normal starting via-point. Therefore, the covariance $\tilde{\vec \Sigma}_{r0}$ and $\tilde{\vec \Sigma}_{v0}$ are both set as $\text{diag}\left(10^{-10}, 10^{-10}, 10^{-10}\right)$. While the other via-points are set as IOVPs. In order to better validate the capability of our proposed approach, the first and third IOVPs choose to release the constraints on the $y$-axis, while the second IOVP chooses to release the constraints on $z$-axis. Therefore, the covariance $\tilde{\vec \Sigma}_{rk}$ and $\tilde{\vec \Sigma}_{vk}$ for $k = 1, 3$ are set as $\text{diag}\left(10^{-10}, 10^{3}, 10^{-10}\right)$ and $\text{diag}\left(10^{-10}, 10^{-10}, 10^{-10}\right)$, respectively. And the covariance $\tilde{\vec \Sigma}_{r2}$ and $\tilde{\vec \Sigma}_{v2}$ are set as $\text{diag}\left(10^{-10}, 10^{-10}, 10^{3}\right)$ and $\text{diag}\left(10^{-10}, 10^{-10}, 10^{-10}\right)$, respectively. In this test, the acceleration cost is not considered. Therefore, the covariance $\tilde{\vec \Sigma}_{ak}$ is ignored.

The weighted curves are designed as shown in Fig. \ref{Fig_GaussianWeightCurve-Experiments}. As shown, $W_k\left(t\right)$ are all Gaussian curves centered at ${\tilde t}_k$ for $k = 1, 2, 3$. They are the weights corresponding to the reproduced trajectories ${\vec R}_k\left(t\right)$ for considering the first, second, and third IOVPs, respectively. Meanwhile, $W_0\left(t\right) = 1 - \sum^{3}_{k=1} {W_k\left(t\right)}$ is the weight curve for the base trajectory ${\vec R}_0\left(t\right)$ by considering the starting via-point. All weighted curves are smooth to guarantee the smoothness of the final reproduced trajectory.

\begin{figure*}[t]
    \centering
    \subfigure[]{\includegraphics[width = 0.49\hsize]{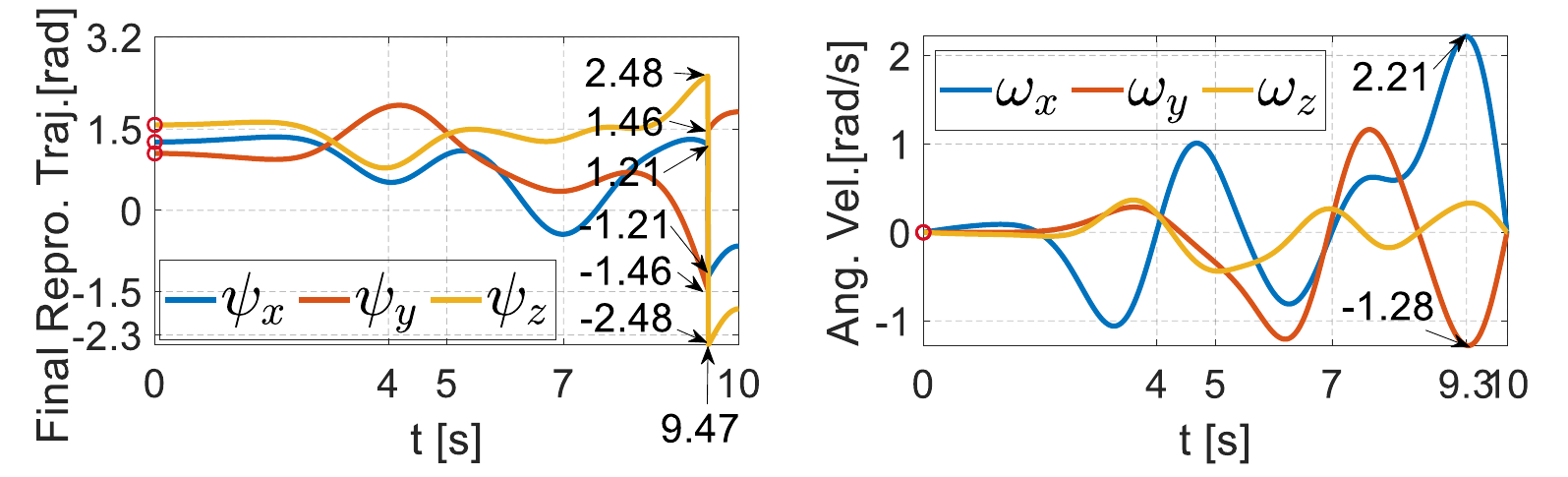}}
    \subfigure[]{\includegraphics[width = 0.49\hsize]{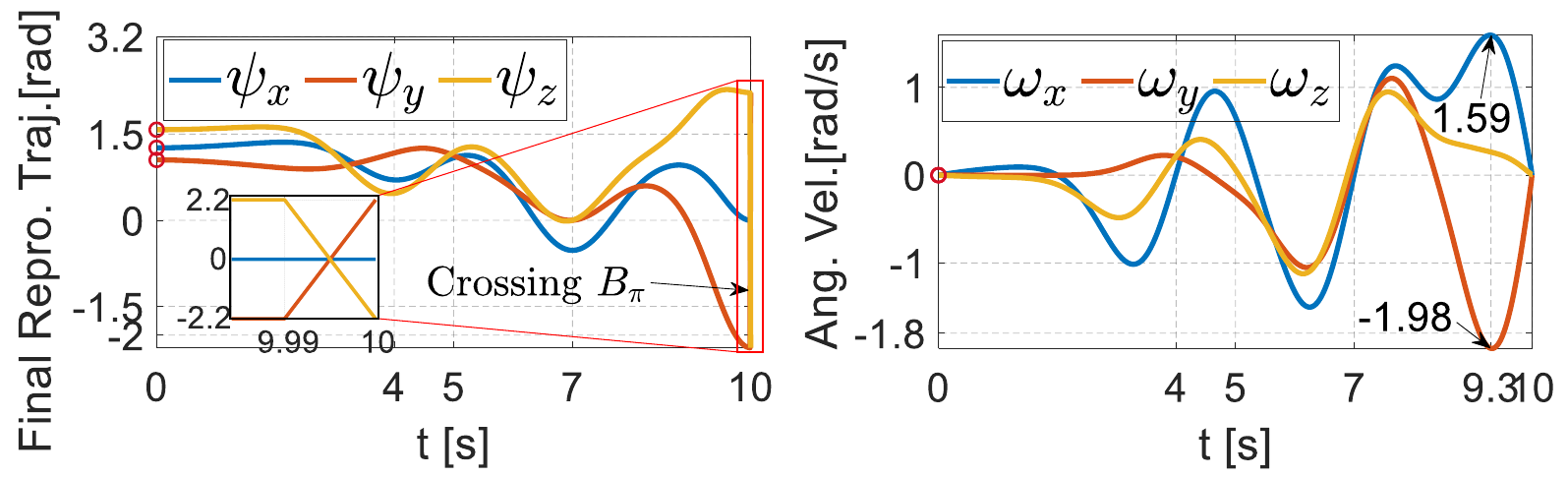}}
    \subfigure[]{\includegraphics[width = 0.49\hsize]{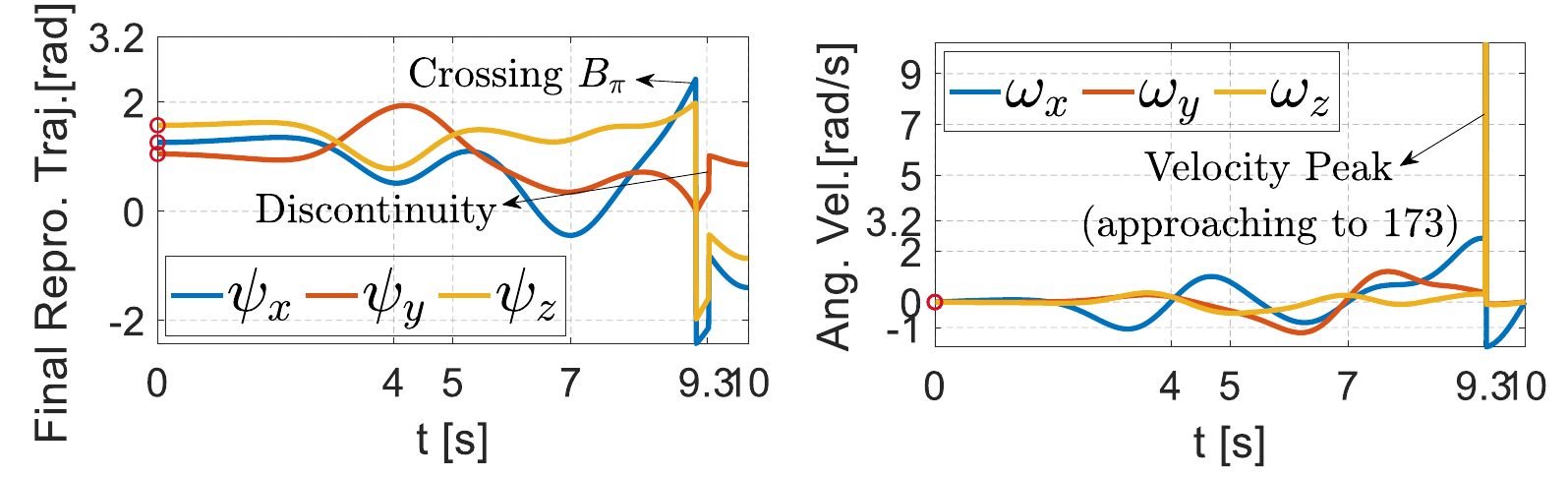}}
    \subfigure[]{\includegraphics[width = 0.49\hsize]{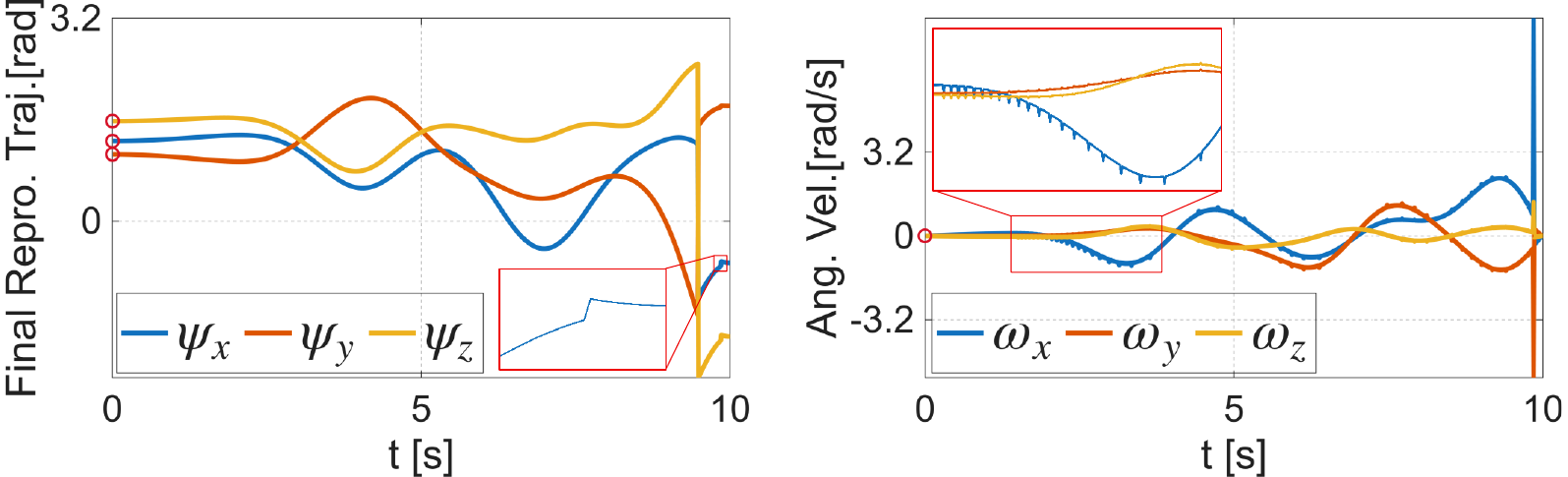}}
    \caption{Final reproduced trajectory with (a) and without (b) considering the IOVPs. Please note that the angular velocity along the $y$-axis is decreased from $-1.98$rad/s to $-1.28$rad/s when considering the IOVP. This validates the effectiveness of our proposed method, since the last IOVP releases the constraint on the $y$-axis. The smoothness of the whole trajectory is also increased, reflected by the decrease in acceleration cost from 1.6832 to 1.4632. (c) The final reproduced trajectory for ${\vec R}_{a3} = \exp\left( \tilde{\vec \psi}_3 \right) \exp\left([0; \frac{2\pi}{3}; 0]\right)$ by using Eq. (\ref{Eq_weightSumInRiemannian}), without our memory-based algorithm. The weighted average mechanism encounters a discontinuity problem. The acceleration cost reaches 18.9836. (d) The final reproduced trajectory by using Karcher flow. The shape of the reproduced trajectory by the Karcher flow is similar to that of ours. This indicates that our memory-based algorithm achieves the same averaging results as the Karcher flow. But the Karcher flow encounters a discontinuity problem at $t = 9.80s$. Its acceleration cost (2.7621) is also approximately twice that of our method.}
    \label{Fig_s-d5_FinalReproTraj}
\end{figure*}

\begin{figure*}[!t]
    \centering
    \subfigure[]{\includegraphics[width = 0.32\hsize]{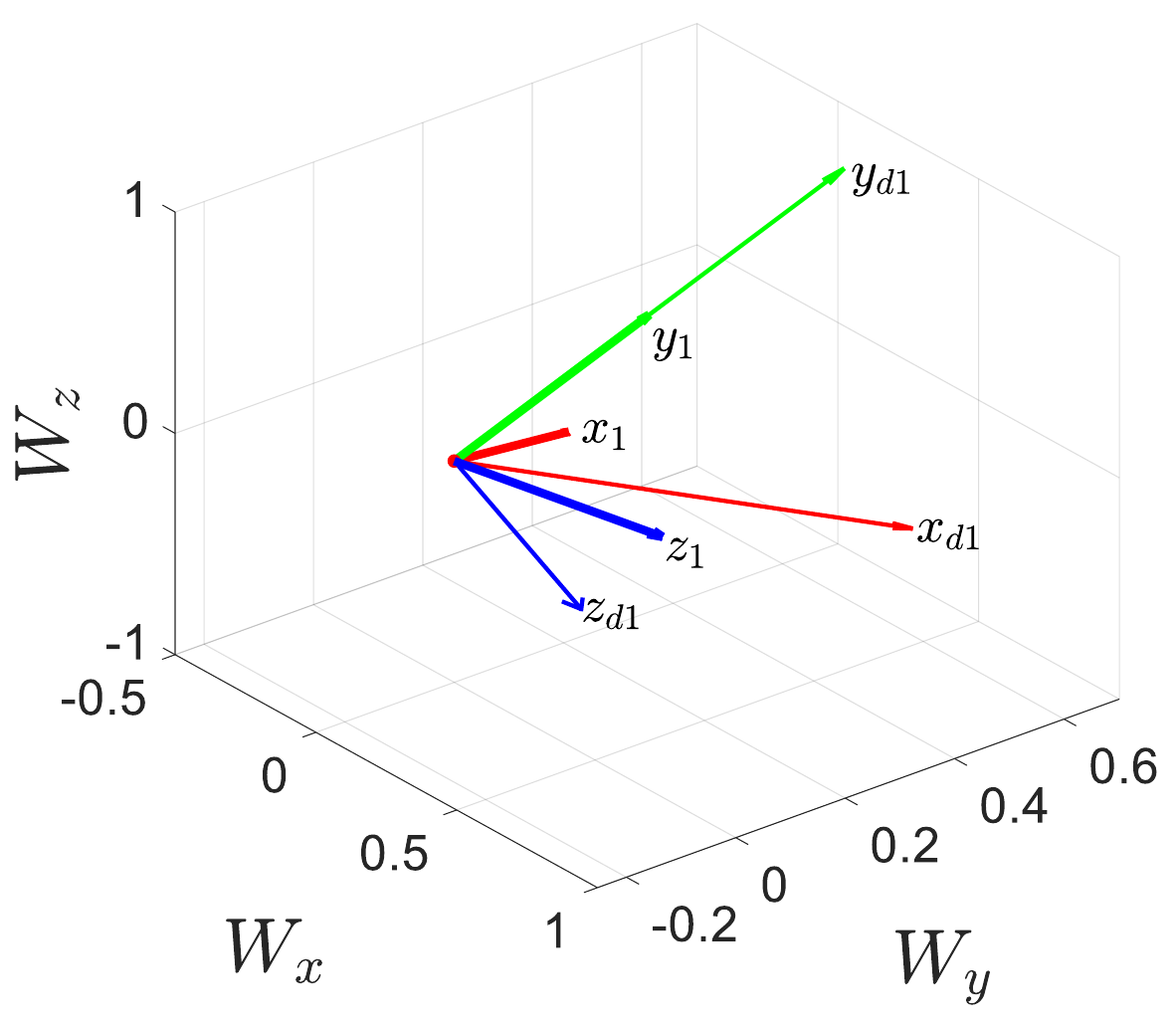}}
    \subfigure[]{\includegraphics[width = 0.32\hsize]{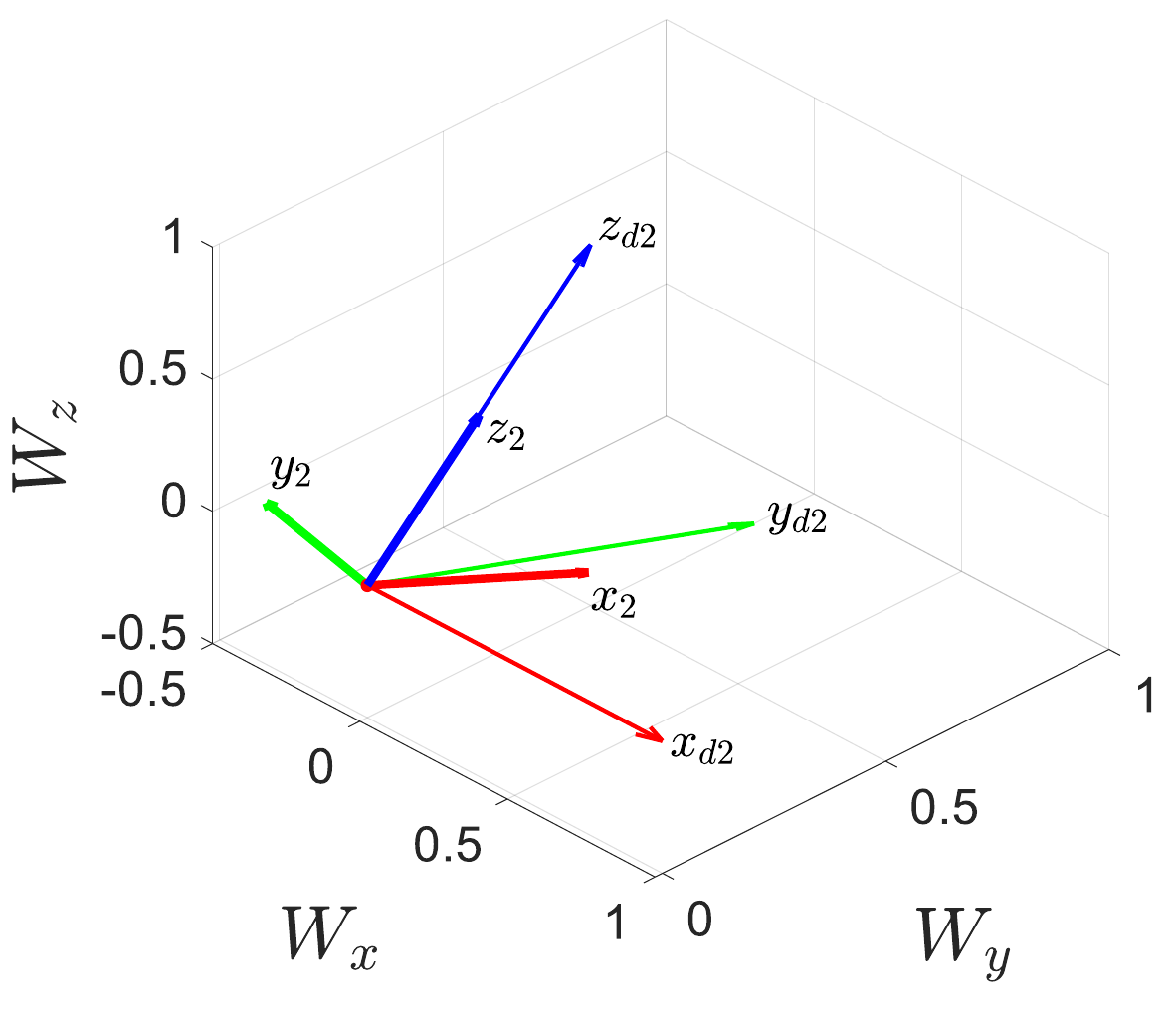}}
    \subfigure[]{\includegraphics[width = 0.32\hsize]{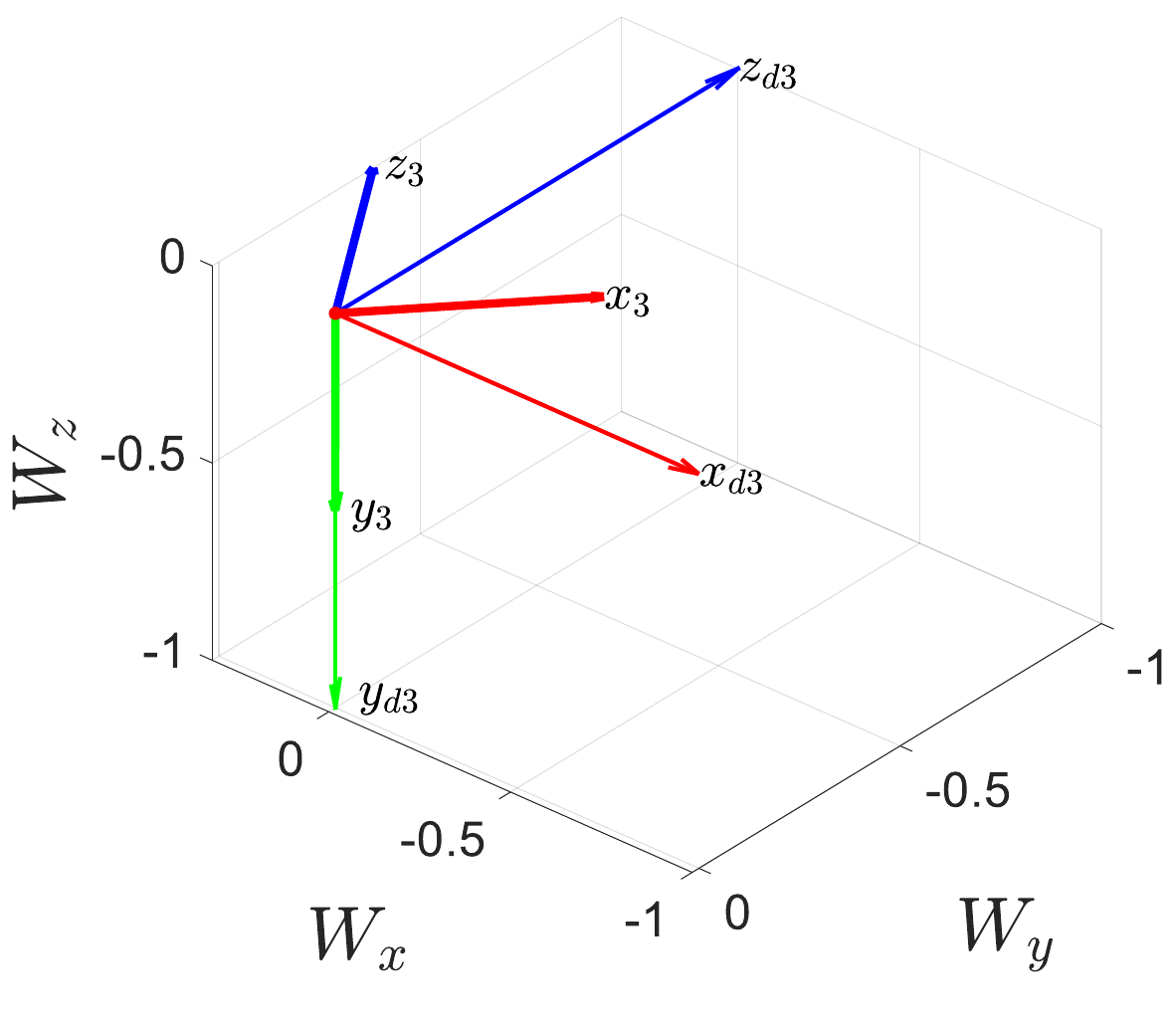}}
    \subfigure[]{\includegraphics[width = 0.32\hsize]{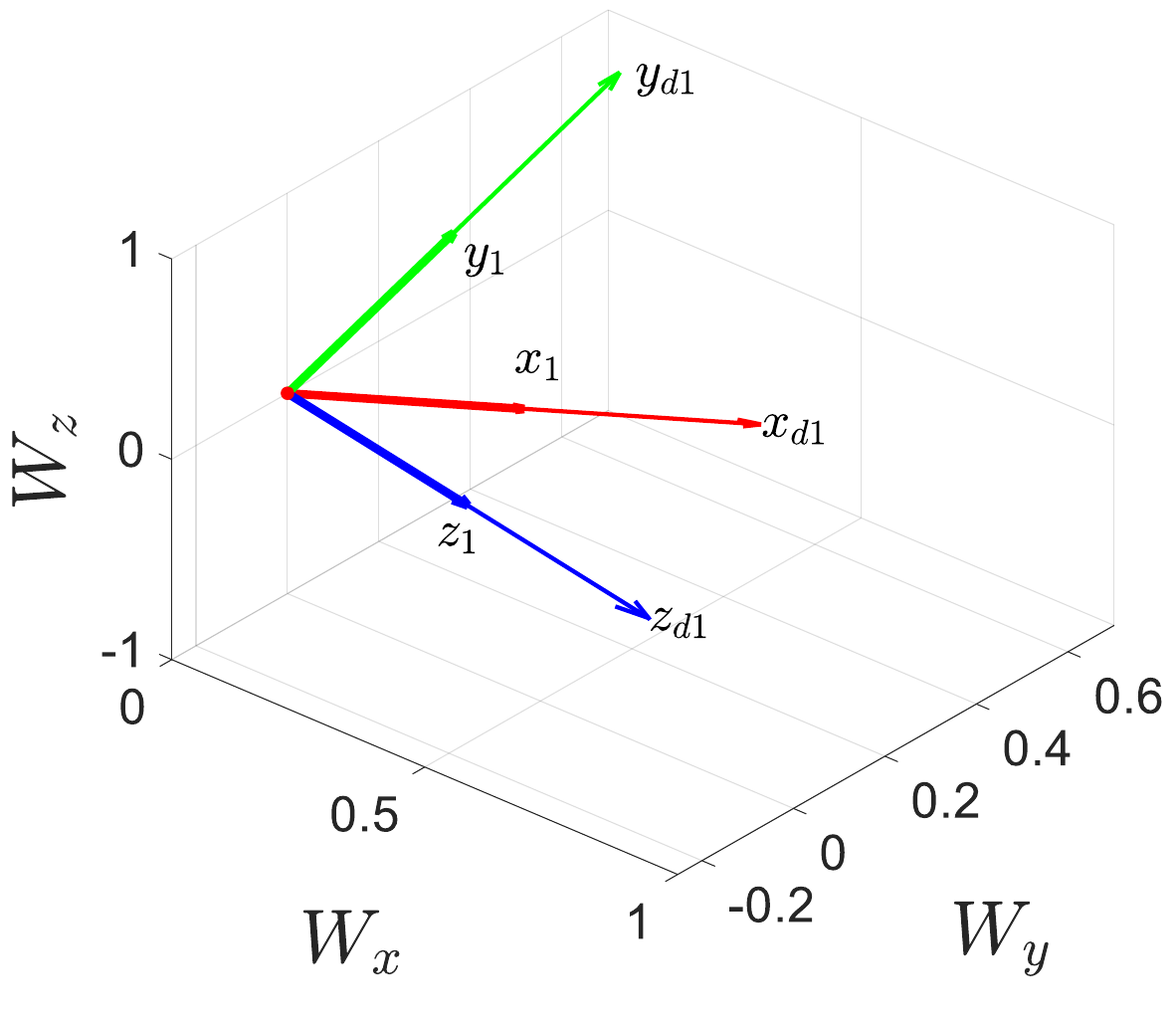}}
    \subfigure[]{\includegraphics[width = 0.32\hsize]{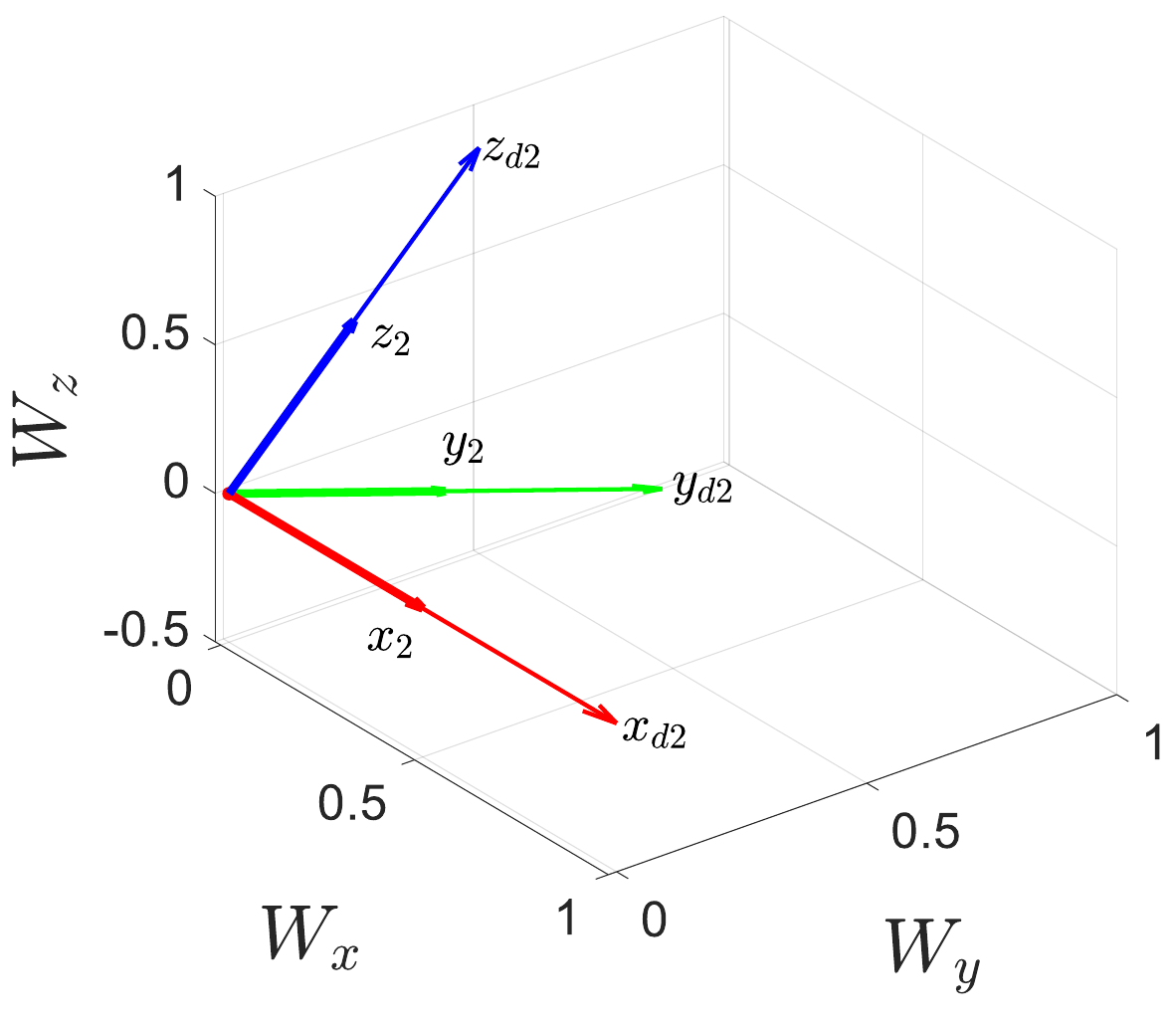}}
    \subfigure[]{\includegraphics[width = 0.32\hsize]{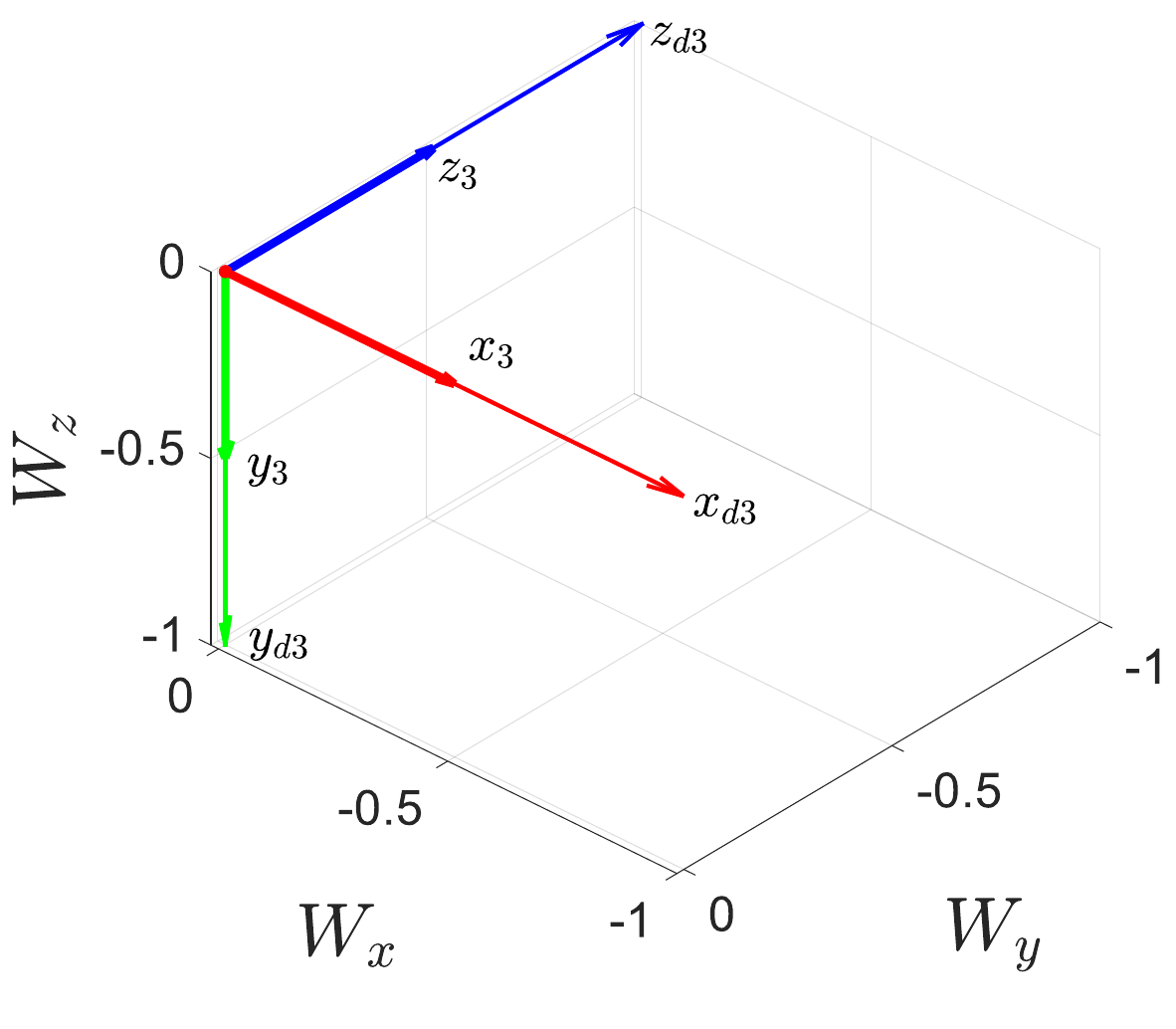}}
    \subfigure[]{\includegraphics[width = 0.32\hsize]{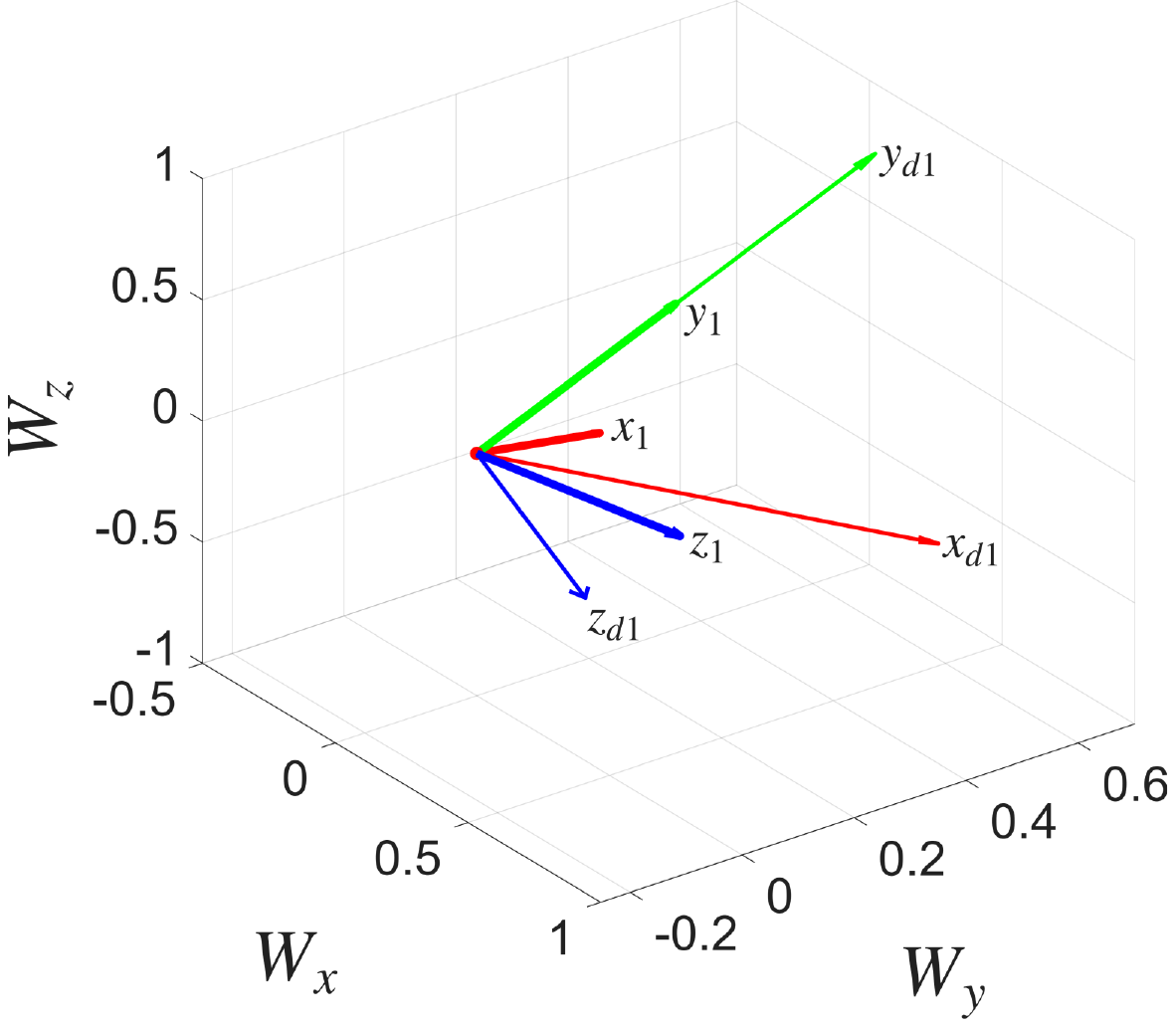}}
    \subfigure[]{\includegraphics[width = 0.32\hsize]{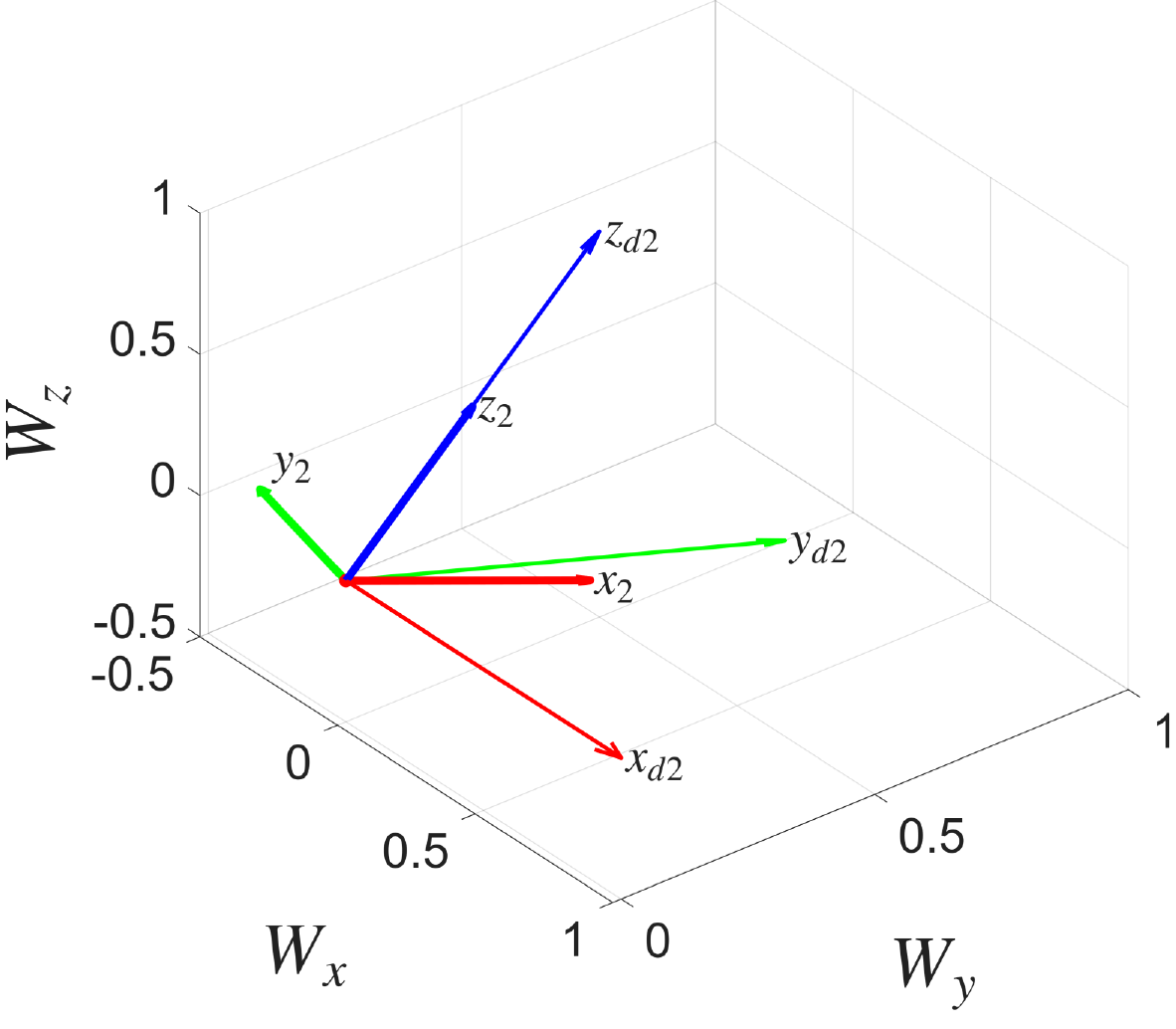}}
    \subfigure[]{\includegraphics[width = 0.32\hsize]{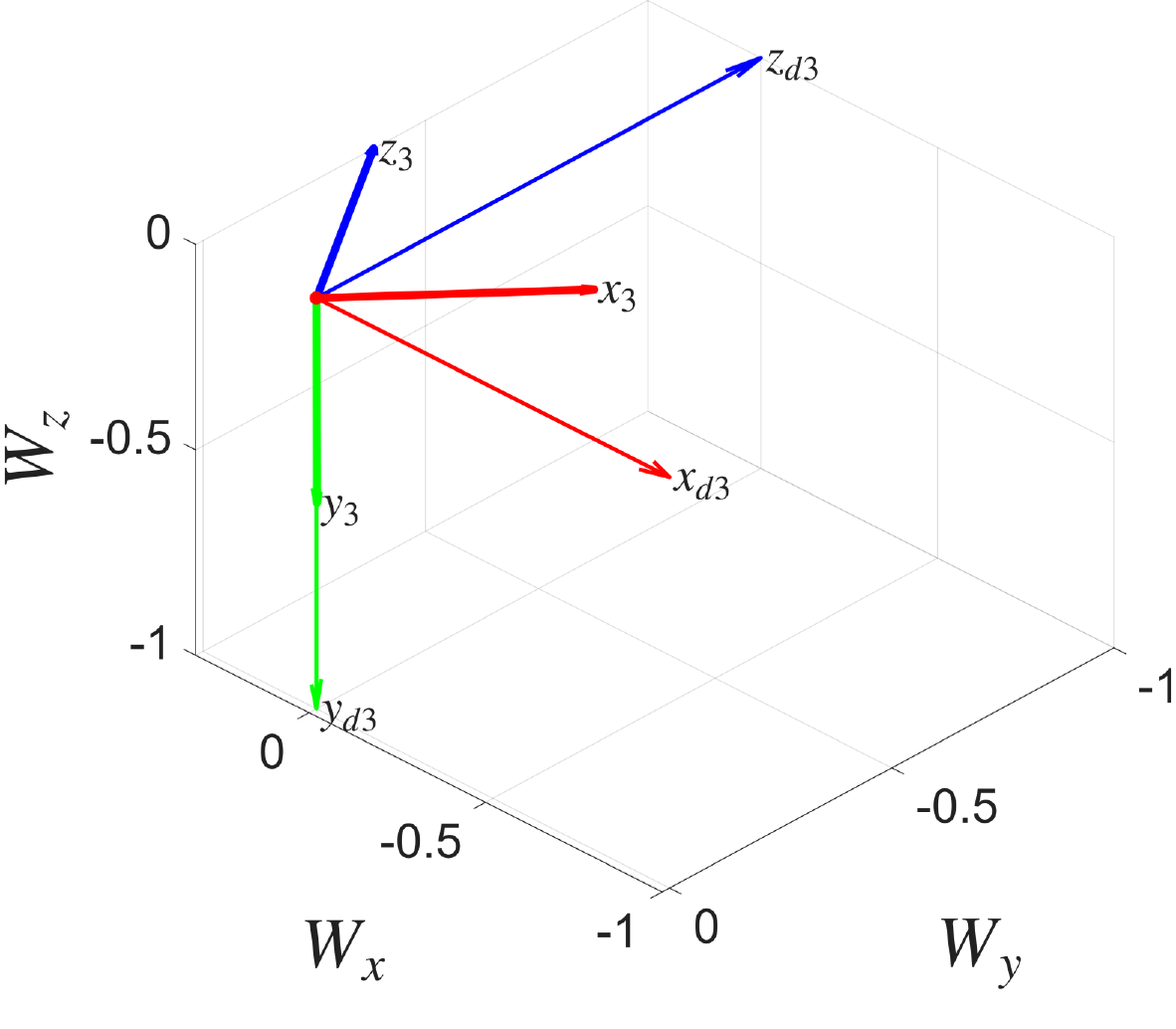}}
    \caption{ (a), (b), and (c) are results with considering IOVPs. (d), (e), and (f) are results without considering IOVPs. (g), (h), and (i) are results by using the Karcher flow. The first, second, and third columns are the reproduced orientations at time $\tilde{t}_1$, $\tilde{t}_2$, and $\tilde{t}_3$, respectively. The frames with longer arrows show the desired via-point ${\vec R}_{ak}$ and the frames with shorter arrows show the reproduced orientation. Although the reproduced orientations are perfectly aligned with the desired ones in (d)(e)(f), the acceleration cost is increased from 1.4632 to 1.6832, leading to the fluctuations of the final reproduced trajectory. In contrast, by considering IOVPs, the strict constraints are guaranteed and the constraints on the relaxed DoF are released, by using our proposed memory-based algorithm and the Karcher flow. In addition, by comparing the first row with the third row, the same reproduced orientations validate the correctness of our memory-based algorithm. All these results can validate that our designed Gaussian weighted curves can satisfy the dominance of each trajectory at the given local time domain. }
    \label{Fig_s-d6-d12_IOVPs}
\end{figure*}

The demonstration data is projected into 4 different tangent spaces by using 4 auxiliary frames ${\vec R}_{ak} = \exp\left({\tilde{\vec \psi}}_k\right)$ for $k = 0, 1, 2, 3$. Then four trajectories ${\vec R}_k\left(t\right)$ by incorporating each via-point are reproduced. The reproduced trajectories are shown in Fig. \ref{Fig_s-d1_ReproTraj-inAuxiliaryFrames}. For the convenience of validating the via-points, ${\vec R}_1\left(t\right)$, ${\vec R}_2\left(t\right)$, and ${\vec R}_3\left(t\right)$ are displayed in the auxiliary frames. As shown in Fig. \ref{Fig_s-d1_ReproTraj-inAuxiliaryFrames}, the desired via-points are depicted by circles. Since the desired via-points are chosen as auxiliary frames in (b), (c), and (d), the desired via-points are all transformed to the identity ${\vec I}_3$. Therefore, in (b), (c), and (d), all red circles are located at position [0; 0; 0]. We can see that our approach modulates orientation and angular velocity simultaneously to go through the first via-point (as shown in Fig. \ref{Fig_s-d1_ReproTraj-inAuxiliaryFrames}(a)), since it is set as a normal starting via-point. While in Fig. \ref{Fig_s-d1_ReproTraj-inAuxiliaryFrames}(b) and (d), our approach modulates the $x$-th and $z$-th components of orientations and the angular velocities to go through the desired point. But the constraint on the $y$-th component is released. And in Fig. \ref{Fig_s-d1_ReproTraj-inAuxiliaryFrames}(c), the constraint on the $z$-th component is released. These results validate that our proposed approach can successfully incorporate each IOVP.

Then the final reproduced trajectory ${\vec R}\left(t\right)$ is obtained by fusing ${\vec R}_k\left(t\right), k = 0, 1, 2, 3$ with our proposed memory-based weighted rotation average mechanism. The final reproduced trajectory is shown in Fig. \ref{Fig_s-d5_FinalReproTraj}(a). We can see that the final trajectory is smooth, reflected by the smoothness of the angular velocity. This indicates the effectiveness of our proposed weighted rotation average mechanism. More specifically, please note the orientation trajectory at time $t = 9.47s$. At this point, the trajectory of ${\vec \psi}\left(t\right) = \log\left({\vec R}\left(t\right)\right)$ passes through the boundary of ${\vec B}_\pi$, since ${\vec \psi}\left(t = 9.48\right) = -{\vec \psi}\left(t = 9.47\right)$ and $\left\|{\vec \psi}\left(t = 9.47/9.48\right)\right\| \approx \pi$. Although the trajectory  ${\vec \psi}\left(t\right)$ in Angle-Axis Space is not continuous, the trajectory ${\vec R}\left(t\right)$ in SO(3) is continuous. The continuity can be reflected by the smoothness of the angular velocity curves. This indicates that our proposed imitation learning framework and the memory-based weighted average mechanism can address the large distortion problem and can still work for points that go beyond the boundary of ${\vec B}_\pi$. 

The reproduced orientations ${\vec R}\left(t\right)$ at $t = {\tilde{t}}_k$ for $k = 1, 2, 3$ are visualized in Fig. \ref{Fig_s-d6-d12_IOVPs}(a)(b)(c). As we can tell, the $y$ directions of ${\vec R}\left({\tilde{t}}_1\right)$ and ${\vec R}\left({\tilde{t}}_3\right)$ are strictly aligned with the ones of ${\vec R}_{a1} = \exp\left({\tilde{\vec \psi}}_1\right)$ and ${\vec R}_{a3} = \exp\left({\tilde{\vec \psi}}_3\right)$, and the $z$ direction of ${\vec R}\left({\tilde{t}}_2\right)$ is strictly aligned with the one of ${\vec R}_{a2} = \exp\left({\tilde{\vec \psi}}_2\right)$. These results validate that the strict constraints for the three IOVPs are all guaranteed, while the rotation on the relaxed DoF is released.

In order to demonstrate the benefits by considering multiple IOVPs, another two groups of simulations are carried out for comparison. In these tests, the fourth via-point is modified as:
\begin{align}
    \label{Eq_via-point4-comparison}
    \begin{split}
        {\vec R}_{a3} = \exp\left( \tilde{\vec \psi}_3 \right) \exp\left([0; \frac{\left(i-6\right)\pi}{6}; 0]\right),
    \end{split}
\end{align}
for $i = 0, 1, \cdots, 11$. This means that ${\vec R}_{a3}$ is the result of $ \exp\left( \tilde{\vec \psi}_3 \right)$ rotating around $y$-axis by an angle $\frac{\left(i-6\right)\pi}{6}$. In total, 12 tests are carried out for each group with the different choices of $i$.
In the first group, the second, third, and fourth via-points are set as IOVPs. Hence the covariance $\tilde{\vec \Sigma}_{rk}$ and $\tilde{\vec \Sigma}_{vk}$ for $k = 1, 2, 3$ are set as previous. While in the second one, all the via-points are treated as normal via-points. Therefore, all covariances are set as $\text{diag}\left(10^{-10}, 10^{-10}, 10^{-10}\right)$. The acceleration cost, as given in Eq. (\ref{Eq_AngCost}), is chosen as a metric to evaluate the performance.

The comparison results are given in Fig. \ref{Fig_s-d13_AccelerationCost}. As we can tell, although the acceleration cost is ignored in the learning process, the costs by considering IOVPs are much smaller than the ones without considering IOVPs. This validates our claim that we can achieve extra benefits by releasing unnecessary constraints. In addition, the costs given in Fig. \ref{Fig_s-d13_AccelerationCost} are much larger than the ones given in Fig. \ref{Fig_s10_jerkCost_IOVP}. This is because the desired via-points given in these tests are far from the demonstration data. These results also validate our generalization ability in adaptation to via-points that are outside the distribution of demonstrations.

\begin{figure}[!t]
    \centering
    \includegraphics[width = \hsize]{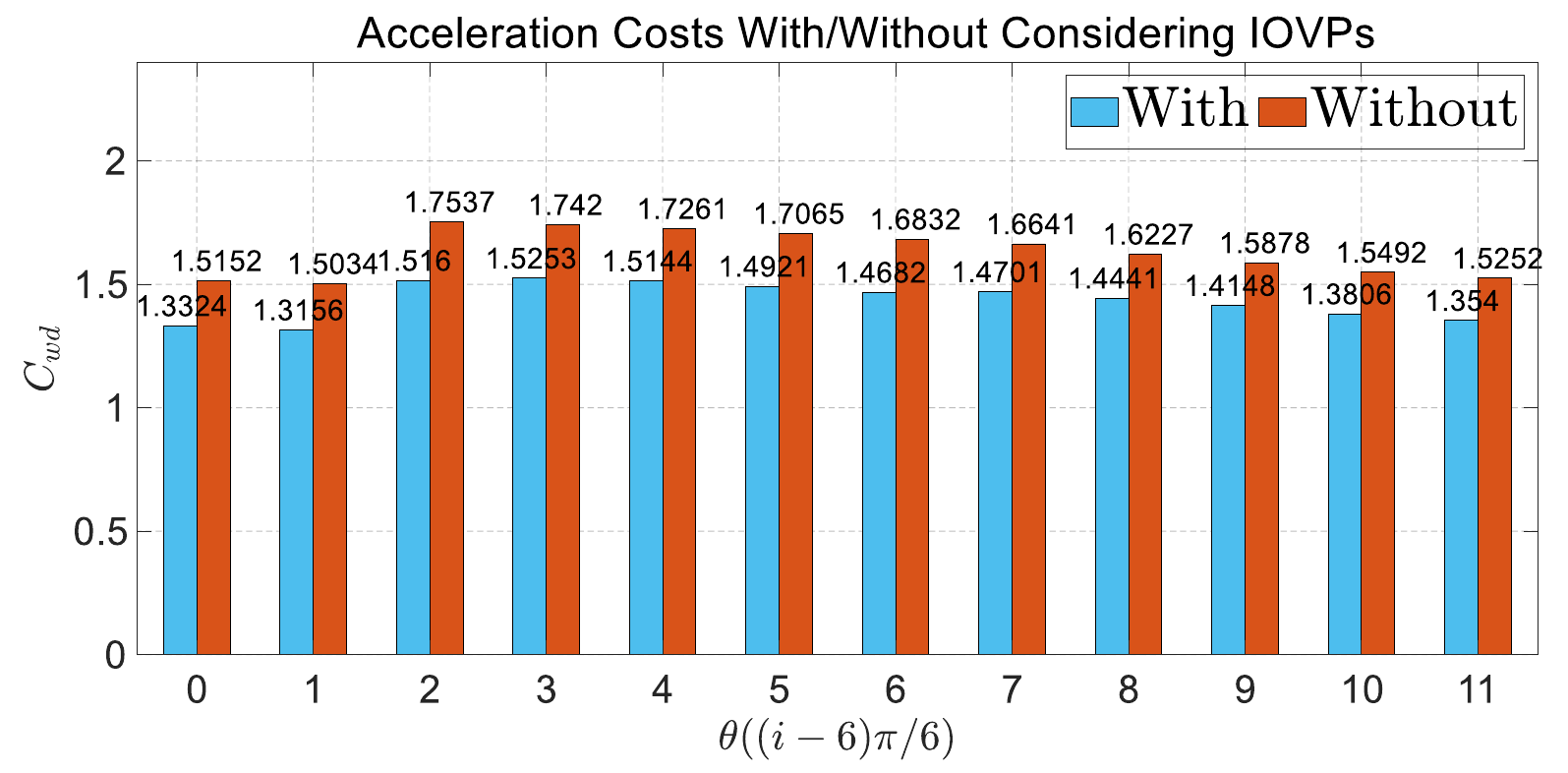}
    \caption{The acceleration costs with and without considering IOVPs. The blue bars are results with considering IOVPs and the red bars are results when changing the IOVPs to normal via-points.}
    \label{Fig_s-d13_AccelerationCost}
\end{figure}

In order to validate the effectiveness of our proposed memory-based algorithm, the final reproduced trajectory for ${\vec R}_{a3} = \exp\left( \tilde{\vec \psi}_3 \right) \exp\left([0; \frac{2\pi}{3}; 0]\right)$ by using Eq. (\ref{Eq_weightSumInRiemannian}) is given in Fig. \ref{Fig_s-d5_FinalReproTraj}(c). Here our memory-based algorithm is not used. As we can tell, the weighted average mechanism encounters a discontinuity problem at time $t = 9.34s$. The highest peak reaches 173\,rad/s and the acceleration cost reaches 18.9836. In contrast, all the reproduced trajectories in the aforementioned two groups are smooth, which validates the effectiveness of our memory-based weighted rotation average mechanism.

The reproduced trajectory without considering IOVPs for $i = 6$ (in this case, the fourth via-point ${\vec R}_{a3} = \exp\left( \tilde{\vec \psi}_3 \right)$ is the same as given in Eq. (\ref{Eq_via-point-MultiIOVPs}) ) is given in Fig. \ref{Fig_s-d5_FinalReproTraj}(b). Similar to Fig. \ref{Fig_s-d5_FinalReproTraj}(a), the trajectory of ${\vec \psi}\left(t\right) = \log\left({\vec R}\left(t\right)\right)$ passes through the boundary of ${\vec B}_\pi$ at $t = 9.99s$. But this does not mean that the trajectory of ${\vec R}\left(t\right)$ is discontinuous. The acceleration cost is increased from 1.4632 to 1.6832, which validates the benefits by considering IOVPs. As reflected in Fig. \ref{Fig_s-d5_FinalReproTraj}, the angular velocity in $y$ axis given in (a) is much smaller than the one given in (b). Correspondingly, the reproduced orientations ${\vec R}\left(t\right)$ at $t = {\tilde{t}}_k$ for $k = 1, 2, 3$ are visualized in Fig. \ref{Fig_s-d6-d12_IOVPs}(d)(e)(f). As shown in (d)(e)(f), all the reproduced orientations are perfectly aligned with the desired via-points, validating the fact that our designed Gaussian weighted curves can satisfy the dominance of each trajectory at the given local time domain. 

To better demonstrate the benefits of our proposed method, 3 different state-of-the-art geometry-aware methods are compared. The first one is the original Geometry-aware GMM and GMR method presented in (\cite{zeestraten2017approach}). However, it is non-trivial to incorporate the local constraints into this method. In order to incorporate the local constraints, we propose another two variants. In the first variant, the probabilistic modelling of demonstration data is replaced by the Geometry-aware GMM and GMR. Then the kernelized approach is employed to adapt to desired IOVPs. In the second variant, the probabilistic modelling is still done in multiple tangent spaces by using the off-the-shelf Euclidean GMM and GMR at the Angle-Axis Space. Then our proposed memory-based weighted rotation averaging mechanism is replaced by the iterative Karcher flow (\cite{2021INFOTEH-Kapic-WeightedRotationAveraging}), which is the state-of-the-art solution and used in (\cite{zeestraten2017approach}) to achieve weighted rotation averaging.

The results for the 3 baseline methods and our proposed method are summarized in Table \ref{Table-ComparisonsToKarcherFlow}. In Table \ref{Table-ComparisonsToKarcherFlow}, the direction error is the distance between the desired direction (denoted as ${\vec \eta}_d$) and the reproduced direction (denoted as ${\vec \eta}_c$). It can be given as:
\begin{align}
    \label{Eq_AngularError-Define}
    \begin{split}
        e\left({\vec \eta}_c, {\vec \eta}_d\right) = \left\|{\vec \eta}_c - {\vec \eta}_d\right\|,
    \end{split}
\end{align}
where ${\vec \eta}_c = {\vec R}\left({\tilde{t}}_k\right) {\vec \kappa}$, ${\vec \eta}_d = {\bar {\vec R}}_{dk} {\vec \kappa}$, and ${\vec \kappa}$ is the direction that the rotational constraint can be released, e.g., the $y$ or $z$ axis. The geodesic distance is the angular distance between the desired orientation ${\bar {\vec R}}_{dk}$ and the reproduced orientation ${\vec R}\left({\tilde{t}}_k\right)$, which is given in Eq. (\ref{Eq_AngleDistance}).

As we can see, the original Geometry-aware GMM/GMR cannot adapt to the desired via-points, reflected by the large angle errors between the desired and the reproduced directions at each IOVP. This is as expected because the local constraints are not considered in the original Geometry-aware GMM/GMR. However, this method is time efficient and can be computed in an online manner, since it only needs to fuse the multiple Gaussian components and its computation time is independent of the number of data points. Therefore, it is suitable for tasks with no extra local constraints. But it is non-trivial to incorporate local constraints in the learning and reproduction process. Even by using the kernelized approach to concatenate the desired via-points into the reference database, the ``Geometry-aware GMM/GMR + Kernelized" approach also fails to satisfy the desired via-points adaptation. This is because the kernelized approach encounters a distortion problem since the larger variance on the relaxed DoF drives the reproduced orientation to deviate from the desired geodesic, as we have explained in Fig. \ref{Fig_IllustrationForDistanceDistortion}(a). This phenomenon indicates the necessity of probabilistic modelling in multiple tangent spaces, in order to incorporate multiple local constraints.  

\begin{table*}[t]
  \centering
  \caption{ The comparisons to the Geometry-aware GMM/GMR and the Karcher Flow }
  \label{Table-ComparisonsToKarcherFlow}
  \setlength{\tabcolsep}{2.7mm}{
  \begin{tabular}{c|c|c|c|c|c|c}
    \hline
    \hline
       \multirow{2}{*}{}   &      & \multicolumn{3}{c|}{Constraint Errors (Radian)} & \multirow{2}{*}{Acc. Cost} & \multirow{2}{*}{ \makecell{Time Consumption \\(second)} } \\ \cline{3-5}
                           &      &  IOVP1   &   IOVP2    &   IOVP3   &   &                \\ \hline
       \multirow{2}{*}{Geo. GMM/GMR}&Direction Error & 1.5066 &   0.9146   &    1.6821  &\multirow{2}{*}{0.1358} &  \multirow{2}{*}{---} \\ \cline{2-5} 
                           & Geodesic Distance     & 2.7362 &   1.8572   &   3.0625   &       &   \\ \hline
       \multirow{2}{*}{ \makecell{Geo. GMM/GMR\\+Kernelized}}&Direction Error & 0.7306 &0.6431&1.5794&\multirow{2}{*}{3.3517}   &  \multirow{2}{*}{---} \\ \cline{2-5}
                           & Geodesic Distance     & 0.9063 &   1.4991   &   2.8773   &       &   \\ \hline
       \multirow{2}{*}{Karcher Flow}&Direction Error & 0.0029 &   0.0034   &   0.0020   &\multirow{2}{*}{2.7621}   &  \multirow{2}{*}{0.2331 $\pm$ 0.0556} \\ \cline{2-5}
                           & Geodesic Distance     & 0.7969 &   1.3503   &   0.6961   &       &    \\ \hline
       \multirow{2}{*}{Ours}&Direction Error       & 0.0004 &   0.0003   &   0.0009   & \multirow{2}{*}{1.4682}   &  \multirow{2}{*}{0.0662 $\pm$ 0.0025} \\ \cline{2-5} 
                           & Geodesic Distance     & 0.7978 &   1.3487   &   0.6964   &       &                       \\ \hline
    \hline
  \end{tabular}
  }
\end{table*}

\begin{table*}[t]
  \centering
  \caption{ The time consumption of 10 trials for the Karcher flow and our proposed method }
  \label{Table-TimeConsumption}
  \setlength{\tabcolsep}{2.5mm}{
  \begin{tabular}{c|c|c|c|c|c|c|c|c|c|c}
    \hline
    \hline
        Trials       &   1    &    2    &    3   &   4    &   5    &   6    &   7    &   8    &   9    &   10   \\ \hline
        Karcher Flow & 0.2461 &  0.1837 & 0.1892 & 0.2782 & 0.1818 & 0.3087 & 0.2193 & 0.2913 & 0.1501 & 0.2829 \\ \hline
        Ours         & 0.0693 & 0.0664  & 0.0653 & 0.0624 & 0.0672 & 0.0706 & 0.0671 & 0.0636 & 0.0649 & 0.0652 \\ \hline
    \hline
  \end{tabular}
  }
\end{table*}

Both the Karcher flow and our proposed memory-based weighted averaging mechanism can guarantee the IOCs, as reflected by the small direction errors between the desired and the reproduced directions at each IOVP. However, the direction errors of the Karcher flow are much higher than those of our memory-based method. This is due to the fact that the Karcher flow is an iterative algorithm and it depends on a predefined threshold (set to 0.001 in our implementation) to stop the iteration. A smaller threshold is helpful in improving the performance. However, the iteration count and computation time would be greatly increased. At the current accuracy, the time consumption of the Karcher flow is already much higher than that of our memory-based method. To better evaluate the time-consumption, each algorithm is run for 10 trials to eliminate randomness. The consumption times of the 10 trials, their means, and their standard deviations are recorded in Table \ref{Table-TimeConsumption} and Table \ref{Table-ComparisonsToKarcherFlow}. As we can see, since the Karcher flow is iterative and non-deterministic, its mean and standard deviation are both higher than those of our proposed memory-based mechanism. The histogram of the iterations of the Karcher flow at the 10-th trial is given in Fig. \ref{Fig_HistogramOfIterations}. As shown, only one-fifth of the iterations converge in fewer than 10 steps. Since the reproduced trajectory includes 1000 data points, the Karcher flow is executed 1000 times. It is obvious that the time-consumption would be increased with the trajectory length.

\begin{figure}[t]
    \centering
    \includegraphics[width = \hsize]{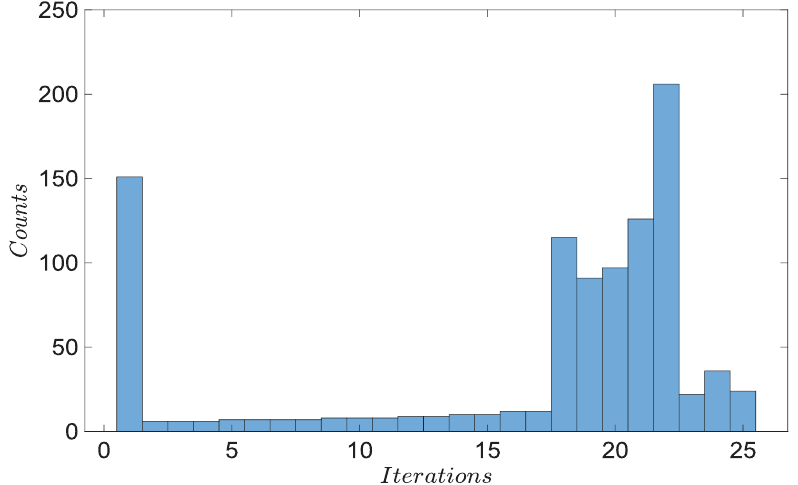}
    \caption{The histogram of the iterations of the Karcher flow for the 1000 data points at the 10-th trial.}
    \label{Fig_HistogramOfIterations}
\end{figure}

To better compare the Karcher flow with our proposed method, the reproduced trajectory is given in Fig. \ref{Fig_s-d5_FinalReproTraj}(d). The reproduced and desired orientations at the desired via-time $\tilde{t}_1 = 4.0s$, $\tilde{t}_2 = 7.0s$, and $\tilde{t}_3 = 10.0s$, using the Karcher flow, are given in Fig. \ref{Fig_s-d6-d12_IOVPs}(g)(h)(i). As shown in Fig. \ref{Fig_s-d5_FinalReproTraj}(d), the final reproduced trajectory of the Karcher flow encounters a discontinuity issue, especially at time $t = 9.8s$. As reflected in Table \ref{Table-ComparisonsToKarcherFlow}, the acceleration cost of the Karcher flow is approximately twice that of our method. These results validate the benefits of our proposed method over the Karcher flow, which is used in existing state-of-the-art method. 

Another interesting conclusion is that our memory-based weighted rotation averaging mechanism achieves the same averaging results as the Karcher flow, at points where the discontinuity issue is absent. As shown in Table \ref{Table-ComparisonsToKarcherFlow}, the reproduced orientations generated by the two different methods have nearly identical distances to all the 3 desired orientations. There is a high possibility that these two methods produce the same averaging results. But of course, this is only a necessary but not sufficient condition. Further supporting evidence is that the shape of the trajectory reproduced by the Karcher flow is similar to that of ours. These results validate the correctness of our proposed averaging mechanism. But our method can address the discontinuity issue and be more efficient than the Karcher flow.

\section{Experimental Evaluations}
\label{Section_Experiments}

In order to validate the capability of our proposed approach in incorporating multiple IOVPs, the earphone's peg-in-hole task is set for a 6-DoF robotic manipulator to fulfill. As shown in Fig. \ref{Fig_ExperimentalSetup}, the earphone's pin is fixed on the gripper of Unitree Z1 Pro, with its direction being aligned with the $y$-axis of the end-effector frame. A phone is fixed on a table, with the hole's direction being aligned with the vertical direction. To fulfill the plugging task, the $y$ direction of the earphone's pin has to be strictly the same as the direction of the hole, which means that the rotations around $x$ and $z$ directions are strictly constrained. However, the rotation around the $y$ direction can be released. The pose of the phone, relative to the manipulator's base frame, is pre-known by manual measurement.

\begin{figure}[t]
    \centering
    \includegraphics[width = \hsize]{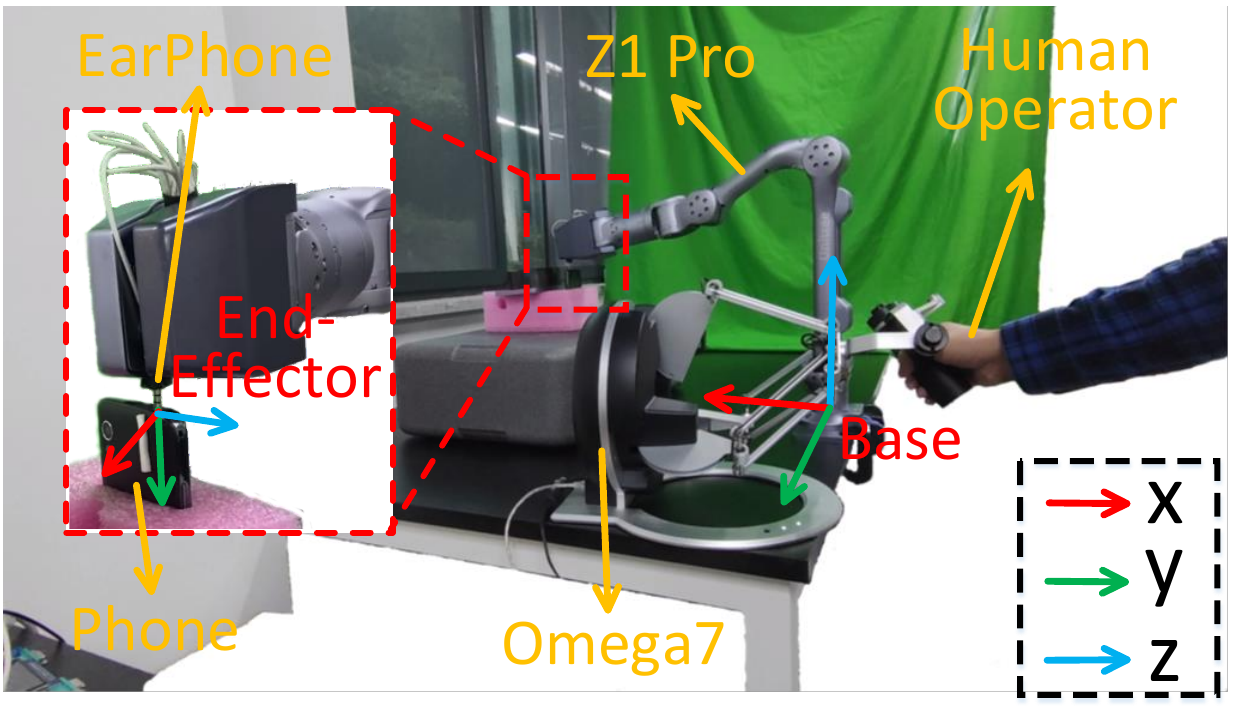}
    \caption{Experimental Setup. The earphone's pin is fixed on the gripper of Unitree Z1 Pro, with its direction being aligned with the $y$-axis of the end-effector frame. A phone is fixed on the table and the pose is known in advance through manual measurement. A human operator tele-operates the manipulator via the Omega 7 to collect demonstrations.}
    \label{Fig_ExperimentalSetup}
\end{figure}

\begin{figure}[h]
    \centering
    \subfigure[]{\includegraphics[width = 0.465\hsize]{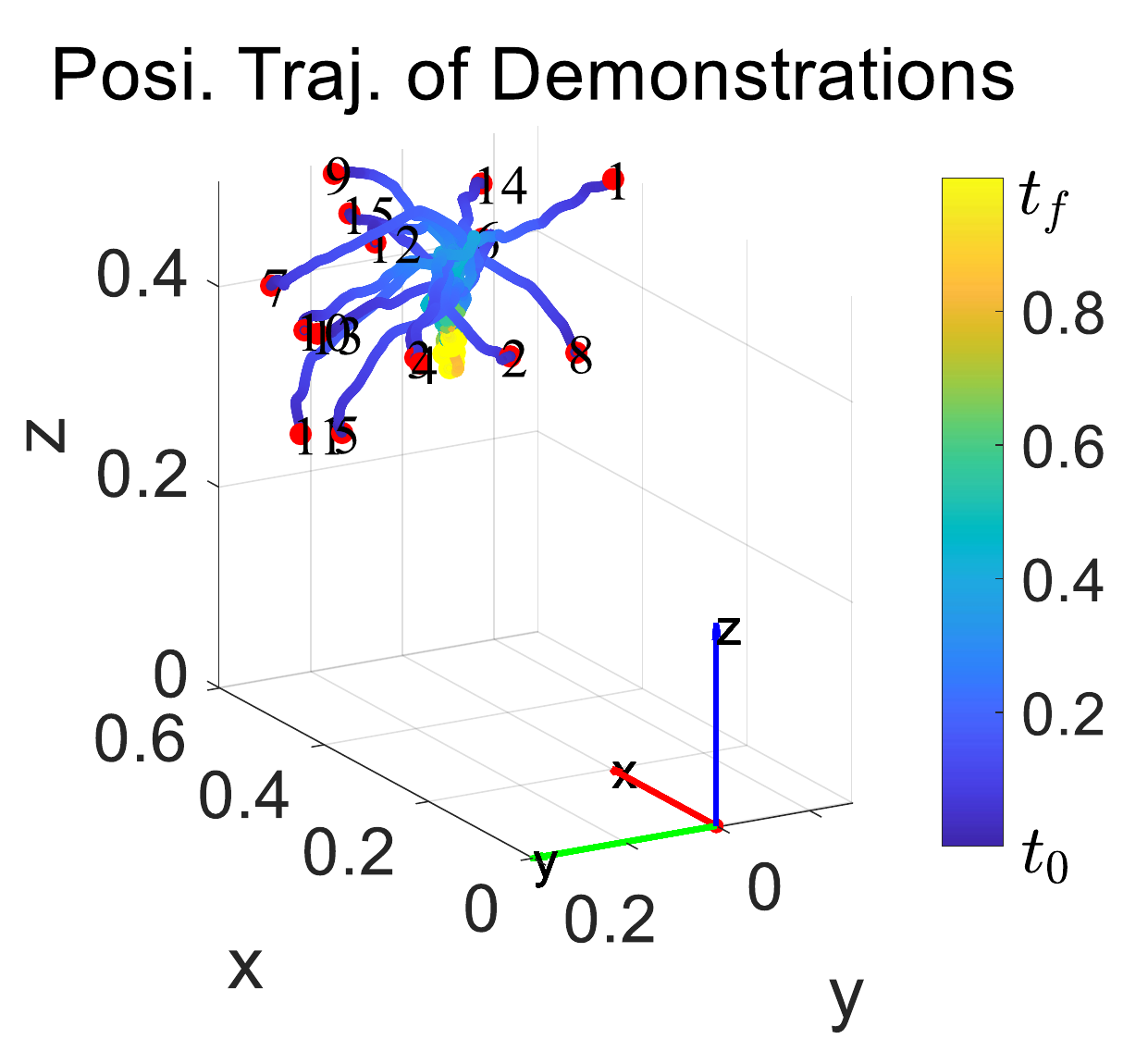}}
    \subfigure[]{\includegraphics[width = 0.515\hsize]{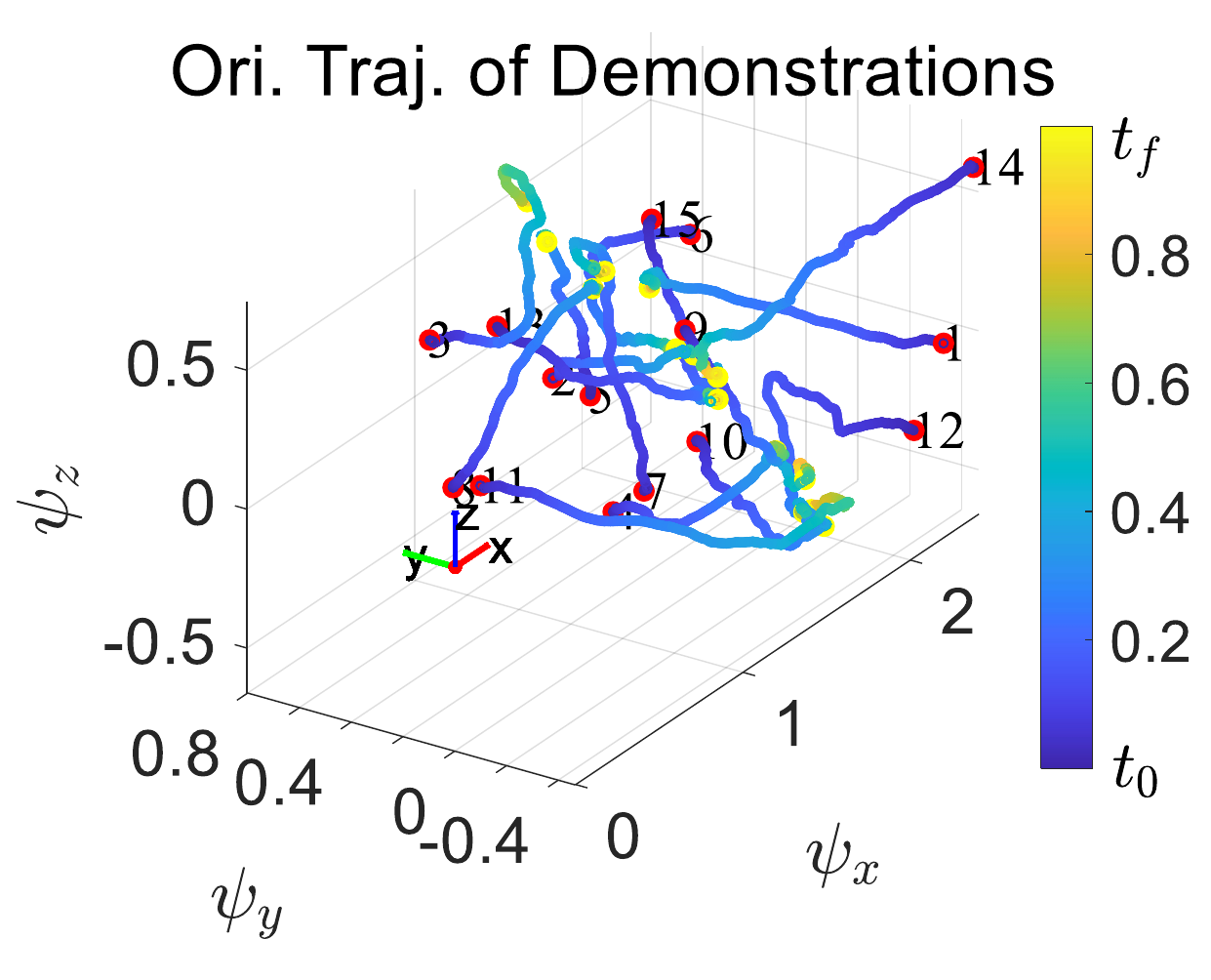}}
    \caption{The 15 demonstration trajectories. (a) Position. (b) Orientation. The red dots represent the starting poses of each trajectory. The color indicates the time sequence, where $t_0 = 0$ corresponds to $0.0$\,s, and $t_f = 1$ corresponds to $12.5751$\,s.}
    \label{Fig_E1DemonstrationsVisualization}
\end{figure}

In the demonstration stage, the manipulator is tele-operated by a human operator via the Omega 7, which is a haptic leader device, to fulfill the peg-in-hole task. 15 demonstrations are collected. In each demonstration, both the position and orientation trajectories of the end-effector frame, relative to the base frame, are recorded. These trajectories are visualized in Fig. \ref{Fig_E1DemonstrationsVisualization}. In this paper, the position and orientation trajectories are handled independently. In future work, it is valuable to learn and reproduce a trajectory directly in the SE(3) space. 

As shown in Fig. \ref{Fig_E1DemonstrationsVisualization}, the starting poses of the 15 demonstrations are diverse. In addition, to better capture the incomplete orientation constraint feature of the peg-in-hole task, the human operator is required to try their best to fulfill the insertion with different orientations. As reflected in Fig. \ref{Fig_E1DemonstrationsVisualization}(b), the orientation trajectories reach different points in the Angle-Axis Space. Actually, as stated in Section~\ref{subsect-discussion-angle-axis}, the end points of all orientation trajectories lie in a geodesic curve in SO(3), which is formed by the locus of ${\vec R}_d$ rotating around the relaxed DoF, denoted by $\gamma_{{\vec R}_d, {\vec \kappa}}\left(\theta\right)$. However, this feature is difficult to be recognized from the distribution of demonstrations in the tangent space. In addition, due to the limited workspace of the manipulator, the human operator is unable to fulfill the insertion from the backside of the phone. In other words, the data for insertion from the backside are missing, i.e., the various demonstrations cannot contain all varied rotations around $y$.

In the reproduction stage, 5 demonstrations are randomly picked to extract the probability distribution, which represents the latent human skills. Then, our approach is required to reproduce a new trajectory according to the starting pose of the manipulator and the pose of the phone. The position trajectory is generated by the KMP approach (\cite{2019IJRR-KMP}) and the orientation trajectory is generated by our proposed approach, respectively. Here we will mainly focus on the orientation part.

Two groups of experiments are conducted. In the first group, the pose of the phone remains the same as the one in the demonstration stage. But 3 different starting poses are randomly chosen for the manipulator. In the second group, both the position and orientation of the phone are changed, while the starting poses of the manipulator are chosen to be the same as the ones in the first group. For each test, 4 position and 3 orientation desired via-points are set, as shown in Table \ref{Table-ViaPointsExp}. Among them, the first via-point is the starting pose of the manipulator, such that the manipulator can avoid a sudden jump during the motion. Here the starting poses in three different tests are given as:
\begin{align}
    \label{Eq_StartingPosesinThreeTests}
    \begin{split}
        {\vec p}_1 &= [0.453, -0.251, 0.464]^\top \text{m}, \\ {\vec R}_1 &= \exp\left([-0.3873, 0.2968, -0.2593]^\top\right);\\
        {\vec p}_2 &= [0.419, 0.062, 0.433]^\top \text{m}, \\ {\vec R}_2 &= \exp\left([2.3733, -0.5512, -0.3892]^\top\right);\\
        {\vec p}_3 &= [0.544, -0.019, 0.501]^\top \text{m}, \\ {\vec R}_3 &= \exp\left([1.7931, -0.1488, 0.7398]^\top\right).
    \end{split}
\end{align}
The second and the third via-points are determined by the pose of the phone. Among them, the third position via-point is the same as the hole position of the phone, while the second position via-point is 30\,mm higher than the third one. The second and the third orientation via-points, which are the same as the hole orientation of the phone, are both set as IOVPs. To fulfill the insertion task, an extra position via-point, which is 20\,mm lower than the third via-point, is set for the position trajectory, while the orientation remains constant from the third point to the fourth point. The desired velocities for all via-points are set as zero. The weighted curves are designed as shown in Fig. \ref{Fig_E1WeightedCurve}.

\begin{table}[t]
  \centering
  \caption{The Desired Via-Points Used in Earphone Insertion Task. Position Unit: Meter}
  \label{Table-ViaPointsExp}
  \setlength{\tabcolsep}{1.2mm}{
  \begin{tabular}{c|c|c|c|c|c|c}
    \hline
    \hline
       \multirow{2}{*}{}  & \multicolumn{3}{c|}{Group I}     & \multicolumn{3}{c}{Group II}  \\ \cline{2-7}
                          &  Test 1   & Test 2   & Test 3    & Test 1 & Test 2 & Test 3       \\ \hline \hline
       $\tilde{t}_0$        &   \multicolumn{6}{c}{0.0s}                                       \\ \hline
       $\tilde{\vec p}_0$    & ${\vec p}_1$  & ${\vec p}_2$ & ${\vec p}_3$ & ${\vec p}_1$  & ${\vec p}_2$ & ${\vec p}_3$   \\ \hline
       $\tilde{\vec \psi}_0$ & ${\vec R}_1$  & ${\vec R}_2$ & ${\vec R}_3$ & ${\vec R}_1$  & ${\vec R}_2$ & ${\vec R}_3$  \\ \hline
       $\tilde{t}_1$        &  \multicolumn{6}{c}{8.8035s}                                     \\ \hline
       $\tilde{\vec p}_1$    & \multicolumn{3}{c|}{$\left[0.509; -0.066; 0.362\right]$} & \multicolumn{3}{c}{$\left[0.420; 0.015; 0.357\right]$}  \\ \hline
       $\tilde{\vec \psi}_1$ & \multicolumn{3}{c|}{$\left[1.5557; -0.0427; 0.0250\right]$} & \multicolumn{3}{c}{$\left[1.1719; 1.1761; 1.2101\right]$}  \\ \hline
       $\tilde{t}_2$        &  \multicolumn{6}{c}{11.3196s}                                     \\ \hline
       $\tilde{\vec p}_2$    & \multicolumn{3}{c|}{$\left[0.509; -0.066; 0.332\right]$} & \multicolumn{3}{c}{$\left[0.420; 0.015; 0.327\right]$}  \\ \hline
       $\tilde{\vec \psi}_2$ & \multicolumn{3}{c|}{$\left[1.5557; -0.0427; 0.0250\right]$} & \multicolumn{3}{c}{$\left[1.1719; 1.1761; 1.2101\right]$}  \\ \hline
       $\tilde{t}_3$        &  \multicolumn{6}{c}{12.5751s}                                     \\ \hline
       $\tilde{\vec p}_3$    & \multicolumn{3}{c|}{$\left[0.509; -0.066; 0.312\right]$} & \multicolumn{3}{c}{$\left[0.420; 0.015; 0.307\right]$}  \\ \hline
    \hline
  \end{tabular}
  }
\end{table}

\begin{figure}[t]
    \centering
    \includegraphics[width = \hsize]{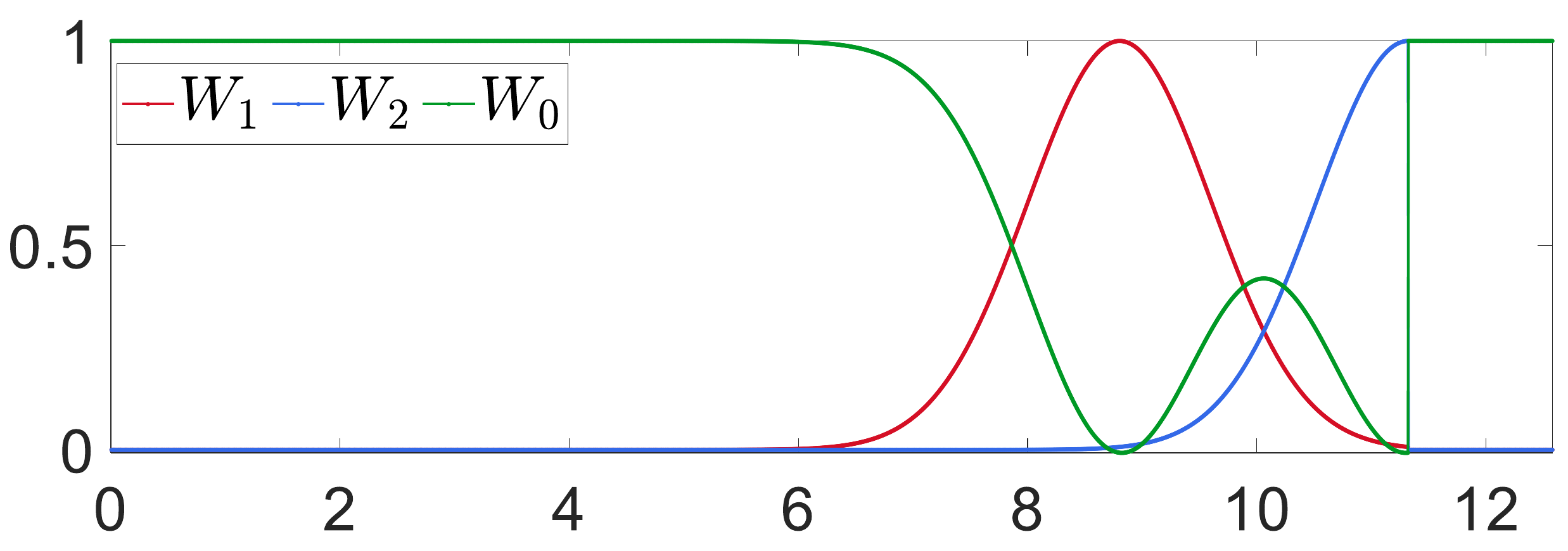}
    \caption{The weighted curves $W_k\left(t\right)$ used in the experiments. In the experiments, the first orientation via-point is a normal one, while the second and the third ones are IOVPs. Therefore, $W_1\left(t\right)$ and $W_2\left(t\right)$ are the weights corresponding to the second and the third via-points, respectively. And $W_0\left(t\right) = 1 - W_1\left(t\right) - W_2\left(t\right)$ is the weight for the baseline trajectory, which is reproduced by considering the first via-point. Here $\Delta t = 2.4s$, $\tilde{t}_1 = 8.8035$s, and $\tilde{t}_2 = 11.3196$s.}
    \label{Fig_E1WeightedCurve}
\end{figure}

\begin{figure}[!t]
    \centering
    \subfigure[]{\includegraphics[width = \hsize]{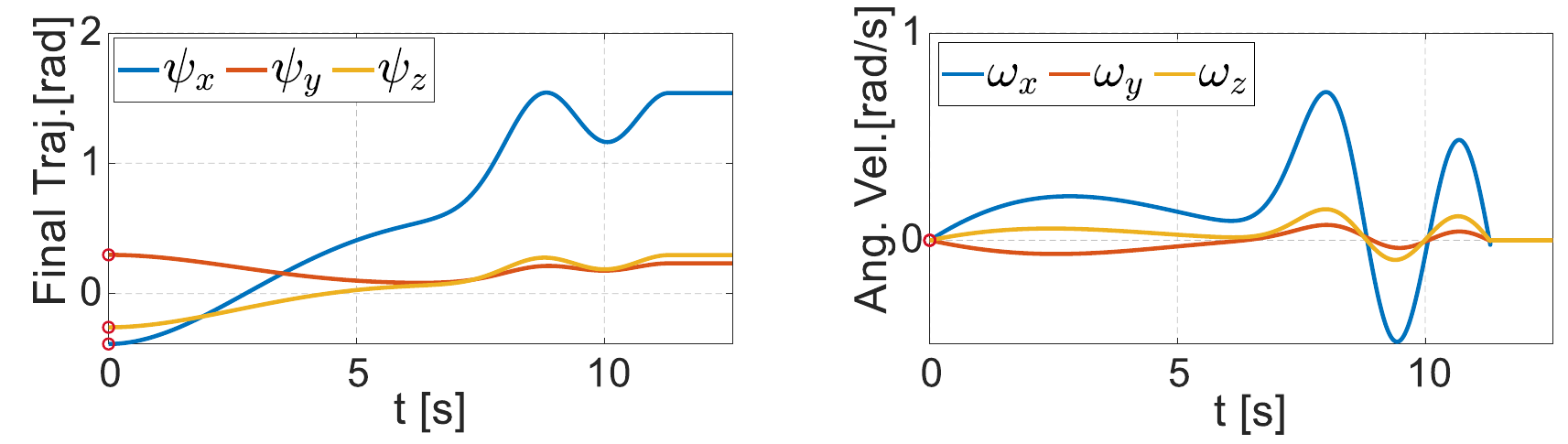}}
    \subfigure[]{\includegraphics[width = \hsize]{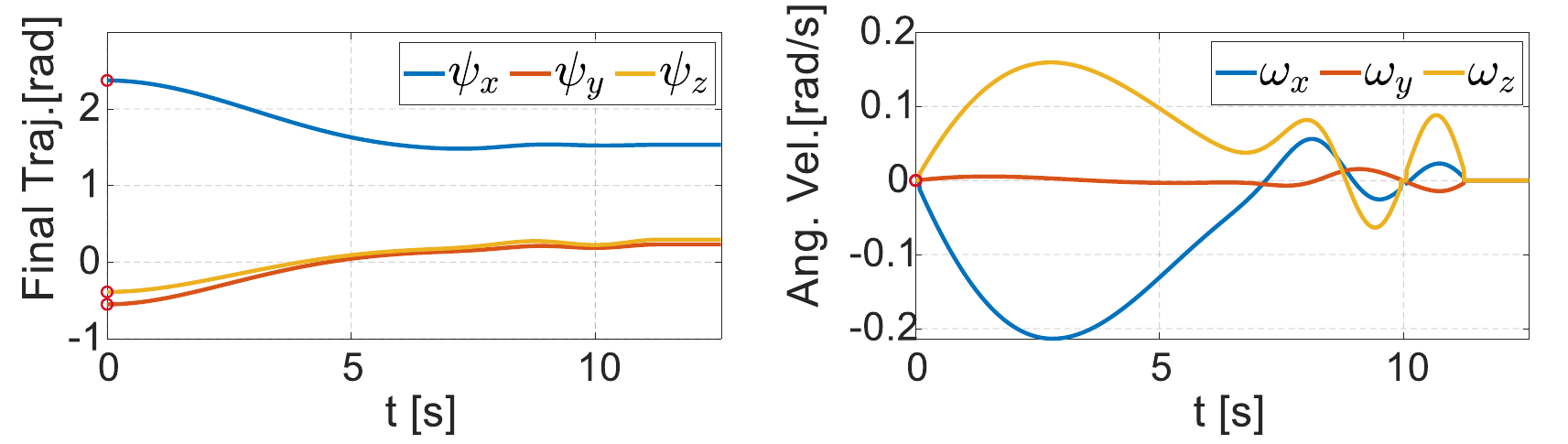}}
    \subfigure[]{\includegraphics[width = \hsize]{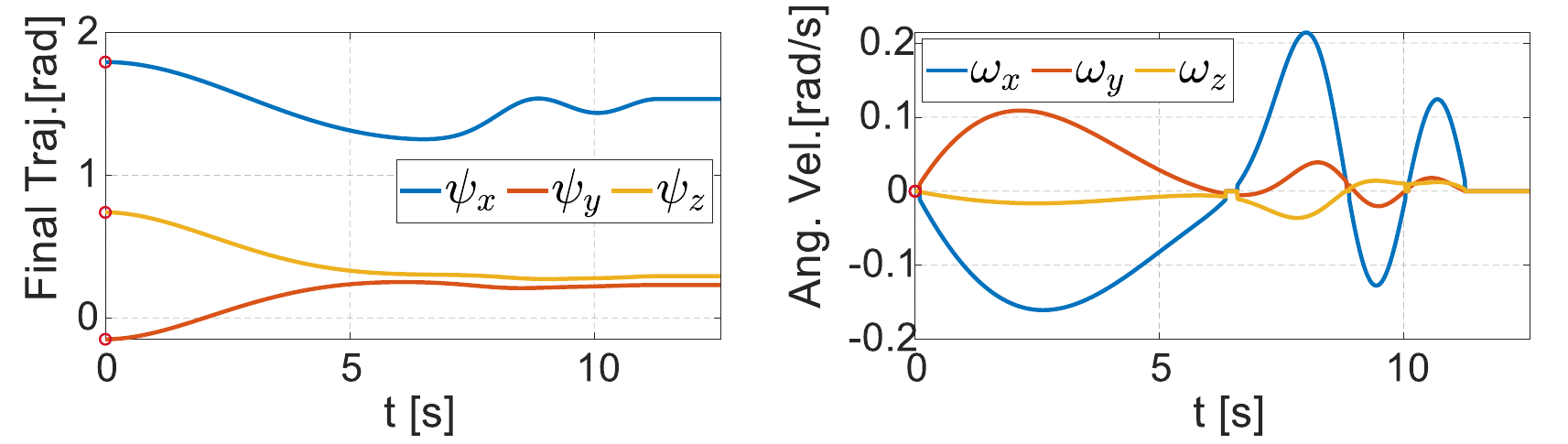}}
    \caption{The final reproduced orientation trajectories in Group I. (a) Test 1. (b) Test 2. (c) Test 3. The first via-point is depicted by red circles.}
    \label{Fig_ExpeGroupI}
\end{figure}

\begin{figure}[!t]
    \centering
    \subfigure[]{\includegraphics[width = \hsize]{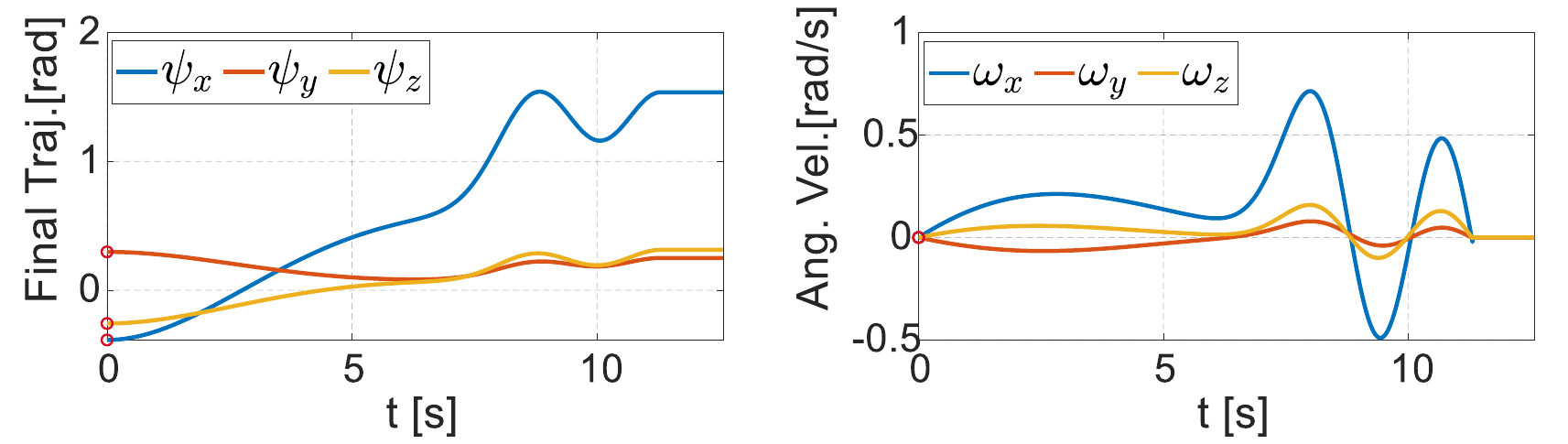}}
    \subfigure[]{\includegraphics[width = \hsize]{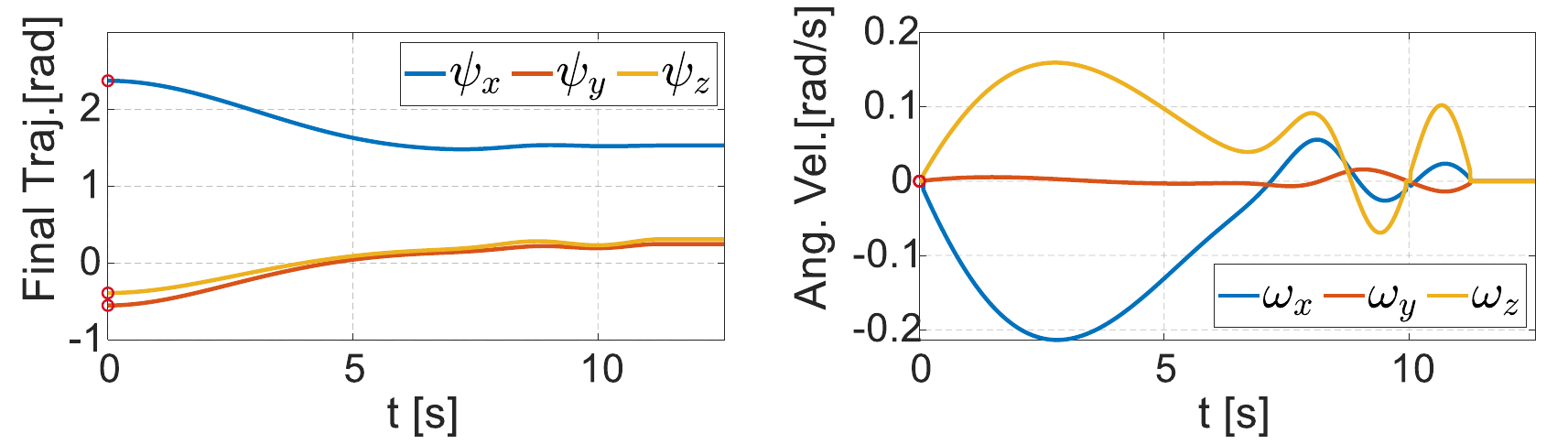}}
    \subfigure[]{\includegraphics[width = \hsize]{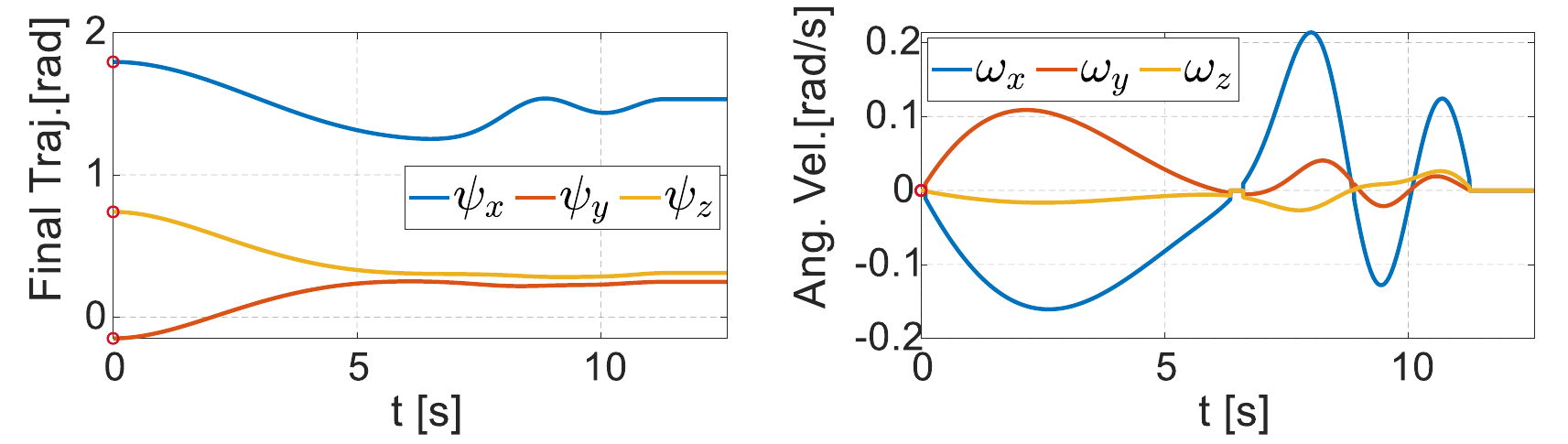}}
    \caption{The final reproduced orientation trajectories in Group II. (a) Test 1. (b) Test 2. (c) Test 3. Although the targets are different from the ones in Group I, our algorithm reproduces the similar trajectories by considering IOCs.}
    \label{Fig_ExpeGroupII}
\end{figure}

In order to reproduce a smooth trajectory that can satisfy all the IOCs simultaneously, our approach reproduces 3 orientation trajectories by considering the 3 desired via-points, respectively. Among them, the first trajectory is a baseline one that can go through the first desired via-point. The second and the third trajectories can satisfy the first and second IOCs only at time $\tilde{t}_1$ and $\tilde{t}_2$, respectively. Then the final orientation trajectory is generated by weighted summing the aforementioned 3 trajectories according to {\bf Algorithm \ref{algorithm1:MemorybasedWeighted Average}}. The experimental results for Group I and Group II are given in Fig. \ref{Fig_ExpeGroupI} and \ref{Fig_ExpeGroupII} \footnote{In (c), the angular velocity is directly set as 0 at time $t \approx 6.5s$ and $t \approx 10s$ due to the fact that its magnitude is smaller than a preset calculation accuracy. But this does not indicate the discontinuity in the orientation trajectory}, respectively.

\begin{table}[t]
  \centering
  \caption{The Angle Errors between the Desired and the Actual Directions at each IOVP. Unit: $\times10^{-3}$rad.}
  \label{Table-AngleError}
  \setlength{\tabcolsep}{4.0mm}{
  \begin{tabular}{c|c|c|c}
    \hline
    \hline
       \multicolumn{2}{c|}{}             & $\tilde{t}_1$          & $\tilde{t}_2$  \\ \hline \hline
       \multirow{3}{*}{Group I} & Test 1 & 6.52 & 6.45 \\ \cline{2-4}
                                & Test 2 & 0.23 & 0.45 \\ \cline{2-4}
                                & Test 3 & 1.67 & 1.76 \\ \hline
       \multirow{3}{*}{Group II}& Test 1 & 6.48 & 6.43 \\ \cline{2-4}
                                & Test 2 & 0.19 & 0.42 \\ \cline{2-4}
                                & Test 3 & 1.64 & 1.73 \\ \hline
       \multirow{3}{*}{Test Tube Insertion} & Test 1 &  6.88 & 7.65   \\ \cline{2-4}
                                & Test 2 & 4.75 & 5.46  \\ \cline{2-4}
                                & Test 3 & 2.62 & 2.92  \\ \hline
    \hline
  \end{tabular}
  }
\end{table}

\begin{figure*}[!t]
    \centering
    \subfigure[]{\includegraphics[width = 0.32\hsize]{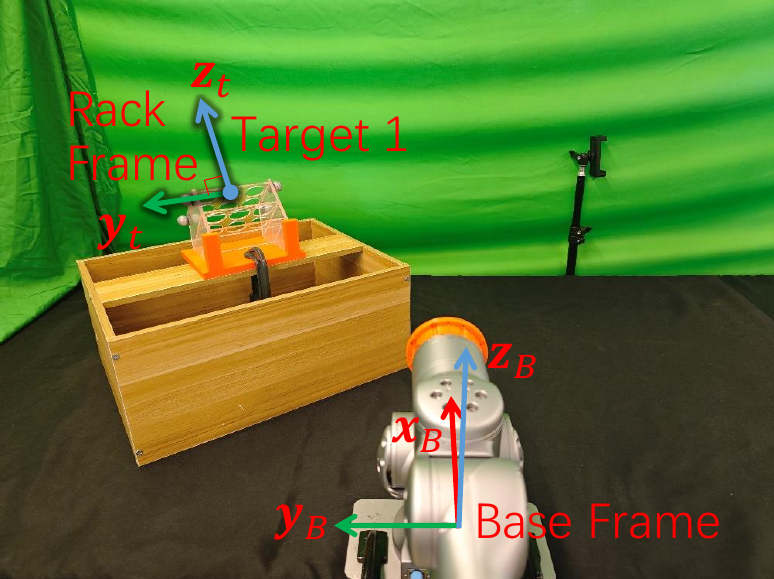}}
    \subfigure[]{\includegraphics[width = 0.32\hsize]{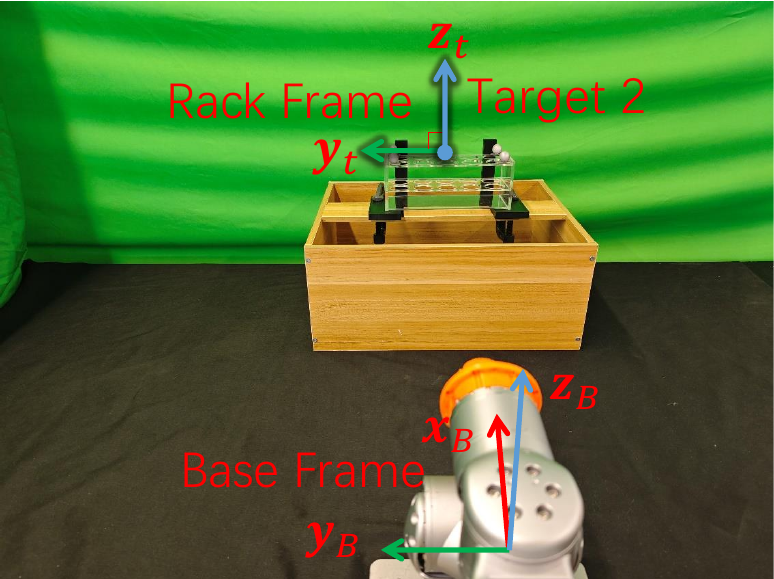}}
    \subfigure[]{\includegraphics[width = 0.32\hsize]{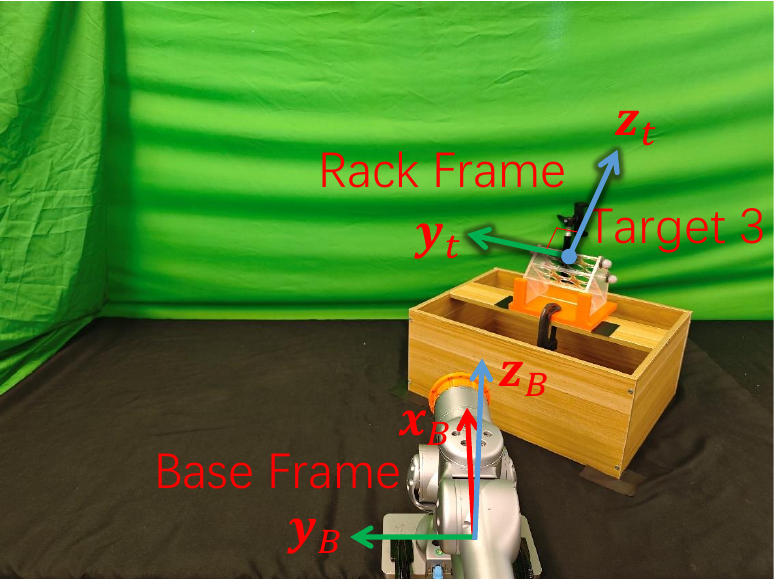}}
    \caption{The experimental setups for the test tube insertion task. (a) Target is on the left with an inclined angle. (b) Target is in the middle with a vertical direction. (c) Target is on the right with an inclined angle. The rotation around the $z$-direction of the rack frame does not affect the insertion task. Therefore, the constraint on the $z$-axis can be released.}
    \label{Fig_TestTubesInsertionSetup}
\end{figure*}

\begin{table*}[t]
  \centering
  \caption{The Via-points Used in Test Tube Insertion Tasks. Position Unit: Meter}
  \label{Table-ViaPointsExp-TestTubes}
  \setlength{\tabcolsep}{5.0mm}{
  \begin{tabular}{c|c|c|c}
    \hline
    \hline
                          &  Test 1   & Test 2   & Test 3    \\ \hline \hline
       $\tilde{t}_0$        &   \multicolumn{3}{c}{0.0s}                                       \\ \hline
       $\tilde{\vec p}_0$    &   \multicolumn{3}{c}{$\left[0.410; 0.0; 0.392\right]$}    \\ \hline
       $\tilde{\vec \psi}_0$&  \multicolumn{3}{c}{$\left[0.0; 0.0; 0.0\right]$}    \\ \hline
       $\tilde{t}_1$        &   \multicolumn{3}{c}{6.2875s(Posi)/8.8035s(Ori)}     \\ \hline
       $\tilde{\vec p}_1$    &  $\left[0.465; 0.354; 0.403\right]$  & $\left[0.523; 0.022; 0.403\right]$ & $\left[0.234; 0.110; 0.283\right]$ \\ \hline
       $\tilde{\vec \psi}_1$ & $\left[-0.1285; 0.6207; 0.8019\right]$ & $\left[-0.0149; -0.0388; -0.0282\right]$ & $\left[0.2354; 0.6845; -0.7125\right]$  \\ \hline
       $\tilde{t}_2$        &   \multicolumn{3}{c}{10.0590s(Posi)/11.3196s(Ori)}      \\ \hline
       $\tilde{\vec p}_2$    &  $\left[0.451; 0.344; 0.379\right]$  & $\left[0.524; 0.021; 0.373\right]$ & $\left[0.476; -0.293; 0.383\right]$ \\ \hline
       $\tilde{\vec \psi}_2$ & $\left[-0.1285; 0.6207; 0.8019\right]$ & $\left[-0.0149; -0.0388; -0.0282\right]$ & $\left[0.2354; 0.6845; -0.7125\right]$  \\ \hline
       $\tilde{t}_3$        &  \multicolumn{3}{c}{12.5751s}                                     \\ \hline
       $\tilde{\vec p}_3$    &  $\left[0.399; 0.307; 0.289\right]$  & $\left[0.528; 0.020; 0.263\right]$ & $\left[0.407; -0.234; 0.277\right]$ \\ \hline
    \hline
  \end{tabular}
  }
\end{table*}

As shown in Fig. \ref{Fig_ExpeGroupI} and \ref{Fig_ExpeGroupII}, all the reproduced trajectories are smooth, which validates that our designed weighted curves can guarantee the smoothness of the final trajectory. Also, all the trajectories can go through the red circles, which validates that the weighted average mechanism does not affect the adaptation to the first via-point. This can also be validated in the attached video, in which the manipulator does not have a sudden jump at the beginning of each test. Moreover, we also calculate the angle errors between the desired and the actual earphone-pin's direction (i.e., $y$ direction of the end-effector frame) at $\tilde{t}_1$ and $\tilde{t}_2$. As shown in the first 6 rows of Table \ref{Table-AngleError}, all the angle errors are very small, which means that all the IOCs are satisfied simultaneously. These results can validate the effectiveness of our weighted average mechanism. As reflected in the attached video, our reproduced trajectories can fulfill the peg-in-hole tasks successfully for all tests.

Another conclusion can be obtained by comparing the results between Fig. \ref{Fig_ExpeGroupI} and \ref{Fig_ExpeGroupII}. Although the target orientations in Group I are totally different from the ones in Group II (rotating by $\pi$ radian around the $y$-axis), the final reproduced trajectories for the same test (i.e., starting from the same pose) are almost the same, since the pin's direction is never changed. These results can validate the effectiveness and benefits of our proposed approach by considering IOCs, which are infeasible for existing approaches.

\begin{figure}[!t]
    \centering
    \subfigure[]{\includegraphics[width = 0.9\hsize]{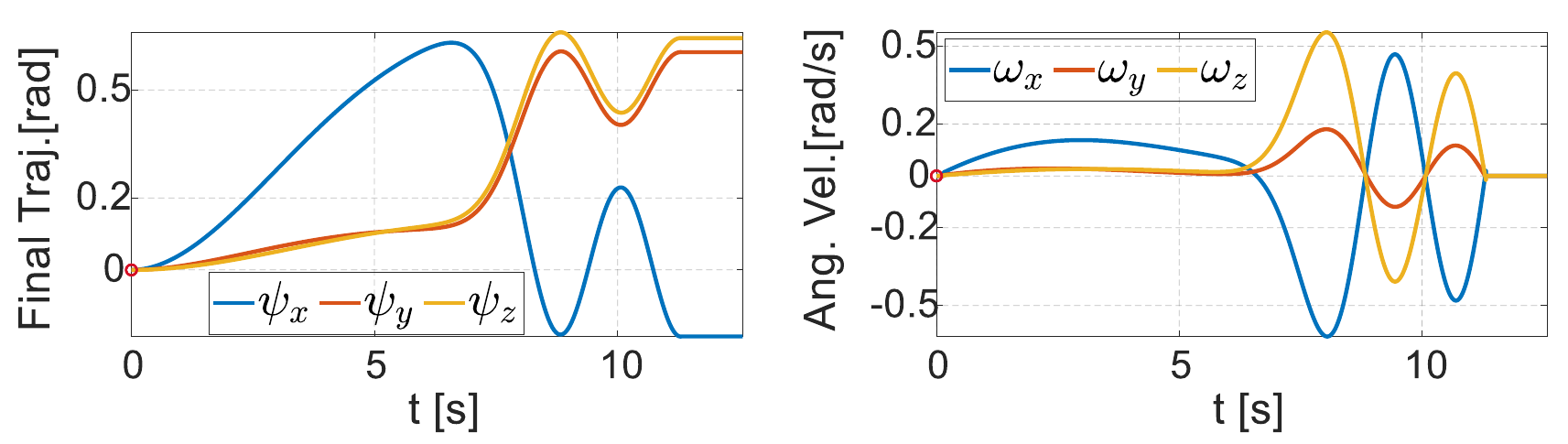}}
    \subfigure[]{\includegraphics[width = 0.9\hsize]{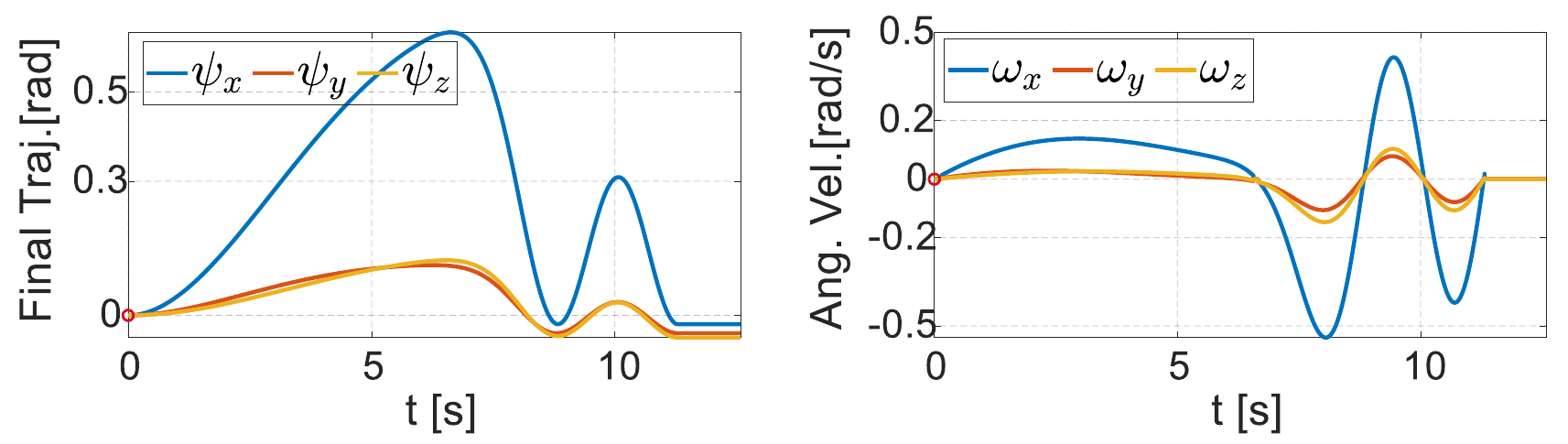}}
    \subfigure[]{\includegraphics[width = 0.9\hsize]{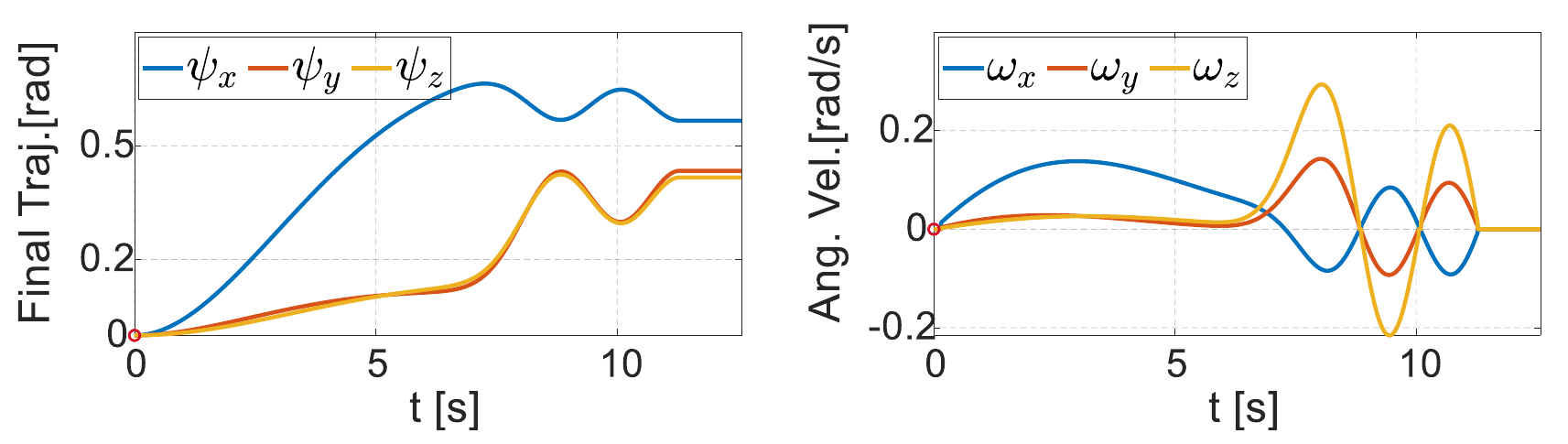}}
    \caption{Final reproduced orientation trajectory. (a) Test 1 (Left). (b) Test 2 (Middle). (c) Test 3 (Right).}
    \label{Fig_TestTubes_OriTraj}
\end{figure}

\begin{figure*}[!t]
    \centering
    \subfigure[]{\includegraphics[width = 0.32\hsize]{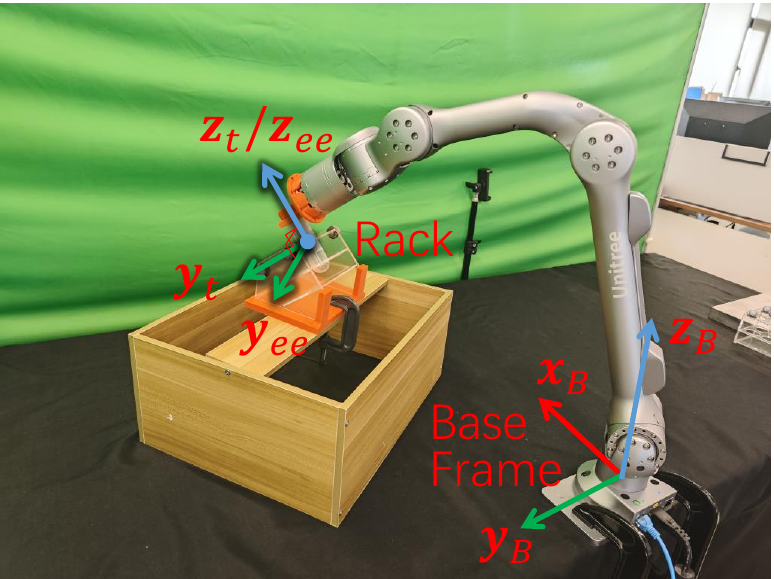}}
    \subfigure[]{\includegraphics[width = 0.32\hsize]{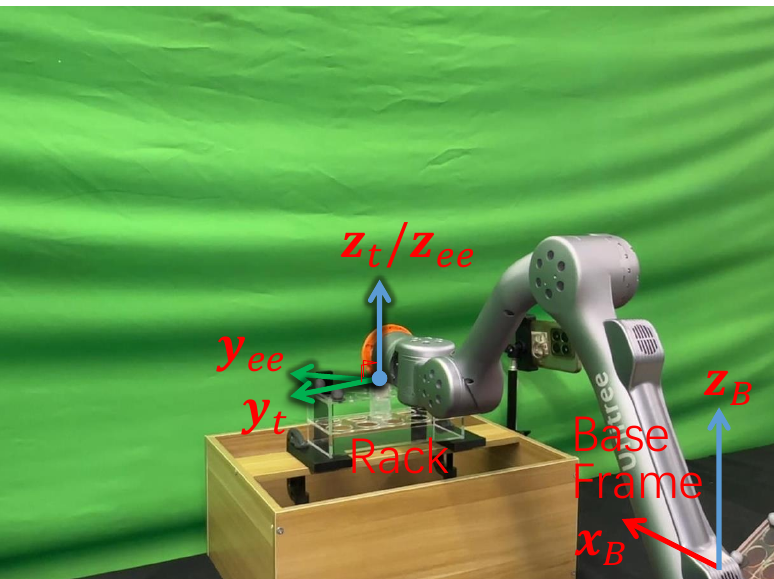}}
    \subfigure[]{\includegraphics[width = 0.32\hsize]{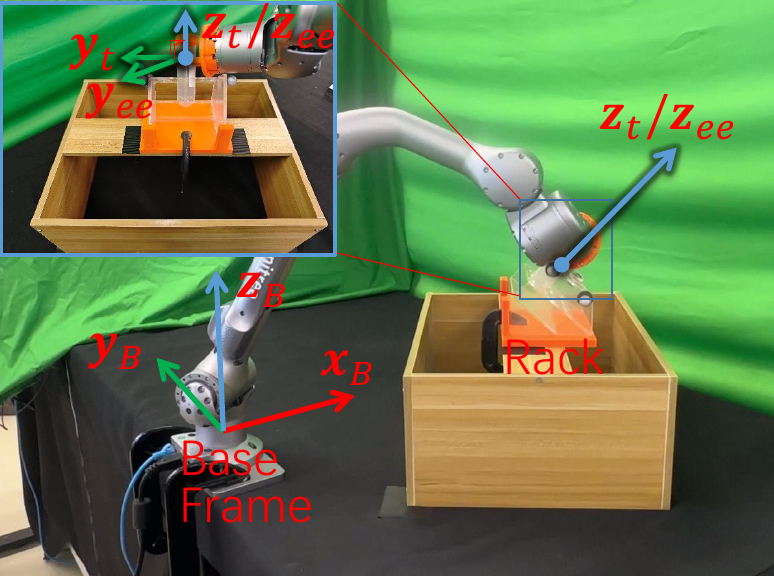}}
    \caption{The Snapshots for the three tests when the insertion is finished. (a) Test 1 for target at the left position. (b) Test 2 for target at the middle position. (c) Test 3 for target at the right position. The complete insertion process can be found in our attached video.}
    \label{Fig_TestTubesInsertionSnapshot}
\end{figure*}

In order to further validate the generalization ability of our proposed approach, the test tube insertion task is designed and carried out. Please note that we do not collect new demonstration data in this new task. Instead, the demonstrations for the earphone insertion task are reused. In this task, a test tube is fixed at the end of the manipulator, with its principal axis being aligned with the $z$-direction of the end-effector frame. Our algorithm is required to reproduce the position and orientation trajectory for the manipulator to insert the test tube into a rack. 

Three tests are carried out. In each test, the starting pose remains the same. The test tube rack is placed at 3 different positions, as shown in Fig. \ref{Fig_TestTubesInsertionSetup}, marked by left, middle, and right positions, respectively. As we can tell, the robot's rotation around the $z$-direction of the rack frame does not affect the insertion task. Therefore, the $z$-direction of the rack frame is the target direction for the test tube and the constraint on $z$-axis can be released. When the rack is placed at the left/right position, an inclined rack is used to generate an inclined target direction which has larger orientation variation compared to the ones in the earphone insertion task. The inclined angle is $45^{\circ}$. When the rack is placed at the middle position, a normal rack is used such that the target direction remains vertical. 

The detailed position and orientation for the left, middle, and right positions are summarized in Table \ref{Table-ViaPointsExp-TestTubes}. Similar to previous settings, 4 position via-points and 3 orientation via-points are set. The first position and orientation via-point are the starting position and orientation of the manipulator, such that the manipulator can avoid a sudden jump during the motion. And the second position via-point is fine tuned such that the generated position trajectory can avoid obstacles in the environment. The third position via-point is selected to be the same as the target position, but with a $95$mm (the length of the test tube) translation along the positive $z$-direction of the rack frame, such that the manipulator is ready for the insertion task. The fourth position via-point is the same as the target position, but with a $15$mm translation along the negative $z$-direction of the rack frame, such that the insertion can be guaranteed. The second and the third orientation via-points, which are both set as IOVPs, are the same as the orientation of the rack frame. And the orientation remains constant from the third point to the fourth point. The desired velocities for all via-points are set as zero. The weighted curves $W_k\left(t\right)$ are designed the same as given in Fig. \ref{Fig_E1WeightedCurve}.

The experimental results are given in Fig. \ref{Fig_TestTubes_OriTraj} and Fig. \ref{Fig_TestTubesInsertionSnapshot}. As given in Fig. \ref{Fig_TestTubes_OriTraj}, all the reproduced orientation trajectories are smooth. These results validate the effectiveness of our proposed imitation learning framework and our memory-based weighted rotation average mechanism. As shown in Fig. \ref{Fig_TestTubesInsertionSnapshot}, all the test tube insertion tests are successfully fulfilled, even for test 1 and test 3, in which the target directions are obviously different from the demonstrations. In addition, there is no need to collect new demonstration data for the test tube insertion task. These results validate the generalization ability of our proposed method. The angle errors between the desired and the actual test-tube's direction are given in the last 3 rows of Table \ref{Table-AngleError}. As we can tell, all the angle errors are very small, which means that all the IOCs are satisfied simultaneously.

\section{Conclusion}
\label{Section_Conclusion}

In order to incorporate an additional constraint for orientation learning, we proposed to encode and/or decode the demonstrations and via-points in the Angle-Axis Space, which is bijective to the rotation group SO(3). In addition, we also proposed a Gaussian-based weighted average mechanism to incorporate multiple IOVPs by utilizing geodesic curves in non-Euclidean spaces. More specifically, a memory-based algorithm was proposed to address the discontinuity problem in weighted rotation average. Based on this, we can generate multiple trajectories by projecting demonstrations into multiple tangent spaces and then fuse them by using a weighted rotation average mechanism. Although the distortion still exists for points that are far from the basepoint, our weighted rotation average mechanism can guarantee that each reproduced trajectory is only effective at the given local time domain. Therefore, our proposed approach addressed the distortion problem in a Riemannian manifold and made the off-the-shelf Euclidean learning algorithm re-applicable in non-Euclidean spaces. The simulation and experimental evaluations validated that our solution can not only inherit the KMP's capability in adapting orientations towards arbitrary desired via-points and coping with angular acceleration constraints, but also incorporate the IOCs to achieve extra benefits. In future work, we will focus on the extension of our proposed approaches for orientation learning by using a Reinforcement Learning (RL) or Vision-Language-Action (VLA) framework, in which the continuity of rotation representations is important, as revealed by \cite{2019CVPR-Zhou-ContinuityRepresentation}. And a more general fusing mechanism, e.g., the RMPflow (\cite{cheng2018rmp}), that considers hierarchical constraints (e.g., IOCs together with obstacle avoidance), remains an interesting direction for future work.

\section*{Funding}

This work is supported by NSFC 62688201 and 62673431, the Open Competition Mechanism on Humanoid Robot of Longyou County, Zhejiang Province, Major Scientific and Technological Research Project under Grant 2025246, the Special Fund for Innovation and Development of Hangzhou West Science and Technology Innovation Corridor under Grant K-20241950, the Independent Project of State Key Laboratory of Industrial Control Technology under Grant ICT2025A06. 

\section*{Conflict of Interest Statements}

The authors declared no potential conflicts of interest with respect to the research, authorship, and/or publication of this article.

\bibliographystyle{SageH}
\bibliography{references}

%\end{CJK}
\end{document}